%% file: main.tex
\documentclass{article}

\usepackage{microtype}
\usepackage{graphicx}
\usepackage{subcaption}
\usepackage{booktabs} 

\usepackage{hyperref}

\usepackage[preprint]{icml2026}

\usepackage{amsmath}
\usepackage{amssymb}
\usepackage{pifont}
\usepackage{mathtools}
\usepackage{amsthm}
\usepackage{booktabs}
\usepackage[]{multirow}
\usepackage{makecell}
\usepackage[table]{xcolor}
\usepackage{graphicx}
\usepackage{bm}
\usepackage{pgfplots}
\usepackage{tikz}
\usepgfplotslibrary{groupplots}
\usepgfplotslibrary{fillbetween}
\usetikzlibrary{decorations.pathreplacing}
\usetikzlibrary{arrows.meta}
\usetikzlibrary{patterns}
\usepackage{xurl}

\usepackage{orcidlink}

\DeclareMathOperator{\E}{\mathbb{E}}
\usepackage{amsthm} 

\usepackage[accsupp]{axessibility}  
\usepackage{glossaries} 
\usepackage{xcolor}

\usepackage{algorithm}
\usepackage{algpseudocode}

\definecolor{Blue1}{RGB}{9, 0, 255}
\definecolor{Red1}{RGB}{204, 0, 0}
\definecolor{Green1}{RGB}{0, 92, 78}
\definecolor{DarkGreen}{RGB}{85,168,104}
\definecolor{DarkYellow}{RGB}{240,180,0}
\definecolor{DarkBlue}{RGB}{108,142,191}
\definecolor{DarkBlue2}{RGB}{51,153,255}
\definecolor{DarkBlue}{RGB}{30,64,175}
\definecolor{MediumBlue}{RGB}{59,130,246}
\definecolor{LightBlue}{RGB}{158,200,252}
\definecolor{DarkGreen}{RGB}{22,101,52}
\definecolor{MediumGreen}{RGB}{34,197,94}
\definecolor{LightGreen}{RGB}{154,234,189}
\definecolor{DarkOrange}{RGB}{154,52,18}
\definecolor{MediumOrange}{RGB}{249,115,22}
\definecolor{LightOrange}{RGB}{255,213,172}
\definecolor{DarkPurple}{RGB}{91,33,182}
\definecolor{MediumPurple}{RGB}{168,85,247}
\definecolor{LightPurple}{RGB}{237,233,254}
\definecolor{DarkYellow}{RGB}{234,179,8}
\definecolor{MediumYellow}{RGB}{255,222,100}
\definecolor{LightYellow}{RGB}{246,228,152}
\definecolor{DarkGrey}{RGB}{65,65,65}       
\definecolor{MediumGrey}{RGB}{128,128,128}  
\definecolor{LightGrey}{RGB}{200,200,200}
\definecolor{m}{RGB}{230, 230, 230} 

\definecolor{DarkGreen}{RGB}{85,168,104}
\definecolor{DarkYellow}{RGB}{240,180,0}
\definecolor{DarkBlue}{RGB}{108,142,191}
\definecolor{LightGrey}{RGB}{200,200,200}
\definecolor{m}{RGB}{230, 230, 230} 
\definecolor{nice_red}{rgb}{0.8, 0, 0}
\definecolor{icmlblue}{RGB}{31, 119, 180}
\definecolor{icmlorange}{RGB}{255, 127, 14}
\definecolor{icmlred}{HTML}{D62728}

\newcommand{\cmark}{\textcolor{green!70!black}{\ding{51}}}
\newcommand{\xmark}{\textcolor{red}{\ding{55}}}

\usepackage{ifthen}

\newboolean{shorttable}
\setboolean{shorttable}{true}

\usepackage[capitalize,noabbrev]{cleveref}

\theoremstyle{plain}
\newtheorem{theorem}{Theorem}[section]

\theoremstyle{definition}

\newtheorem{assumption}[theorem]{Assumption}
\theoremstyle{remark}

\usepackage[textsize=tiny]{todonotes}

\glsdisablehyper 
\makeglossaries

\newacronym{dl}{DL}{Deep learning}
\newacronym{fim}{FIM}{Fisher information matrix}
\newacronym{flop}{FLOP}{Floating point operation}
\newacronym{kfa}{KFA}{Kronecker-factored approximation}
\newacronym{ml}{ML}{Machine learning}
\newacronym{milp}{MILP}{Mixed integer linear programming}
\newacronym{ilp}{ILP}{Integer linear programming}
\newacronym{nas}{NAS}{neural architecture search}
\newacronym{svd}{SVD}{Singular value decomposition}
\newacronym{vit}{ViT}{Vision Transformer}
\newacronym{zcm}{ZCM}{Zero cross-moment}
\newacronym{pp}{p.p.}{percentage points}
\newacronym{pela}{PELA}{PELA: Learning Parameter-Efficient Models with Low-Rank Approximation}
\newacronym{miou}{mIoU}{mean intersection over union}

\newacronym{gfwsvd}{GFWSVD}{Generalized fisher-weighted svd: Scalable kronecker-factored fisher approximation for compressing large language models}
\newacronym{kfac}{KFAC}{Optimizing neural networks with kronecker-factored approximate curvature}
\newacronym{svd-llm}{SVD-LLM}{SVD-LLM: Truncation-aware singular value decomposition for large language model compression.}
\newacronym{ours}{FACTS}{Fisher approximation for compression of vision transformers with SVD} 
\newacronym{oursearch}{CoRS}{Constrained rank search} 

\newacronym{deit}{DeiT}{Training data-efficient image transformers \& distillation through attention} 
\newacronym{swin}{Swin}{Swin transformer: Hierarchical vision transformer using shifted windows.} 
\newacronym{convnext}{ConvNeXt}{A convnet for the 2020s} 
\newacronym{mamba}{MambaVision}{Mambavision: A hybrid mamba-transformer vision backbone.}

\newacronym{imagenet}{ImageNet}{ImageNet} 
\newacronym{coco}{COCO}{Common Objects in Context} 

\newcommand{\ours}{FACTS}
\icmltitlerunning{Mind the Approximation: Fisher-Weighted SVD Compression for Vision Transformers}

\begin{document}

\twocolumn[
  \icmltitle{Mind the Approximation: Fisher-Weighted SVD Compression for ViTs}
  \icmlsetsymbol{equal}{*}

  \begin{icmlauthorlist}
    \icmlauthor{Moritz Thoma}{tum,bmw,equal}
    \icmlauthor{Maximilian Groezinger}{bmw,equal}
    \icmlauthor{Maximilian Forstenhäusler}{bmw,uk}
    \icmlauthor{Emad Aghajanzadeh}{tum,bmw}
    \icmlauthor{Ryan Pegoud}{vienna,bmw}
    \icmlauthor{Manoj Rohit Vemparala}{bmw}
    \icmlauthor{Pierpaolo Mori}{bmw}
    \icmlauthor{Alexander Frickenstein}{bmw}
    \icmlauthor{Daniel Mueller-Gritschneder}{vienna}
    \icmlauthor{Ulf Schlichtmann}{tum}
  \end{icmlauthorlist}

  \icmlaffiliation{tum}{Technical University of Munich, DE}
  \icmlaffiliation{bmw}{BMW Group, DE}
  \icmlaffiliation{uk}{University of Glasgow, GB}
  \icmlaffiliation{vienna}{TU Wien, AT}

  \icmlcorrespondingauthor{Moritz Thoma}{moritz.thoma@tum.de}
  \icmlkeywords{Machine Learning, ICML}

  \vskip 0.3in
]
\printAffiliationsAndNotice{\icmlEqualContribution}

\begin{abstract}
  Model compression is essential to mitigate deployment challenges of growing machine learning models. In this area of research, \glslink{svd}{singular value decomposition (SVD)}\glsunset{svd}-based compression offers a compelling trade-off between computational efficiency and model accuracy.\ Fisher-weighted \gls{svd} in particular provides principled, loss-aware compression. However, we find that improving the fidelity of Fisher approximation used in the compression is poorly predictive of post-compression accuracy for \glspl{vit}. Motivated by this observation, we propose \glslink{ours}{FACTS}\glsunset{ours}, a structured \textbf{F}isher \textbf{A}pproximation tailored to \textbf{C}ompressing Vi\textbf{T}s with Fisher-weighted \textbf{S}VD, which enforces token-local aggregation while preserving within-token activation-gradient dependence.\ Additionally, we introduce a fast \glslink{oursearch}{\textbf{Co}nstrained \textbf{R}ank \textbf{S}earch (CoRS)}\glsunset{oursearch}, that optimizes layer-wise rank allocation while adhering to a fixed \glslink{flop}{floating point operation (FLOP)}\glsunset{flop} constraint. Extensive experiments across \glspl{vit} and hybrid architectures demonstrate that \gls{ours} consistently improves accuracy-efficiency trade-offs without requiring finetuning. Notably, it outperforms the strongest \gls{svd} baseline by up to +5.8 \gls{pp} Top-1 on Swin-B, with further gains driven by our search method. \textit{\href{https://github.com/MoritzTho/FACTS}{Code on Github.}}
\end{abstract}

\input{sec/1_introv2}
\input{sec/2_related_work}
\input{sec/3_methodology}
\input{sec/4_experiments}
\input{sec/5_conclusion}

\bibliography{egbib}
\bibliographystyle{icml2026}

\newpage
\appendix
\onecolumn
\input{sec/X_suppl}

\end{document}

%% file: sec/1_introv2.tex
\section{Introduction}
\label{sec:intro}
Modern \glslink{dl}{deep learning}\glsunset{dl} models~\cite{deit, swin, convnext, mambavision} have achieved state-of-the-art performance across a diverse range of applications, including common computer vision tasks like image classification and object detection~\cite{deit, maskrcnn}. However, their ever growing computational demands and memory footprints make it increasingly difficult to deploy them on resource-constrained or real-world systems. A wide range of research has sought to improve computational efficiency of large-scale models through parameter reduction~\cite{luo2025icp, fang2024isomorphicpruning, fang2024maskllm, wang2024svd-llm, DeepCompress}, model quantization~\cite{frantar2023optq, Fu_2025_CVPR}, and \gls{nas}~\cite{Chang2024FLoRA, liu2018darts}.
While all of these approaches improve efficiency, in this work we exclusively focus on parameter reduction.

Among parameter reduction approaches, unstructured pruning~\cite{luo2025icp, Agarwal_2024_CVPR} removes individual weights to achieve high sparsity with minimal performance loss; however, unstructured removal rarely translates to practical acceleration on hardware like GPUs.
Structured pruning~\cite{Yang_2023_Nvit, gao2024disp, fang2024isomorphicpruning} removes entire channels or layers, yielding significant latency benefits, but often causes severe performance degradation and introduces complex structural dependencies.\
Semi-structured sparsity~\cite{fang2024maskllm,sparsegpt} has provided a middle ground, removing weights according to fixed structural rules (e.g., removing 2 out of 4 consecutive weight values) performing pruning without severe structural changes.\ However, its speedups rely on specialized hardware and are less effective for vision workloads. 
In other recent work~\cite{DeepCompress}, weights are directly encoded in s smaller space and decoded during inference, yielding strong memory reductions at the cost of retraining and limited computation reduction.
In contrast, \textbf{\gls{svd}-based compression} offers an attractive alternative:
by decomposing large weight matrices into two smaller sequential layers, \gls{svd} creates a low-rank approximation of the original weight that substantially reduces the parameter count without complex architectural changes. Moreover, because it maintains dense matrix multiplications, \gls{svd}-based compression avoids the hardware inefficiencies of unstructured and semi-structured sparsity.

The main challenge in \gls{svd}-based compression is preserving task-relevant information while performing the rank reduction that reduces the layer size.
By incorporating the \gls{fim}, which quantifies the sensitivity of model predictions to parameter perturbations, Fisher-weighted \gls{svd} approaches~\cite{chekalina2025generalizedfisherweightedsvdscalable, hsu2022fwsvd} address this challenge.\ 
Through weighting of the \gls{svd} objective, Fisher-weighted \gls{svd} approaches better align compression and task-level performance. 
However, since direct computation of the \gls{fim} is prohibitively expensive due to its large dimensions, it necessitates the use of approximations.
A common approximation with an optimal solution for \gls{svd} is a Kronecker-factored product of the \gls{fim}~\cite{chekalina2025generalizedfisherweightedsvdscalable}. Prior works~\cite{kfac-martens15, shampoo-gupta, shampoo-squared, fisherhessianmartens} have investigated and proposed different ways of obtaining these Kronecker factors for a range of settings and conditions. 
However, our analysis of existing \gls{fim} approximations reveals a counterintuitive finding: accurately reconstructing the empirical \gls{fim} (achieving high cosine similarity) does not guarantee high post-compression accuracy in \gls{svd}-based compression for \glspl{vit}.

Inspired by this finding, we propose \textbf{\gls{ours}}, a novel approximation that is grounded in our analysis of structural priors in Fisher-weighted SVD compression. Rather than maximizing \gls{fim} approximation quality, \gls{ours} focuses directly on improving post-compression accuracy. This is achieved by explicitly modeling beneficial structural priors, i.e., preserving token-local aggregation and activation-gradient interaction. 
Additionally, we introduce \textbf{\gls{oursearch}}, an efficient rank allocation method that improves upon prior approaches~\cite{2025_flar-svd,yuan2024asvd, xiao2023comcat, azizi2024memoryViT} by framing the search as a \glslink{milp}{mixed-integer linear programming (MILP)}\glsunset{milp} problem that can globally optimize the layer-wise budget allocation based on interpolated sensitivity measurements.
Our contributions are four-fold:
\begin{itemize}
    \vspace{0.05in}
    \item We analyze common Kronecker-factored \gls{fim} approximations for Fisher-weighted SVD compression in \glspl{vit} and show that higher \textbf{cosine similarity} to the empirical \gls{fim} is \textbf{not aligned} with post-compression accuracy.
    \vspace{0.05in}
    \item Motivated by this misalignment, we propose \textbf{\gls{ours}}, a structured Kronecker \gls{fim} approximation for Fisher-weighted \gls{svd} that uses \textbf{token-local aggregation} while preserving \textbf{activation-gradient dependence} to improve post-compression accuracy.
    \vspace{0.05in}
    \item We propose \textbf{\gls{oursearch}}, a \gls{milp}-based search that finds \textbf{proxy-optimal rank allocations} across layers, outperforming heuristic or iterative methods.
    \vspace{0.05in}
    \item  We provide extensive evaluations across \glspl{vit} and hybrid backbones demonstrating that \gls{ours} consistently outperforms the strongest \gls{svd} baselines. Notably, it yields Top-1 accuracy gains of up to \textbf{+5.8}~\gls{pp} (\textbf{+9.0}~\gls{pp} w/ \gls{oursearch}) on Swin-B, with improvements generalizing to both downstream tasks and various model sizes.
\end{itemize}

%% file: sec/2_related_work.tex
\section{Related Work}
\label{sec:related_work}
\paragraph{SVD-based low-rank compression.}
The simplest approach to low-rank compression is using \gls{svd} to minimize the reconstruction error of weight matrices~\cite{pela, azizi2024memoryViT, Luo2024FastLRD}. However, this can lead to suboptimal task performance~\cite{qinsi2025dobisvd, hsu2022fwsvd,wang2024svd-llm} as simple weight reconstruction may not restore small, but task critical parts. Thus, recent work has shifted toward minimizing either the intermediate feature error or the final task loss. 
\textit{Feature reconstruction methods} minimize the $\ell_2$-error of a layer’s reconstructed output activations. ASVD~\cite{yuan2024asvd} rescales the weight with activation magnitude to preserve important features. \acrshort{svd-llm}~\cite{wang2024svd-llm} extends this approach, providing optimality guarantees for the same objective by demonstrating that the optimal solution involves whitening the weight matrix with the covariance of the input activations before truncation.\ Its successor, SVD-LLMv2~\cite{2025-svd-llmv2}, improves numerical stability.\ FLAR-SVD~\cite{2025_flar-svd} further stabilizes the whitening process through covariance shrinkage, improving robustness for low sample counts. 
\textit{Loss-aware methods} incorporate the \gls{fim} to weight the \gls{svd} objective based on the loss sensitivity to parameter changes.\ This addresses a key issue of feature reconstruction methods, which guarantee minimal intermediate feature error, but do not consider that feature dimensions contribute unequally to the final loss. FW-SVD~\cite{hsu2022fwsvd} proposed this concept using a simple and overly coarse diagonal \gls{fim} approximation. \acrshort{gfwsvd}~\cite{chekalina2025generalizedfisherweightedsvdscalable} extends Fisher-weighted SVD by fitting Kronecker factors to the empirical \gls{fim} (via a Lanczos procedure), yielding a more faithful approximation with a closed-form SVD solution. However, our analysis suggests that improving \gls{fim} approximation fidelity alone is not sufficient to ensure strong post-compression accuracy.

In contrast, we explicitly analyze how the structural assumptions in common Kronecker factorizations affect \gls{svd} compression. With \gls{ours}, we synthesize these insights into a factorization tailored to compressing \glspl{vit} to retain higher post-compression accuracy.
\vspace{-0.1in}
\paragraph{Searching layer-wise ranks.}
After a single-layer compression method is defined, a global strategy is needed to allocate ranks (i.e., the compression budget) across layers, which greatly improves performance~\cite{2025_flar-svd, yuan2024asvd, azizi2024memoryViT, xiao2023comcat}. 
\textit{Gradient-based methods} introduce learnable parameters to optimize layer ranks. ComCat~\cite{xiao2023comcat} implements this principle by introducing differentiable rank choice distributions for vision models, while other works~\cite{gao2024adaptive_ARS, qinsi2025dobisvd} have applied similar strategies to large language models.\ 
FLORA~\cite{Chang2024FLoRA} leverages a one-shot \gls{nas} approach, training a supernet to jointly optimize all rank choices.\ This strategy, while effective, is computationally expensive and sensitive to hyperparameters. 
\textit{Greedy approaches} offer a simpler alternative, iteratively pruning ranks with the lowest sensitivity. In prior work, the sensitivity estimation differs: MemViT~\cite{azizi2024memoryViT} uses layer-wise matrix energy, while FastLRD~\cite{Luo2024FastLRD} uses weight similarity. 
\textit{Equal-error strategies} balance compression across layers by targeting uniform output error, ensuring that each layer contributes equally to performance degradation. Both ASVD~\cite{yuan2024asvd} and FLAR-SVD~\cite{2025_flar-svd} adopt this principle, equalizing reconstruction errors across layers achieving strong results. Yet, such heuristics lack guarantees of global reconstruction error minimization.

Different from prior work, we formulate the rank selection problem as a \gls{milp}. This simple reframing enables efficient per-layer rank allocation that, unlike prior works, \textit{globally} minimizes the total proxy error under a hard \gls{flop} budget.

%% file: sec/3_methodology.tex
\section{Preliminaries}
\label{sec:preliminaries}
This section introduces the mathematical background to generalized Fisher-weighted \gls{svd} and summarizes Kronecker-factored \gls{fim} approximations for linear weight-sharing layers, which form the foundation of \gls{ours}.

\subsection{Generalized Fisher-Weighted SVD Compression}
\label{sec:fw-svd}
To minimize the impact of the compression on the model predictions, the perturbation $\Delta\mathbf{W}$ caused by the low-rank compression of a linear layer $\mathbf{W}\!\in\!\mathbb{R}^{m\times n}$ should minimize the expected second-order loss approximation $\mathbb{E}[\Delta \mathcal{L}] \approx \frac{1}{2}\mathrm{vec}(\Delta\mathbf{W})^\top \mathbf{F}\,\mathrm{vec}(\Delta\mathbf{W})$ \cite{chekalina2025generalizedfisherweightedsvdscalable}. Since the full \gls{fim}, $\mathbf{F}\in\mathbb{R}^{mn\times mn}$, is intractable, it is commonly approximated as a Kronecker product $\mathbf{F} \approx \mathbf{A} \otimes \mathbf{B}$ with $\mathbf{A}\in\mathbb{R}^{n\times n}$ and $\mathbf{B}\in\mathbb{R}^{m\times m}$~\cite{chekalina2025generalizedfisherweightedsvdscalable, shampoo-gupta, shampoo-squared, kfac-martens15, runa-kfacreduce}. Under this structure, the \emph{optimal low-rank solution} is obtained via Fisher-whitening~\cite{chekalina2025generalizedfisherweightedsvdscalable}, where one computes the truncated \gls{svd} of the transformed weights $\overline{\mathbf{W}}=\mathbf{B}^{1/2}\mathbf{W}\mathbf{A}^{1/2}$ ($\approx \mathbf{U}_k\boldsymbol{\Sigma}_k\mathbf{V}_k^\top$) and projects the result back via 
$\widetilde{\mathbf{W}}_k=\mathbf{B}^{-1/2}\mathbf{U}_k\boldsymbol{\Sigma}_k\mathbf{V}_k^\top\mathbf{A}^{-1/2}$. 
We refer to Supplement \cref{supp:sec:fw-svd-objective} for the detailed derivation.
For the compression of a layer, this rank-$k$ matrix is divided into two smaller consecutive linear layers $\mathbf{W}_A \in \mathbb{R}^{m\times k}$ and $\mathbf{W}_B \in \mathbb{R}^{k\times n}$ by splitting the unwhitened approximation such that $\mathbf{W}_A = \mathbf{B}^{-1/2}\mathbf{U}_k\boldsymbol{\Sigma}^{1/2}_k$ and $\mathbf{W}_B = \boldsymbol{\Sigma}^{1/2}_k\mathbf{V}_k^{\top}\mathbf{A}^{-1/2}$. 
This yields $\widetilde{\mathbf{W}}_k = \mathbf{W}_A \mathbf{W}_B$, which reduces the parameter count compared to $\mathbf{W}$ whenever $k<\frac{nm}{n+m}$. 
Ultimately, the efficacy of this \gls{svd} compression largely depends on how well $\mathbf{A}$ and $\mathbf{B}$ capture the relevant compression curvature.

\subsection{Kronecker Structure of the Fisher for Weight-Sharing Layers}
\label{sec:kronecker_fisher_ws}
In Transformers, a linear layer $\mathbf{W}$ processes a sequence of $T$ tokens. Let $\mathbf{x}_t$ and $\mathbf{g}_t$ denote the input and pre-activation gradient for token $t$. Due to weight sharing, the total layer gradient is the sum of contributions across all tokens: $\nabla_{\mathbf{W}}\mathcal{L}=\sum_{t=1}^{T}\mathbf{g}_t\mathbf{x}_t^\top$. Hence, the layer's FIM is a sum of Kronecker products over all token pairs:
\vspace{-0.1in}
\begin{multline}
\mathbf{F}
:=\mathbb{E}\!\left[\mathrm{vec}(\nabla_{\mathbf{W}}\mathcal{L})\,\mathrm{vec}(\nabla_{\mathbf{W}}\mathcal{L})^\top\right]\\
=\mathbb{E}\!\left[
\sum_{t=1}^{T}\sum_{s=1}^{T} (\mathbf{x}_t\mathbf{x}_s^\top)\otimes(\mathbf{g}_t\mathbf{g}_s^\top)
\right].
\label{eq:ws_fisher_sum}
\end{multline}
\vspace{-0.1in}

In practice one seeks a \emph{single} Kronecker approximation $\mathbf{F}\approx \mathbf{A}\otimes\mathbf{B}$. Existing estimators \cite{kfac-martens15, runa-kfacreduce, chekalina2025generalizedfisherweightedsvdscalable, shampoo-squared} differ in two structural axes to obtain $(\mathbf{A},\mathbf{B})$.

\vspace{-0.1in}
\paragraph{Token Aggregation (local vs global).}
The double sum in \cref{eq:ws_fisher_sum} contains local moments ($t=s$) and cross-token mixed moments ($t\neq s$). 
\textbf{KFAC-expand}~\cite{runa-kfacreduce, kfac-martens15} enforces locality by treating tokens as independent samples, effectively discarding cross-token terms ($t \neq s$).
Conversely, \textbf{KFAC-reduce}~\cite{runa-kfacreduce} performs global aggregation (summing $\mathbf{x}_t$ before the outer product), thereby incorporating cross-token correlations into the factors. \textbf{Shampoo$^2$} and \textbf{GFWSVD} compute a nearest Kronecker product to the \gls{fim} (via rearrangement/power-iteration or Lanczos-style routines), and therefore inherit whichever cross-token moments dominate the \gls{fim} they are fit to~\cite{shampoo-squared,chekalina2025generalizedfisherweightedsvdscalable}.

\vspace{-0.1in}
\paragraph{Activation-Gradient Coupling.}
Estimators also differ in their consideration for activation-gradient couplings.
Standard KFAC-style estimators~\cite{kfac-martens15,runa-kfacreduce} assume that activations and output gradients are independent.
In contrast, nearest Kronecker product methods (Shampoo$^2$, GFWSVD) do not impose an activation-gradient prior. Like their treatment of cross-token structure, their nearest product fit reflects dependence of activations and gradients.

Taken together, the KFAC-reduce/expand and the recent Shampoo$^2$/GFWSVD formulations define a compact design space of Kronecker \gls{fim} estimators for weight-sharing layers, which is summarized in \cref{fig:motivation:a}.

\section{Methodology}
\label{sec:methodology}
\subsection{Kronecker Fisher Approximations for \gls{svd} Compression}
\label{sec:fisher_analysis}
To design a curvature estimator tailored to Fisher-weighted \gls{svd} compression, we first examine which structural priors of Kronecker \gls{fim} approximations affect post-compression accuracy. We evaluate three representative estimators (Shampoo$^2$, KFAC-reduce, and KFAC-expand) within the same Fisher-weighted \gls{svd} pipeline, keeping the compression budget and whitening procedure fixed. This isolates the effect of how the structural priors (token aggregation, activation-gradient coupling) affect post-compression accuracy and alignment (cosine similarity) to the empirical \gls{fim}.
Additional experimental details and empirical results are provided in Supplement \cref{supp:subsec:fisher_comparison_motivation} and \ref{supp:sec:more_actgrad_and_zcm}, respectively.
\vspace{-0.1in}
\paragraph{Similarity vs.\ Accuracy.}
Table~\ref{fig:motivation:a} shows that, counterintuitively, higher cosine similarity to the empirical \gls{fim} does not reliably translate into better compression. Both Shampoo$^2$ and KFAC-reduce achieve higher cosine similarity than KFAC-expand (0.17/0.13 vs.\ 0.12), yet only Shampoo$^2$ improves accuracy, while KFAC-reduce performs substantially worse (67.0\% vs.\ 75.9\%). 
This discrepancy motivates a closer examination of the two structural differences between these estimators: token aggregation and activation-gradient coupling.
\vspace{-0.1in}
\paragraph{Token Aggregation.}

We first examine whether modeling cross-token structure is beneficial for Fisher-weighted \gls{svd}. \cref{fig:motivation:b} visualizes a cross-token coupling diagnostic derived from activations and output gradients.\ The diagnostic exhibits structured off-diagonal patterns, indicating that mixed moments between different token positions \textit{are present} in the empirical \gls{fim} for DeiT (and other vision models cf. Supplement \cref{supp:fig:extended_motivation_grid}). The observed structure is consistent with the fixed spatial organization of vision tokens: token indices often retain a stable spatial meaning across samples, so nearby positions correspond to nearby image regions whose activations and gradient responses can be correlated.\ 
Because these spatial neighborhood relations between tokens stay the same across images, the aggregated statistics can lead to prominent cross-token moments in the \gls{fim}.
This observation helps to explain why globally aggregating estimators, such as KFAC-reduce, or estimators fitted to the full empirical Fisher, such as Shampoo$^2$, can achieve higher cosine similarity to the empirical \gls{fim}: they represent cross-token structure that KFAC-expand explicitly discards. 

However, better Fisher alignment does not necessarily imply a better low-rank subspace for compression. 
When considering a perturbation \(\Delta\mathbf W\), the compression-induced first-order error at token \(t\) is \(e_t(\Delta\mathbf W)=\mathbf g_t^\top\Delta\mathbf W\mathbf x_t\), and the Fisher objective decomposes as Eq. \ref{eq:cross_moments_explain}, with two main terms. The first term measures the sensitivity to the perturbation at individual token positions, whereas the second captures mixed second moments between errors at different positions.\ Since the same weight \(\mathbf W\) is applied at every token position, the token-local terms directly characterize how the shared weight behaves under compression; the cross-token terms additionally reflect how perturbation-induced errors co-vary across positions. Thus, matching spatially structured cross-token moments can improve \gls{fim} similarity without necessarily selecting low-rank directions that best preserve post-compression accuracy.
\vspace{-0.1in}
\begin{multline}\label{eq:cross_moments_explain}
    \E\!\left[\left(\sum_t e_t(\Delta\mathbf W)\right)^2\right]
    =
    \sum_t \E[e_t(\Delta\mathbf W)^2]
    \\+
    \sum_{t\neq s} \E[e_t(\Delta\mathbf W)e_s(\Delta\mathbf W)]
\end{multline}
\vspace{-0.1in}

The comparison between KFAC-reduce and KFAC-expand isolates this effect most directly, since both use the same KFAC-style activation-gradient decoupling but differ in their treatment of the weight-sharing dimension. Despite achieving higher cosine similarity to the empirical \gls{fim}, KFAC-reduce yields substantially lower post-compression accuracy than KFAC-expand (67.0\% vs.\ 75.9\%). This suggests that, for Fisher-weighted \gls{svd} compression of weight sharing layers in vision, token-local curvature statistics are more useful than cross-token mixed moments for selecting an effective low-rank approximation.
\input{figures/method_figure}
\vspace{-0.125in}
\paragraph{Activation-Gradient Coupling.}
Token-local aggregation alone does not fully explain post-compression performance: Shampoo$^2$ outperforms KFAC-expand (77.2\% vs.\ 75.9\%) despite not enforcing the same token-local prior.\ A core difference is that KFAC-style estimators form separate activation and gradient factors, replacing within-token joint structure by a product of marginal statistics. \cref{fig:motivation:c} indicates substantial dependence between \(\mathbf x_t\) and \(\mathbf g_t\) in the investigated layer, suggesting that this decoupling may discard relevant information for Fisher-weighted \gls{svd}. Thus, the benefit of Shampoo$^2$ is consistent with the importance of within-token activation-gradient coupling.
\vspace{-0.125in}
\paragraph{Implication.}

These observations suggest two desirable properties for Fisher-weighted \gls{svd} in vision models: the approximation should emphasize token-local curvature, while preserving within-token activation-gradient dependence. 
KFAC-expand satisfies the first property but decouples activations and gradients; Shampoo$^2$ and GFWSVD can retain activation-gradient dependence, but may also fit cross-token mixed moments that improve Fisher similarity without reliably improving post-compression accuracy. This motivates \gls{ours}, combining token-local aggregation and within-token activation-gradient coupling to target curvature statistics most relevant for low-rank compression of weight sharing layers in vision.

\subsection{\gls{ours}}
The preceding analysis suggests that an effective Fisher approximation for \gls{svd} compression should retain coupled activation-gradient statistics within each token, while suppressing cross-token mixed moments. In this section, we instantiate this principle in \gls{ours}. 
Starting from the weight-sharing \gls{fim} expansion in \cref{eq:ws_fisher_sum}, we impose a token-local structural prior by removing only the cross-token \emph{mixed} moments. 
\vspace{-0.1in}
\paragraph{Zero Cross-Moment (ZCM).}
We introduce the following token-local structural prior, used as an inductive bias for compression:

\begin{assumption}[Zero Cross-Moment (ZCM)]
\label{as:zcm-assumption}
For any two distinct token indices $t\neq s$, the mixed second moments vanish in expectation:
\begin{equation}
  \mathbb{E}\!\left[(\mathbf{x}_t^\top \mathbf{x}_s)\, (\mathbf{g}_t\mathbf{g}_s^\top)\right] = \mathbf{0}.  
\end{equation}
\end{assumption}

Importantly, this is not an independence assumption over tokens. Instead, it is a compression-oriented approximation that discards the specific cross-token terms that can dominate global Fisher similarity, while leaving the within-token activation-gradient coupling intact.
\vspace{-0.1in}
\paragraph{Token-wise power iteration.}
Under ZCM, we show that the optimal Kronecker factors for a structured \gls{fim} approximation can be obtained via a token-wise power-iteration procedure that avoids cross-token mixing while preserving within-token activation-gradient dependence. 
To express this approximation compactly, let $\mathbf{X} \!=\! [\mathbf{x}_1, \dots, \mathbf{x}_T]^\top \!\in\! \mathbb{R}^{T \times n}$ and
$\mathbf{G} \!=\! [\mathbf{g}_1, \dots, \mathbf{g}_T]^\top \!\in\! \mathbb{R}^{T \times m}$
be the matrices formed by stacking the $T$ token-wise vectors.

\begin{theorem}[Token-Wise Structured \gls{fim} Approximation]
\label{thm:zcm-ours}
Under Assumption~\ref{as:zcm-assumption}, one power-iteration step from identity initialization yields the approximation:
\[
\boxed{
\mathbf{F} \approx
\frac{
\mathbb{E}\!\left[ \mathbf{X}^\top\operatorname{diag}(\mathbf{G}\mathbf{G}^\top)\mathbf{X} \right]
\otimes
\mathbb{E}\!\left[ \mathbf{G}^\top\operatorname{diag}(\mathbf{X}\mathbf{X}^\top)\mathbf{G} \right]
}{
\mathbb{E}\!\left[\operatorname{tr}\!\left(\mathbf{G}^\top\operatorname{diag}(\mathbf{X}\mathbf{X}^\top)\mathbf{G}\right)\right]
}.
}
\]
\end{theorem}

\begin{proof}
Using the rearrangement operator $R(\cdot)$ (cf.\ Shampoo$^2$~\cite{shampoo-squared}),
\begin{multline}
R\!\left(\mathbf{F}\right)
  = R\!\bigl(\operatorname{vec}(\mathbf{G}^\top \mathbf{X})\,\operatorname{vec}(\mathbf{G}^\top \mathbf{X})^\top\bigr) \\
  = \mathbb{E}\!\left[(\mathbf{G}^\top \mathbf{X}) \otimes (\mathbf{G}^\top \mathbf{X})\right]
\end{multline}

A single power-iteration step for the right factor, $\mathbf{B}$, starting from $\mathbf{A}_0 = \mathbf{I}$, is given by $\mathbf{B} \propto \mathbb{E}[\mathbf{G}^\top \mathbf{X} \mathbf{A}_0 \mathbf{X}^\top \mathbf{G}] = \mathbb{E}[\mathbf{G}^\top \mathbf{X} \mathbf{X}^\top \mathbf{G}]$.
We expand this term using definitions of $\mathbf{G}$ and $\mathbf{X}$:
\begin{multline}
\mathbb{E}\!\left[\mathbf{G}^\top \mathbf{X}\mathbf{X}^\top \mathbf{G}\right]
= \mathbb{E}\!\left[\Big(\sum_{t=1}^T \mathbf{g}_t \mathbf{x}_t^\top\Big)
                \Big(\sum_{s=1}^T \mathbf{x}_s \mathbf{g}_s^\top\Big)\right] \\
= \mathbb{E}\!\left[\sum_{t,s}(\mathbf{g}_t\mathbf{x}_t^\top)(\mathbf{x}_s\mathbf{g}_s^\top)\right].
\end{multline}
Under Assumption~\ref{as:zcm-assumption}, all cross-token terms ($t\neq s$) vanish in expectation. This leaves only the diagonal terms ($t=s$):
\begin{equation}
    \mathbf{B}
\propto \sum_{t=1}^{T}\mathbb{E}\!\left[(\mathbf{g}_t\mathbf{x}_t^\top)(\mathbf{x}_t\mathbf{g}_t^\top)\right]
= \sum_{t=1}^{T}\mathbb{E}\!\left[\|\mathbf{x}_t\|_2^2\,\mathbf{g}_t\mathbf{g}_t^\top\right].
\end{equation}
The sum over $T$ is equivalent to the matrix form $\mathbb{E}\!\left[\mathbf{G}^\top\operatorname{diag}(\mathbf{X}\mathbf{X}^\top)\mathbf{G}\right]$, since the $t$-th diagonal entry of $\operatorname{diag}(\mathbf{X}\mathbf{X}^\top)$ is $\mathbf{x}_t^\top\mathbf{x}_t = \|\mathbf{x}_t\|_2^2$.
By symmetry, the multiplication from the left yields
$\mathbb{E}[\mathbf{X}^\top \operatorname{diag}(\mathbf{G}\mathbf{G}^\top)\mathbf{X}]$.

\noindent Normalizing by the expected trace gives the final Kronecker approximation:
\[
\mathbf{F} \approx
\frac{
\mathbb{E}\!\left[\mathbf{X}^\top \operatorname{diag}(\mathbf{G}\mathbf{G}^\top)\mathbf{X}\right]
\otimes
\mathbb{E}\!\left[\mathbf{G}^\top \operatorname{diag}(\mathbf{X}\mathbf{X}^\top)\mathbf{G}\right]
}{
\mathbb{E}\!\left[\operatorname{tr}\!\left(\mathbf{G}^\top \operatorname{diag}(\mathbf{X}\mathbf{X}^\top)\mathbf{G}\right)\right]
}.
\]
\end{proof}
\vspace{-0.2in}
\paragraph{Intuition.} The diagonal operators $\operatorname{diag}(\mathbf{G}\mathbf{G}^\top)$ and $\operatorname{diag}(\mathbf{X}\mathbf{X}^\top)$ reweight tokens by gradient and activation energy, respectively.
This enforces token-local aggregation by suppressing cross-token mixed moments, while preserving within-token activation-gradient dependence through reweighting itself.
The full algorithm is provided in the Supplement \cref{supp:sec:method_pseudo_code}.
\subsection{\glslink{oursearch}{Constrained Rank Search (CoRS)}\glsunset{oursearch}}\label{sec:constrained-rank-search}
While \gls{ours} is designed to improve compression on individual layers, models contain layers with varying sensitivity. Uniform truncation therefore leads to suboptimal global performance~\cite{yuan2024asvd, 2025_flar-svd}. To distribute the budget, \gls{oursearch} frames rank allocation as a constrained optimization problem comprising three stages.
\vspace{-0.1in}
\paragraph{Stage 1: Error Profiling.}
Like prior works~\cite{yuan2024asvd, 2025_flar-svd} we quantify each layer’s cost-error relationship. For every linear layer $i$, we evaluate the reconstruction error $e_{ij}$ and computational cost $f_{ij}$ for a set of candidate ranks based on compression rates $j \in J$. The error is obtained as a KL-divergence measurement between the original model output and the model output with layer $i$ compressed to compression rate $j$. Using the output as reference ensures that the downstream impact is quantified and comparable between layers. Crucially, to allow the optimizer to skip sensitive layers, the uncompressed configuration $(f_{i,\mathrm{orig}}, 0)$ with zero error is always included. This yields a discrete cost-error profile $\mathcal{D}_i=\{(f_{ij}, e_{ij})\}_{j=1}^{J}$.
\vspace{-0.1in}
\paragraph{Stage 2: Granular Extension.}
To minimize profiling overhead, the initial points $\mathcal{D}_i$ are sparse. We obtain a denser approximation $\tilde{\mathcal{D}}_i$ via interpolation. 
Since a single global spline can introduce undesirable artifacts such as overshoot on sparse cost-error profiles~\cite{deboor1978practical, fritsch1980monotone}, we use a rolling-window cubic interpolation.
Specifically, we fit cubic curves to overlapping point triplets.
This yields a smooth, artifact-free profile 
$\tilde{\mathcal{D}}_i = \{(\tilde{f}_{ij}, \tilde{e}_{ij})\}_{j=1}^{n_i}$. Further algorithmic details are provided in Supplement \cref{supp:sec:interpolation_details}.
\vspace{-0.1in}
\paragraph{Stage 3: Constrained Optimization.}
By formulating the rank allocation as a \gls{milp}~\cite{nemhauser1988integer}, we can solve the rank allocation problem to proxy optimality.
Let $B_{\text{\glspl{flop}}}$ denote the global budget and $x_{ij} \!\in\! \{0,1\}$ be a binary decision variable where $x_{ij}\!=\!1$ indicates configuration $j$ is set for layer $i$. The optimization is formulated as:
\vspace{-0.1in}
\begin{align}
\min_{{x_{ij}}} \quad & \sum_{i=1}^{L} \sum_{j=1}^{n_i} e_{ij} \cdot x_{ij} \label{eq:cors_opt} \\
\text{s.t.} \quad & \sum_{i=1}^{L} \sum_{j=1}^{n_i} f_{ij} \cdot x_{ij} \leq B_{\text{FLOPs}}, \quad \sum_{j=1}^{n_i} x_{ij} = 1 \;\; \forall i. \nonumber
\end{align}
Eq.~\ref{eq:cors_opt} minimizes the total reconstruction error. The constraints below enforce a \gls{flop} budget and ensure that exactly one rank choice is active per layer to maintain a valid result. When solving with standard \gls{milp} solvers~\cite{john_forrest_2024_13347261_cbc}, this yields a proxy-optimal solution.

%% file: figures/method_figure.tex
\begin{figure*}[t]%
    \centering
    \subcaptionbox{
    Cross-token coupling diagnostic on DeiT linear layer.
    \label{fig:motivation:b}}[0.24\textwidth]{%
        \centering
        \begin{tikzpicture}
            \node[anchor=south west, inner sep=0] (containerimg) at (0,0) {%
                \includegraphics[width=\linewidth]{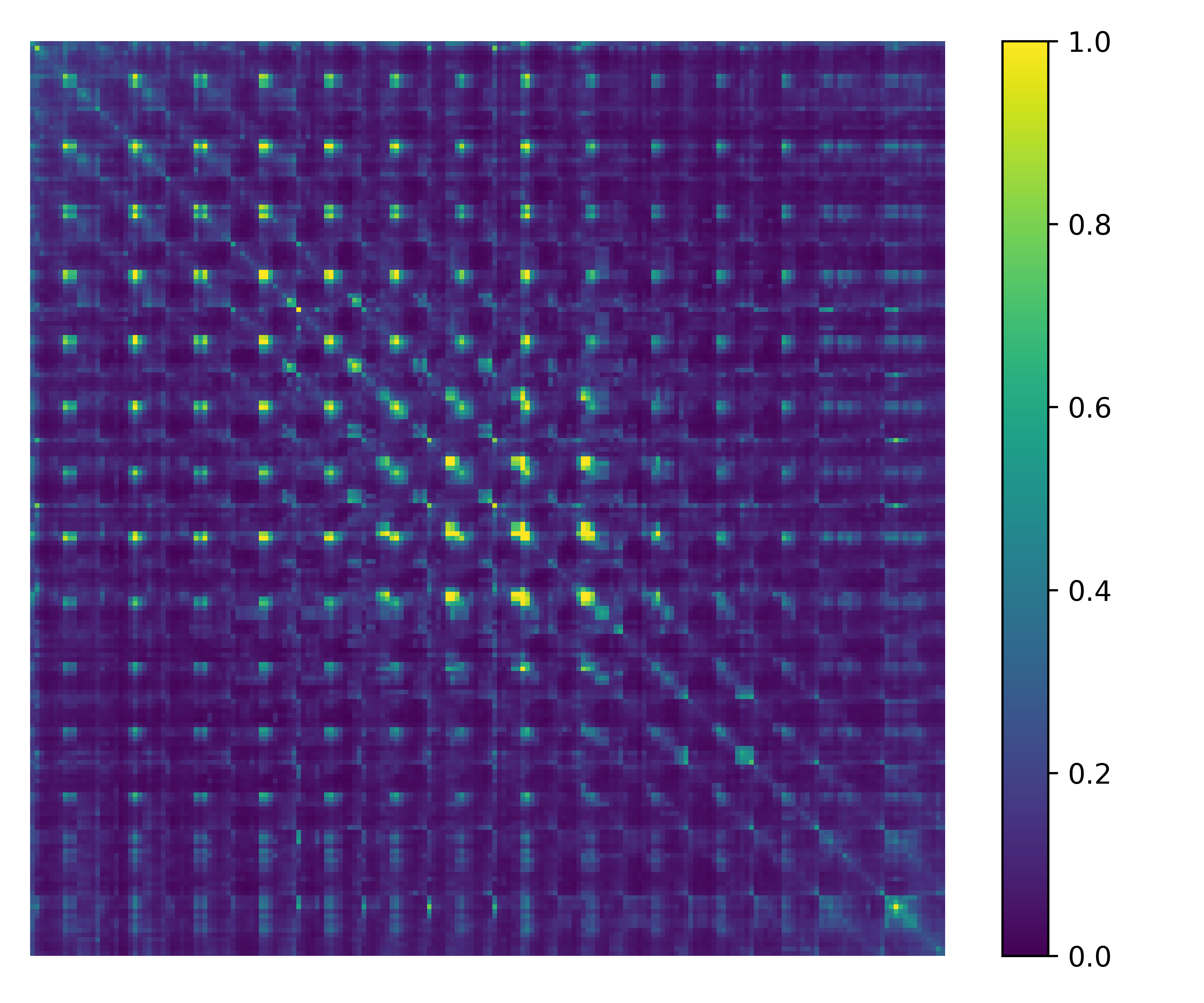}
            };
            \node[anchor=south west, text=white, font=\sffamily\tiny, inner sep=1pt] 
                  at ([xshift=2pt, yshift=3.0pt] containerimg.south west) 
                  {DeiT-B};
        \end{tikzpicture}
            }%
    \hspace{0.02\textwidth}
    \subcaptionbox{
    Input activation-output grad. distance correlation on DeiT.
    \label{fig:motivation:c}}[0.24\textwidth]{%
        \centering
        \resizebox{\linewidth}{!}{%
            \input{figures/motivation_figure_correlation}
        }%
    }
    \hspace{0.02\textwidth}
    \subcaptionbox{
    \gls{fim} similarity vs. post-compression accuracy under different structural assumptions.
    \label{fig:motivation:a}
    }[0.44\textwidth]{%
        \centering
        \resizebox{\linewidth}{!}{%
            \begin{tabular}{lcccc}
                \toprule
                \toprule
                Approach & \makecell{token-local\\aggregation} & \makecell{act-grad\\coupling} & cos-sim & top-1 \\
                \midrule
                Shampoo$^2$   & \xmark & \cmark & 0.17 & 77.2 \\
                KFAC-red.     & \xmark & \xmark & 0.13 & 67.0\\
                KFAC-exp.     & \cmark & \xmark & 0.12 & 75.9\\
                \midrule
                \textbf{Research gap} & \cmark & \cmark & - & -\\
                \bottomrule
                \bottomrule
            \end{tabular}%
        }
        \vspace{+0.12in}
    }%
    \caption{
    Structural axes motivating FACTS for DeiT-B compression of linear layers. (a) Cross-token coupling diagnostic showing structured mixed moments in the empirical \gls{fim}. (b) Distance correlation showing activation-gradient dependence. (c) Despite these structures, higher cosine similarity to the empirical \gls{fim} does not necessarily yield higher post-compression accuracy, motivating token-local aggregation with activation-gradient coupling.
    }%
    \label{fig:motivation}%
    \vspace{-0.1in}
\end{figure*}

%% file: figures/motivation_figure_correlation.tex
\pgfplotstableread[col sep=comma]{figures/data/dcorr_results.csv}\dcorrdata

\begin{tikzpicture}
\begin{axis}[
    width=\linewidth,
    height=4.2cm,
    xlabel={Token index},
    ylabel={Distance correlation},
    grid=both,
    grid style={gray!15},
    tick style={black},
    ymin=0.1,
    ymax=0.8,
    xmin=-5,
    xmax=205,
    line width=1pt,
    tick label style={font=\scriptsize},
    label style={font=\scriptsize},
    title style={font=\small},
    ylabel style={yshift=-0.4cm},
    xlabel style={yshift=+0.5cm},
    xticklabels={0, 0 , , , ,200},
]
\addplot[
    only marks,
    mark=x,
    mark size=.5pt,
    color=blue,
    fill=blue,
    forget plot,
] table [x=token, y=dcorr, filter discard warning=false, row predicate/.code={
    \pgfplotstablegetelem{#1}{significant}\of{\dcorrdata}
    \ifx\pgfplotsretval\empty
        \pgfplotstablerowfalse
    \else
        \edef\temp{\pgfplotsretval}
        \ifx\temp true\relax
            \pgfplotstablerowtrue
        \else
            \pgfplotstablerowfalse
        \fi
    \fi
}] {\dcorrdata};

\addplot[dashed, black, thick] coordinates {(-10,0.47) (220,0.47)};
\node[anchor=west, black] at (axis cs:100,0.57) {\scriptsize Mean};

\end{axis}
\end{tikzpicture}

%% file: sec/4_experiments.tex
\section{Evaluation}
\label{sec:evaluation}
We evaluate the performance of \gls{ours} and  \gls{oursearch}, and assess performance gains and hardware speedups.\ Additional details and experiments are provided in the Supplement.

\subsection{Experimental Setup}
\paragraph{Models and Dataset.}
For the experiments we rely on pretrained checkpoints of \glslink{imagenet}{\acrshort{imagenet}}\glsunset{imagenet} for \glslink{deit}{\acrshort{deit}}\glsunset{deit}~\cite{deit}, \glslink{swin}{\acrshort{swin}}\glsunset{swin}~\cite{swin}, \glslink{convnext}{\acrshort{convnext}}\glsunset{convnext}~\cite{convnext} and \glslink{mamba}{\acrshort{mamba}}\glsunset{mamba}~\cite{mambavision} architectures. \gls{deit}, \gls{swin} and \gls{convnext} are adopted from the \texttt{timm}~\cite{Wightman_PyTorch_Image_Models} library, while for \gls{mamba} we use the official code~\cite{mambavision}. In the few instances where  finetuning is performed, we use a slightly modified recipe of the baseline procedure. Refer to the Supplement \cref{supp:sec:extended_experiment_setup} for more details. We evaluate Top-1 accuracy on the \gls{imagenet}-1k~\cite{imagenet} validation set. Downstream tasks are tested on their respective metrics.
\vspace{-0.1in}
\paragraph{Decomposition.}
To evaluate the effectiveness of our approach, we compare it against six recent decomposition methods: PELA~\cite{pela}, FW-SVD~\cite{hsu2022fwsvd}, ASVD~\cite{yuan2024asvd}, SVD-LLM~\cite{wang2024svd-llm}, GFWSVD~\cite{chekalina2025generalizedfisherweightedsvdscalable}, and FLAR-SVD~\cite{2025_flar-svd}. Although some of these baselines (e.g., SVD-LLM) were originally introduced for large language models, their core optimization of linear maps, e.g. minimizing activation error, is domain-agnostic, making them directly applicable and highly competitive for \glspl{vit}.
We calibrate all methods using the same 16,384 randomly chosen samples from the \gls{imagenet} training set, with batch size 64.\ An ablation on different calibration set sizes is provided in Supplement \cref{supp:sec:calib_set_size}. Similar to prior work~\cite{ chekalina2025generalizedfisherweightedsvdscalable, 2025_flar-svd}, we apply regularization on factors $\mathbf{A}$ and $\mathbf{B}$ to ensure stability. For all approaches we consider only linear layers and exclude the head from compression.\ Additional details are provided in Supplement \cref{supp:sec:extended_experiment_setup}. 
\vspace{-0.1in}
\paragraph{Search.}
For \gls{oursearch}, error measurements are obtained with 512 randomly sampled images from \gls{imagenet} training set, with layers compressed to 5 ratios $J\in \{0.1, 0.3, 0.5, 0.7, 0.9\}$. We use our interpolation to generate 20 additional points between measurements.\ The CBC solver~\cite{john_forrest_2024_13347261_cbc} and the pulp~\cite{pulp} framework are used to optimize \gls{oursearch} and derive layer-wise rank choices in just a few seconds. Other search strategies follow a similar procedure; complete setup descriptions are provided in the Supplement \cref{supp:sec:extended_experiment_setup}.

\input{tables/main_table_imagenet}

\subsection{\gls{ours} Evaluation}
\paragraph{Image Classification Results on \gls{imagenet}.}
To investigate the performance of \gls{ours} we compare it to other \gls{svd} methods across four vision models, including both \gls{deit} and \gls{swin}, as well as other architectures that use weight sharing layers, such as \gls{convnext} and \gls{mamba}.\ 
In the evaluation every model is compressed \textit{uniformly} to the same ratio.\ Additionally, we provide results for \gls{ours} with our proposed \gls{oursearch}.\ As shown in \cref{tab:classification_50}, \gls{ours} achieves improvements over other methods across models.\ Specifically, even without \gls{oursearch} for \gls{deit}-B we improve by \textbf{2.5} \gls{pp}\ over both \acrshort{svd-llm} and FLAR-SVD.\ On \gls{swin}-B, the improvement is even more pronounced, with \gls{ours} outperforming \acrshort{svd-llm} and FLAR-SVD by \textbf{5.8} p.p.\ Interestingly, some approaches get very low accuracy on Swin, likely due to their lack of importance weighting (PELA), too simple diagonal approximations (ASVD, FW-SVD) and conditioning problems (GFWSVD). For \gls{convnext}-B and \gls{mamba}-B, \gls{ours} achieves improvements of \textbf{3.6} and \textbf{1.0} p.p., respectively, over the next-best decomposition.\ Notably, \gls{mamba} is the only architecture where GFWSVD ties 2nd best, matching the performance of \acrshort{svd-llm} and FLAR-SVD. When replacing uniform compression for \gls{ours} with \gls{oursearch}, performance further increases, with improvements ranging from \textbf{3.6} p.p. to \textbf{9.0}~p.p. Additional results for compression rates 40\% and 60\% are provided in Supplement \cref{supp:sec:additional_ratios} and show similar improvements.\ Overall,  without finetuning, \gls{ours} can preserve high accuracy across models, with \gls{oursearch} providing further gains.

\vspace{-0.1in}
\paragraph{Comparisons to other Fisher Approximations.}\label{sec:fisher_aprrox_main_content}
To further evaluate the improvements of \gls{ours}, we compare it against state-of-the-art Fisher-weighted \gls{svd} methods (FW-SVD~\cite{hsu2022fwsvd}, \acrshort{gfwsvd}~\cite{chekalina2025generalizedfisherweightedsvdscalable}) and established Kronecker curvature estimators (KFAC-expand/reduce~\cite{kfac-martens15,runa-kfacreduce}, Shampoo$^2$~\cite{shampoo-squared}) within the \gls{svd} pipeline described in \cref{sec:fw-svd}. Further details are provided in Supplement~\cref{supp:sec:extended_experiment_setup}. All methods are evaluated at uniform compression for \gls{deit}-B and \gls{swin}-B backbones. The results are summarized in \cref{tab:ablation_fishers}. 
Across both \gls{vit} models, the results align with the observation from \cref{sec:fisher_analysis}: global \gls{fim} alignment, measured by cosine similarity, is not a reliable indicator of post-compression performance. On \gls{swin}-B, Shampoo$^2$ achieves the highest cosine similarity (0.177), while \gls{ours} achieves the best Top-1 accuracy outperforming the other baselines (75.9\% vs.\ 76.6\%) despite having the second-lowest cosine similarity (0.127). Similarly, on DeiT-B, \gls{ours} outperforms the baselines despite exhibiting lower cosine similarity than the second best model (Shampoo$^2$).\ 
As the core difference between Shampoo$^2$ and \gls{ours} is the modeling of cross-token moments, the accuracy advantage strongly supports our insight to suppress the cross-token mixed moments. 
\acrshort{gfwsvd} exhibits an additional practical failure mode: although it achieves the highest cosine similarity on DeiT-B (0.174), the resulting factors can be ill-conditioned, making the whitening step numerically unstable in our pipeline and leading to poor post-compression accuracy. This effect is especially severe on Swin-B. 
Overall, by combining token-local aggregation with within-token activation-gradient coupling, \gls{ours} achieves the best Top-1 accuracy on both architectures (+0.3~\gls{pp} on DeiT-B and +0.7~\gls{pp} on Swin-B over the next best method), without achieving the highest cosine similarity. These results support that, for Fisher-weighted \gls{svd} compression, accurately approximating the \gls{fim} does not guarantee high post-compression accuracy.
\input{tables/ablation_fishers}
\vspace{-0.02in}
\subsection{\gls{oursearch} Evaluation}
We evaluate \gls{oursearch} against representative rank allocation methods across three categories: gradient-based (ComCat), greedy (MemoryViT), and equal-error search (ASVD, FLAR-SVD). To ensure a fair comparison, all methods utilize \gls{ours} for the underlying decomposition (detailed in \cref{supp:sec:extended_experiment_setup}). As shown in \cref{tab:search_approach_comparison}, \gls{oursearch} achieves the highest accuracy, outperforming the best baselines on \gls{deit} and \gls{swin} by \textbf{0.6}\,p.p.\ and \textbf{0.5}\,p.p., respectively. While MemoryViT reports the lowest search time by relying on easy-to-obtain matrix energy scores, it normalizes energy across layers. This approach neglects layer-specific sensitivities and results in lower overall accuracy. The equal-error strategies, ASVD and FLAR-SVD, yield strong results. They are 2nd on \gls{deit} and \gls{swin}, but remain less effective and slower than \gls{oursearch}. Finally, ComCat is the least efficient (requiring 40 min) and ranks last on \gls{deit} while offering limited flexibility. Overall, \gls{oursearch} provides the best trade-off between post-compression accuracy and computational efficiency.
\input{tables/search_table}

\subsection{Ablations}
\paragraph{Hardware Acceleration.}
Our previous experiments demonstrate that \gls{ours} and \gls{oursearch}, achieve excellent compression-performance trade-offs.\ However, while the \glspl{flop} savings are a significant advantage, previous work~\cite{runddontwalk} established that theoretical \glspl{flop} reduction does not necessarily translate to practical acceleration on parallel hardware, which is critical for real-world applications.\ 
To quantify the throughput speedups achieved with \gls{ours}, we benchmark the 50\% \gls{svd}-compressed \gls{deit}-B and \gls{swin}-B models on a range of different hardware targets, including Nvidia V100, A100, and H100 GPUs, as well as an Intel Xeon E5-2698 v4 CPU. The results are presented in \cref{fig:throughput}.\ The \gls{svd}-compressed models achieve consistent speedups of \textbf{1.6x} for \gls{deit}-B and \textbf{1.4x} for \gls{swin}-B across all evaluated hardware\ platforms, demonstrating practical throughput accelerations with low impact on task performance making it a great choice for real-world applications.\ We provide additional inference comparisons, including a direct comparison to semi-structured pruning approaches, in Supplement~\cref{supp:sec:extended_throughput_eval}.\

\vspace{-0.1in}
\paragraph{Comparison to Structured Pruning.}
We compare \gls{ours} against structured pruning, which offers hardware acceleration but incurs substantial architectural modifications and initial accuracy loss. We benchmark against three strong baselines on \gls{deit}-B: NViT~\cite{Yang_2023_Nvit}, isomorphic pruning~\cite{fang2024isomorphicpruning}, and DISP-LLM~\cite{gao2024disp} (\cref{tab:pruning_comparisons}). For fairness against NViT (which optimizes during pruning), we report results for the other methods after 5 epochs of finetuning (adding $\approx$1 hour/epoch).
Isomorphic pruning (fast, one-shot Taylor-based) suffers a massive accuracy drop to 25.9\% post-pruning. Conversely, NViT's iterative approach maintains high accuracy (79.2\%) but is computationally expensive. DISP-LLM (gradient-based) retains better accuracy (40.1\%) than isomorphic pruning but incurs higher latency due to indexing overhead. After five finetuning epochs, isomorphic pruning and DISP-LLM recover to 80.4\% and 79.8\% respectively, outperforming NViT.
\gls{ours} outperforms all methods in both scenarios. Without finetuning, it matches isomorphic pruning's speed while achieving \textbf{54.0}\,p.p.\ higher accuracy. With just one finetuning epoch, \gls{ours} reaches near-baseline accuracy (within 0.4\,p.p.), exceeding the 2nd-best by \textbf{1.0}\,p.p.\ with a \textbf{4x} smaller compute budget, showing its high efficiency.
\begin{table*}[t]
    \centering
    \begin{minipage}[t]{0.43\textwidth}
        \centering
        \caption{Comparison of our \gls{ours} (w/ \gls{oursearch}) to structured pruning approaches on DeiT-B. The first part is w/o finetuning; the second is w/ finetuning for X epochs ($\dagger$:X).}
        \label{tab:pruning_comparisons}
        \resizebox{\linewidth}{!}{%
            \begin{tabular}{l c c c c}
            \toprule
            \midrule
            Method & MParams & Time [h] & Top-1$\uparrow$ & Lat. [ms]\\
            \midrule
            Base~\cite{deit} & 87.3 & - & 83.3 & 58.8 \\
            \cmidrule{1-5}
            DISP-LLM~\cite{gao2024disp} & 50.1 & 1.2 & \underline{40.1} & 42.7 \\
            Isomorphic~\cite{fang2024isomorphicpruning} & 50.6 & \textbf{0.1} & 25.9 & \underline{40.5} \\
            \rowcolor{m} \textbf{\gls{ours}} & 49.2 & \underline{0.2} & \textbf{79.9} & \textbf{40.4}\\
            \cmidrule{1-5}
            DISP-LLM$^{\dagger: 5}$\cite{gao2024disp} & 50.1 & 6.2 & 79.8 & 42.7 \\
            Isomorphic$^{\dagger: 5}$~\cite{fang2024isomorphicpruning} & 50.6 & \underline{5.1} & \underline{80.4} & \underline{40.5} \\
            NViT~\cite{Yang_2023_Nvit} & 53.6 & 7.3 & 79.2 & 40.6 \\
            \rowcolor{m} \textbf{\gls{ours}}$^{\dagger: 1}$ & 49.2 & \textbf{1.2} & \textbf{81.4} & \textbf{40.4}\\
            \bottomrule
            \bottomrule
            \end{tabular}
        }
    \end{minipage}
    \hfill
    \begin{minipage}[t]{0.55\textwidth}
        \centering
        \caption{Top-1 for the best SVD approaches on large (L) and small (S) variants of the used models \textbf{uniformly} compressed to 50\% remaining size. Final column is \gls{ours} with \gls{oursearch} applied denoted as +s. }
        \label{tab:size_comparisons}
        \resizebox{\linewidth}{!}{%
            \begin{tabular}{l|c|c|c|c}
            \toprule
            \midrule
            \multirow{1}{*}{Model} & \multicolumn{1}{c|}{GFWSVD} & \multicolumn{1}{c|}{FLAR} & \multicolumn{1}{c|}{\cellcolor{m}\textbf{\ours}} & \multicolumn{1}{c}{\cellcolor{m}\textbf{\ours} \textbf{+s}} \\
            \midrule
            Swin-L~\cite{swin} & 66.4 & 76.8 & \cellcolor{m}\underline{78.9} & \cellcolor{m}\textbf{83.0} \\
            ConvNext-L~\cite{convnext} & 66.3 & 80.5 & \cellcolor{m}\underline{81.9} & \cellcolor{m}\textbf{83.9} \\
            MambaVis.-L~\cite{mambavision} & 77.8 & 79.5 & \cellcolor{m}\underline{80.5} & \cellcolor{m}\textbf{83.2} \\
            \midrule
            DeiT-S~\cite{deit} & 15.2 & 37.8 & \cellcolor{m}\underline{51.1} & \cellcolor{m}\textbf{62.2} \\
            Swin-S~\cite{swin} & 2.9 & 49.4 & \cellcolor{m}\underline{56.4} & \cellcolor{m}\textbf{67.6} \\
            ConvNext-S~\cite{convnext} & 23.4 & 64.5 & \cellcolor{m}\underline{69.2} & \cellcolor{m}\textbf{77.5} \\
            MambaVis.-S~\cite{mambavision} & 66.7 & 65.0 & \cellcolor{m}\underline{67.4} & \cellcolor{m}\textbf{76.5} \\
            \midrule
            \bottomrule
            \end{tabular}
        }
    \end{minipage}
\end{table*}
\vspace{-0.1in}
\paragraph{Scaling to Small and Large Models.}
To further investigate the generalization across different model sizes, we evaluate the large (L) and small (S) variants of all models. \cref{tab:size_comparisons} shows the results for the highest performing variants in the main table. Since there is no L version of DeiT, it is not included in the table. The results demonstrate the remarkable scalability of our approach across both small and large variants of the tested architectures, achieving improvements of up to +13.3 p.p. over the second-best baseline on the small models and +2.1 p.p. on the large models \textit{without search}, and +24.4 p.p. and +6.2 p.p. with search on the small and large models, respectively.
\vspace{-0.1in}
\input{figures/throughput_downstream_mixed_caption_pos}
\paragraph{What about LLMs?}
To determine whether \gls{ours}'s performance improvements are specific to vision, we run a brief exploratory study on a LLM. Qualitatively, the cross-token moments of Qwen3-1.7B exhibit less spatial regularity than those in vision models (cf. Supplement). This aligns with inherent modality differences: unlike image patches, language tokens lack fixed geometric positions, making their dependencies highly sequence- and context-driven. Despite this structural difference, \gls{ours} continues to perform well on WikiText perplexity and zero-shot language benchmarks. 
While this provides some indication of the performance in the language domain, addressing the complexity and challenges that larger LLMs introduce goes beyond the scope of this work. 
We refer to contemporary work \cite{thoma2026advancing} that specifically addresses these challenges. Additional evaluation details, figures and tables are provided in Supplement Section~\ref{supp:subsec:LLM_eval_details} and Section~\ref{supp:sec:llm_checks}.
\vspace{-0.1in}
\paragraph{Downstream Tasks.}
We evaluate additional downstream tasks (object detection and instance segmentation) on the \glslink{coco}{COCO}\glsunset{coco} 2017 dataset~\cite{COCO}.\ We use a Mask R-CNN~\cite{maskrcnn} model built with a \gls{swin}-B backbone and compare against \acrshort{svd-llm}~\cite{wang2024svd-llm} (most competitive \gls{svd} baseline) and \glslink{pela}{PELA}\glsunset{pela}~\cite{pela} (reference downstream pipeline for \gls{svd}-compressed backbones).\ 
We follow \gls{pela}'s setup in compressing only the backbone.\ However, whereas \gls{pela}'s approach requires a 50-epoch finetuning followed by downstream training, our method is evaluated using a zero-shot paradigm, leveraging \gls{svd}'s ability to serve as a drop-in layer replacement without introducing dependency issues or significant task performance degradation.
\gls{ours} significantly outperforms \acrshort{svd-llm}, and notably, even ties the fully retrained \gls{pela}, while avoiding their costly multi-stage retraining.\ Moreover, a single epoch of finetuning (1ep.FT) can effectively close the small remaining gap to the uncompressed baseline and surpass the converged \gls{pela}.\ The results are summarized in~\cref{tab:down_det}.
For semantic segmentation with DeiT based UPerNet on ADE20k, results follow similar trends. Details and full semantic segmentation results are provided in Appendix \ref{supp:sec:extended_experiment_setup} and \ref{supp:sec:downstream_semseg}, respectively.

%% file: tables/main_table_imagenet.tex
\begin{table*}[t]
    \centering
    \caption{Top-1 accuracy for various models on \gls{imagenet} at \textbf{uniform} compression (50\% linear layers FLOPs remaining) for related \gls{svd} methods. In addition to uniform compression, we report \gls{ours} with our \textbf{\gls{oursearch}} to isolate the effect of decomposition and search.}
    \resizebox{0.96\textwidth}{!}{%
    \begin{tabular}{l  c c   c c   c c   c c }
    \toprule
    \midrule
    \multirow{2}{*}{Method} & \multicolumn{2}{c}{\textbf{DeiT-B}} & \multicolumn{2}{c}{\textbf{Swin-B}} & \multicolumn{2}{c}{\textbf{ConvNeXt-B}} & \multicolumn{2}{c}{\textbf{MambaVis.-B}} \\
    & GFLOPs & Top-1$\uparrow$ & GFLOPs & Top-1$\uparrow$ & GFLOPs & Top-1$\uparrow$ & GFLOPs & Top-1$\uparrow$ \\
    \midrule
    Baseline & 33.9 & 83.3 & 30.3 & 85.1 & 30.8 & 85.8 & 29.9 & 83.9 \\
    \cmidrule{2-9}
    PELA~\cite{pela} & 17.1 & 66.8 & 15.5 & 11.7 & 15.9 & 19.7 & 21.5 & 69.5 \\
    FW-SVD~\cite{hsu2022fwsvd} & 17.1 & 73.0 & 15.5 & 28.5 & 15.9 & 42.7 & 21.5 & 68.2 \\
    ASVD~\cite{yuan2024asvd} & 17.1 & 72.3 & 15.5 & 25.5 & 15.9 & 49.4 & 21.5 & 69.7 \\
    SVD-LLM~\cite{wang2024svd-llm} & 17.1 & 75.0 & 15.5 & 60.1 & 15.9 & 72.2 & 21.5 & 71.3 \\
    GFWSVD~\cite{chekalina2025generalizedfisherweightedsvdscalable} & 17.1 & 71.0 & 15.5 & 16.8 & 15.9 & 18.7 & 21.5 & 71.3 \\
    FLAR-SVD~\cite{2025_flar-svd} & 17.1 & 75.0 & 15.5 & 60.1 & 15.9 & 72.2 & 21.5 & 71.3 \\
    \rowcolor{m} \textbf{\gls{ours}} & 17.1 & \underline{77.5} & 15.5 & \underline{65.9} & 15.9 & \underline{75.8} & 21.5 & \underline{72.3} \\
    \rowcolor{m} \textbf{\gls{ours}} $+$ \textbf{\gls{oursearch}} & 17.1 & \textbf{81.3} & 15.5 & \textbf{74.9} & 15.9 & \textbf{79.4} & 21.5 & \textbf{79.7} \\
    \midrule
    \bottomrule
    \end{tabular}
    }
    \label{tab:classification_50}
\end{table*}

%% file: tables/ablation_fishers.tex
\begin{table*}[t]
\centering
\captionof{table}{Comparison of Kronecker \gls{fim} estimators for Fisher-weighted SVD compression at matched GFLOPs on DeiT-B and Swin-B. We report cosine similarity to the empirical \gls{fim} and post-compression Top-1, along with whether each method uses token-local aggregation and preserves activation-gradient coupling. All results are obtained without search.
}
\resizebox{.96\textwidth}{!}{%
\begin{tabular}{l c c  c c c  c c c}
\toprule
\midrule
\multirow{2}{*}{Method} & \multirow{2}{*}{\makecell{token-local\\aggregate}} & \multirow{2}{*}{\makecell{act-grad.\\coupling}} & \multicolumn{3}{c}{\textbf{DeiT-B}\cite{deit}} & \multicolumn{3}{c}{\textbf{Swin-B}\cite{swin}} \\\cmidrule{4-9}
&  & & GFLOPs & cos-sim$\uparrow$ & Top-1$\uparrow$ & GFLOPs & cos-sim$\uparrow$ & Top-1$\uparrow$\\
\midrule
Baseline & - & - & 33.9 & - & 83.3 & 30.3 & - & 85.1\\
\cmidrule{2-9}
FW-SVD~\cite{hsu2022fwsvd} & - & - & 17.1 & - & 73.0 & 18.4 & - & 61.0\\
GFWSVD~\cite{chekalina2025generalizedfisherweightedsvdscalable} & \xmark & \cmark & 17.1 & \textbf{0.174} & 71.0 & 18.4 & 0.117 & 52.2\\
\cmidrule{2-9}
Shampoo$^2$~\cite{shampoo-squared} & \xmark & \cmark & 17.1 & \underline{0.166} & \underline{77.2} & 18.4 & \textbf{0.177} & \underline{75.9}\\
KFAC-red.~\cite{runa-kfacreduce} & \xmark & \xmark & 17.1 & 0.129 & 67.0 & 18.4 & \underline{0.138} & 56.5\\
KFAC-exp.~\cite{kfac-martens15}  & \cmark & \xmark & 17.1 & 0.119 & 75.9 & 18.4 & 0.103 & 75.3\\
\rowcolor{m}\textbf{\ours}         & \cmark & \cmark & 17.1 & 0.128 & \textbf{77.5} & 18.4 & 0.127 & \textbf{76.6}\\
\bottomrule
\bottomrule
\end{tabular}
}
\label{tab:ablation_fishers}
\end{table*}

%% file: tables/search_table.tex
\begin{table}[t]
\centering
\captionof{table}{Comparison between our search and other SVD focused searches, comparing resulting top-1 accuracy on ImageNet1k for a fixed number of FLOPs and time to search (tts).
}
\resizebox{1\linewidth}{!}{%
\begin{tabular}{c|ccc|ccc}
\midrule
\multirow{2}{*}{\textsc{Method}} & \multicolumn{3}{c}{\textsc{DeiT-B}} & \multicolumn{3}{c}{\textsc{Swin-B}} \\
& GFLOPs & Top-1$\uparrow$ & tts [min]$\downarrow$ & GFLOPs & Top-1$\uparrow$ & tts [min]$\downarrow$\\
\midrule
Baseline & 33.7 & 81.8 & - & 30.3 & 85.1 & -\\
\midrule
\textbf{uniform} & 17.0 & 71.5 & - & 18.4 & 76.4 & -\\
ASVD & 17.0 & 79.1 & 19.9 & 18.4 & 80.2 & 26.5\\
FLAR-SVD & 17.5 & 79.2 & 52.7 & 19.4 & 80.8 & 110.0\\
MemViT & 17.0 & 78.1 & \textbf{0.1} & 18.4 & 79.0 & \textbf{0.1}\\
ComCat & 17.2 & 76.2 & 40.0 & n.a. & n.a. & n.a.\\
\rowcolor{m}\textbf{\gls{oursearch}} & 17.0 & \textbf{79.8} & 6.6 & 18.4 & \textbf{81.3} & 15.0\\
\bottomrule
\end{tabular}
}
\label{tab:search_approach_comparison}
\end{table}

%% file: figures/throughput_downstream_mixed_caption_pos.tex
\begin{figure*}[t]
    \centering
    
    \pgfplotsset{
        bar_chart_style/.style={
            ybar,
            width=\linewidth,
            height=3.95cm,
            ymin=0,
            ymax=2300,
            bar width=7pt,
            symbolic x coords={V100, A100, H100, CPU},
            xtick=data,
            nodes near coords,
            nodes near coords align={vertical},
            every node near coord/.append style={font=\tiny, text=black, yshift=0.5pt},
            tick label style={font=\tiny},
            ytick={0, 500, 1000, 1500, 2000},
            ylabel style={font=\scriptsize, yshift=-0.05cm},
            ymajorgrids=true,
            grid style=dashed,
            legend style={
                draw=none, 
                font=\tiny, 
                legend columns=1, 
                at={(0.02, 0.98)}, 
                anchor=north west
            },
        }
    }
    \begin{minipage}[t]{0.34\linewidth} 
        \vspace{0pt}
        \centering
        \begin{tikzpicture}
            \begin{axis}[
                bar_chart_style,
                ylabel={Throughput (img/s)},
                enlarge x limits=0.25,
                ylabel style={font=\scriptsize, yshift=-0.25cm}
            ]
                \addplot[fill=DarkBlue, draw=black, nodes near coords={}] 
                    coordinates {(V100,350) (A100,503) (H100,1269) (CPU,30.7)};

                \addplot[fill=DarkGreen, draw=black,
                    point meta rel=per plot,
                    visualization depends on={rawy/\thisrow{base} \as \ratio},
                    nodes near coords={\pgfmathprintnumber[fixed,precision=1]{\ratio}x},
                    every node near coord/.append style={font=\tiny, text=black, anchor=south, inner sep=1pt},
                ] table[row sep=\\, meta=base] {
                    x      y      base  \\
                    V100   568    350   \\
                    A100   851    503   \\
                    H100   2005   1269  \\
                    CPU    47.7   30.7  \\
                };
                \legend{Baseline, Ours}
            \end{axis}
        \end{tikzpicture}
        
        \vspace{0.cm}
        \caption{Throughput (DeiT-B) on different platforms showing acceleration of \gls{svd}.
        }
        \label{fig:throughput}
    \end{minipage}%
    \hfill%
    \begin{minipage}[t]{0.64\linewidth}
        \vspace{0pt}
        \centering
        \makeatletter\def\@captype{table}\makeatother 
        \caption{Mask R-CNN with Swin-B object detection and instance segmentation results for different \gls{svd} compression approaches on the \gls{coco} dataset.}
        \label{tab:down_det}
        
        \vspace{0.15cm}
        \resizebox{\linewidth}{!}{%
        \begin{tabular}{l c c c c c c c}
        \toprule
        \toprule
        Method & $mAP^b$ & $mAP^b_{50}$ & $mAP^b_{75}$ & $mAP^m$ & $mAP^m_{50}$ & $mAP^m_{75}$ & GFLOPs \\
        \midrule
        Baseline & 46.6 & 68.6 & 51.3 & 42.6 & 65.9 & 46.2 & 358.5 \\
        \cmidrule{2-8}
        PELA~\cite{pela} & \textbf{45.3} & \textbf{67.1} & \textbf{50.2} & 41.3 & 64.1 & 44.6 & 290.3 \\
        SVD-LLM~\cite{wang2024svd-llm} & 42.9 &  64.2 & 47.1 & 39.5 & 61.5 & 42.6 & 290.3 \\
        \rowcolor{m}\textbf{\gls{ours}} & \textbf{45.3} & 67.0 & 49.7 & \textbf{41.6} & \textbf{64.4} & \textbf{45.2} & 290.3  \\
        \rowcolor{m}\textbf{\gls{ours}+1ep.FT} & \textbf{45.9} & \textbf{67.8} & \textbf{50.2} & \textbf{42.1} & \textbf{65.1} & \textbf{45.8} & 290.3  \\
        \bottomrule
        \bottomrule
        \end{tabular}%
        }
    \end{minipage}
\end{figure*}

%% file: sec/5_conclusion.tex
\section{Conclusion and Limitations}
\label{sec:conclusion}

We propose \gls{ours}, a structured Kronecker \gls{fim} approximation for Fisher-weighted SVD compression of \gls{vit}. Derived via token-wise power iteration, \gls{ours} combines token-local aggregation (discarding cross-token mixed moments) with preserved activation-gradient dependence, yielding a curvature model more effective for \gls{svd} compression than standard Kronecker choices. Together with our global rank allocator (\gls{oursearch}), \gls{ours} achieves the best accuracy-efficiency trade-offs across the evaluated architectures, delivering strong accuracy retention with practical inference speedups under fixed \gls{flop} budgets.

Our results also highlight an intriguing empirical phenomenon: cross-token mixed moments can substantially increase \gls{fim} cosine similarity, yet retaining them does not reliably improve Fisher-weighted \gls{svd} compression. Developing a more complete theoretical understanding of when and why these terms are misaligned with compression-induced perturbations is an interesting direction for future work. Finally, we plan to extend the structured-curvature perspective to other compression modalities (e.g., quantization) and to reduce calibration cost by exploring cheaper block- or layer-wise \gls{fim} approximations that lower the resource demands of gradient-based statistics.

%% file: sec/X_suppl.tex
\section{Extended Experimental Setup}\label{supp:sec:extended_experiment_setup}
\input{sec/supplementary/extended_experimental_setup}

\input{figures/supplementary/extended_fim_figure}

\section{Extended Evaluation}\label{sec:supp_extended_eval}
\subsection{Extended Analysis on Model Properties}\label{supp:sec:more_actgrad_and_zcm}
\input{sec/supplementary/extended_emprical_insights}

\subsection{Extended Throughput Analysis}
\input{sec/supplementary/extended_throughput_eval}

\subsection{Semantic Segmentation Results}\label{supp:sec:downstream_semseg}
\input{sec/supplementary/extended_downstream_results}

\subsection{Additional Compression Ratios}\label{supp:sec:additional_ratios}
\input{sec/supplementary/extended_main_tables}

\subsection{Impact of Calibration Set Size}\label{supp:sec:calib_set_size}
\input{sec/supplementary/calibration_data_impact}

\subsection{Individual Contribution to \gls{oursearch} Performance}
The individual contributions of \gls{oursearch}'s components to the overall accuracy and computational cost, for DeiT-B and {Swin-B}, are summarized in \cref{tab:individual_contribution}.\ 
Uniform compression (\gls{ours} without search) is fast, taking only 8.3 min for DeiT-B and 10.9 min for Swin-B on a single V100 GPU. Using ILP-based search yields substantial performance gains, improving accuracy by $+3.8$ p.p. on DeiT-B and $+4.4$ p.p.\ on Swin-B.
However, the gain comes at the cost of sensitivity measurements. A high-resolution search (9 points, $J \in \{0.1, \dots, 0.9\}$) adds 12.7 min (DeiT-B) and 28.3 min (Swin-B) \textit{on top} of the decomposition time.
To reduce this overhead, we limit the search to 5 measurements ($J \in \{0.1, 0.3, \dots, 0.9\}$). This lowers the total runtime to 14.8 min (DeiT-B) and 27.3 min (Swin-B), a reduction of roughly 50\%, at the cost of an accuracy drop ($-0.1$ p.p. on DeiT, $-0.4$ p.p. on Swin). Applying our interpolation recovers this performance, matching or improving accuracy relative to the high-resolution search.
Collectively, \gls{oursearch} attains high accuracy at minimal overhead.\ Moreover, since decomposition factors and sensitivity scores are invariant to the target compression, the search constitutes a \textbf{one-time cost}, with scores reusable across target budgets.
\input{tables/individual_contributions}

\subsection{What about LLMs?}\label{supp:sec:llm_checks}
\input{sec/supplementary/llm_stuff}

\section{Extended Methodology for \gls{ours}}
This section expands the methodology in the main section, presenting the more detailed explanations for the preliminaries, theoretical derivation of our objective, detailed algorithmic pseudocode, and the specific mechanisms of our constrained rank search.
\subsection{Extended Preliminary Explanations}\label{supp:sec:preliminaries}
\input{sec/supplementary/preliminaries}

\subsection{Decomposition (The Fisher-Weighted SVD Objective)}
\input{sec/supplementary/kroneker_fwsvd_closed_form_solution}
\input{sec/supplementary/method_algorithms}

\subsection{Search}
\input{sec/supplementary/search_extensions}

%% file: sec/supplementary/extended_experimental_setup.tex
This section details the implementation specifics, baselines, and evaluation protocols required to reproduce the reported results.

\subsection{Comparison between Kronecker Fisher Approximations}
\label{supp:subsec:fisher_comparison_motivation}

In \cref{sec:fisher_analysis} we compare different Kronecker-factored \gls{fim} estimators to assess how two structural priors (cross-token aggregation (local vs.\ global) and activation-gradient decoupling) affect downstream post-compression accuracy.
Because \gls{fim} statistics can be sensitive to the data distribution, we control both the class balance and the calibration/validation split.
We start from the ImageNet validation set (class-balanced) and sample \textbf{16k} images (16 per class) to estimate Kronecker factors for all methods using a batch size of 64.
We then sample a disjoint \textbf{2k} image set (2 per class) to evaluate cosine similarity and to compute the qualitative diagnostics shown in \cref{fig:motivation:b,fig:motivation:c}. All Kronecker factors are obtained using an identical \gls{svd} pipeline, with only method-specific factor computations changed. For post-compression accuracy, we follow the main experimental protocol (see \cref{supp:subsec:experimental_setup_decomposition}).

\paragraph{Cosine similarity.}
We report an operator cosine similarity between the Kronecker approximated \gls{fim} $\widehat{\mathbf{F}}$ and the empirical \gls{fim} $\mathbf{F}$ using only Fisher-vector products (i.e., without materializing the prohibitively large $\mathbf{F}\in\mathbb{R}^{mn\times mn}$).
Concretely, we view $\mathbf{F}$ and $\widehat{\mathbf{F}}$ as linear operators acting on a probe matrix $\mathbf{V}\in\mathbb{R}^{m\times n}$ (equivalently, on $\mathrm{vec}(\mathbf{V})$).
We draw $K$ i.i.d.\ probe matrices $\{\mathbf{V}_k\}_{k=1}^K$ with zero mean and unit variance entries, and estimate the Frobenius inner product between operators using stochastic trace identities (Hutchinson-style estimators) \cite{hutchinson1989stochastic, avron2011trace}.
In particular, we compute $\mathbf{F}(\mathbf{V}_k)$ from held-out samples and $\widehat{\mathbf{F}}(\mathbf{V}_k)$ from the Kronecker form, and report the normalized alignment
\begin{equation}
\cos(\mathbf{F},\widehat{\mathbf{F}})
\;\approx\;
\frac{\frac{1}{K}\sum_{k=1}^K \langle \mathbf{F}(\mathbf{V}_k),\, \widehat{\mathbf{F}}(\mathbf{V}_k)\rangle_F}
{\sqrt{\left(\frac{1}{K}\sum_{k=1}^K \|\mathbf{F}(\mathbf{V}_k)\|_F^2\right)
       \left(\frac{1}{K}\sum_{k=1}^K \|\widehat{\mathbf{F}}(\mathbf{V}_k)\|_F^2\right)}}.
\end{equation}
We reuse the same probe set $\{\mathbf{V}_k\}$ across all estimators to ensure a fair comparison.

\paragraph{Cross-token structure diagnostic.}
In addition to cosine similarity, we visualize cross-token structure using a diagnostic derived from activations and output gradients.
For each token pair $(t,s)$, we form cross-sample Gram matrices of activations and gradients,
$K_x^{(t,s)}[b,d]=\langle x_{b,t},x_{d,s}\rangle$ and $K_g^{(t,s)}[b,d]=\langle g_{b,t},g_{d,s}\rangle$,
compute their element wise product, and report a normalized Frobenius magnitude $\|K_x^{(t,s)}\odot K_g^{(t,s)}\|_F$
(normalized by the mean diagonal).
Pronounced off-diagonal values indicate strong cross-token coupling. We emphasize that this heatmap is a qualitative diagnostic of token interactions and is \emph{not} itself used as a \gls{fim} approximation.

\paragraph{Activation-gradient coupling.}
To quantify within-token activation-gradient coupling, we compute \emph{distance correlation} (dCor) between activations and output gradients per token on the 2k sample set \cite{szekely2007dcorr}.
For each token index $t$, we collect pairs $\{(x_{b,t}, g_{b,t})\}_{b=1}^{N}$ across samples, and then compute the distance-correlation statistic $\mathrm{dCor}(x_{t}, g_{t})$ using \texttt{Dcorr} of the \texttt{hyppo} library~\cite{panda2019hyppo}.

\subsection{Decomposition}\label{supp:subsec:experimental_setup_decomposition}
Since some official repositories for decomposition methods primarily support Large Language Models (LLMs)~\cite{yuan2024asvd, chekalina2025generalizedfisherweightedsvdscalable, hsu2022fwsvd, wang2024svd-llm}, we reimplement these baselines for vision architectures. We build upon the codebase of \cite{2025_flar-svd}, who have implemented PELA, ASVD~\cite{yuan2024asvd}, SVD-LLM~\cite{wang2024svd-llm}, FW-SVD~\cite{hsu2022fwsvd}, and FLAR-SVD~\cite{2025_flar-svd}. Additionally, we adapt the GFWSVD~\cite{chekalina2025generalizedfisherweightedsvdscalable} method that is originally designed for LLMs to computer vision tasks.
For all comparative approaches, we strictly adhere to the hyperparameter settings reported in our main manuscript. A specific adaptation was necessary for GFWSVD~\cite{chekalina2025generalizedfisherweightedsvdscalable}: due to memory constraints (its iterative method requires \textit{all} gradient samples of a layer to be present in memory at once), we follow the original implementation and use 64-sample average gradients as samples and follow their regularization strategy. 
For all other Kronecker \gls{fim} methods, we ensure computational stability during the calculation of regularization terms $\mathbf{A}$ and $\mathbf{B}$ by regularizing gradients and applying Tikhonov shrinkage.

\subsection{Search} 
\textbf{MemViT:} We reimplement the search algorithm following the original paper~\cite{azizi2024memoryViT}, utilizing a decay rate $\gamma=80$ and executing the iterative procedure for 500 iterations. Regarding the definition of matrix energy, we empirically evaluated both squared and non-squared singular values. We observed that normalizing based on non-squared singular values yielded superior convergence. Hence, we report results using non-squared singular values.
\textbf{ComCat:} We utilize the official implementation~\cite{xiao2023comcat} for DeiT. To ensure a fair comparison with our zero-shot approach, we initialize the model with weights decomposed by our method and execute the ComCat search for 30 epochs without intermediate retraining. However, because the search converges rapidly, our reported configuration was obtained within one epoch, after which it stagnated. Moreover, since authors did not provide any functionality for Swin and the framework is very rigid, no Swin values are reported.
\textbf{ASVD:} We adapt the vision-compatible implementation from the FLAR-SVD repository. We extend the method to support FLOPs-based compression targets (specifically for Swin Transformer) and refine the sensitivity measurement range. While the original work evaluated sensitivities in the $[0.4, 0.9]$ range, we broaden this to $[0.1, 0.9]$ with $0.1$ increments to better capture the compression tolerance of vision models. Furthermore, we align the sensitivity metric with our own by substituting the Cross-Entropy loss with KL-Divergence, which we find to work better.
\textbf{FLAR-SVD:} We utilize the authors provided implementation~\cite{2025_flar-svd}. However, as the original method optimizes the latency as well as possible for a given target error, hitting a specific FLOPs budget may require trial and error. To compare against fixed-FLOPs baselines, we augmented their code with an automated optimization routine designed to strictly target specific FLOPs budgets while using the error target mechanism of the original implementation. 
All implementations and settings are part of the code.

\subsection{Finetuning}
We evaluate finetuning performance using the DeiT model within the structured pruning comparison, following the recipe provided in the official code of \cite{fang2024isomorphicpruning}.
\begin{itemize}
    \item \textbf{Our Method:} To accommodate the constrained training budget of a single epoch, we employ a learning rate of $8.7 \times 10^{-6}$ without warm-up.
    \item \textbf{Structured Pruning Baselines:} As these methods suffer more significant accuracy degradation post-pruning, they require larger weight updates to recover. Consequently, we increase the learning rate to $1.0 \times 10^{-4}$ (no warm-up) as this performed better than using the same as we did for ours.
\end{itemize}
In both scenarios, we replace the standard Cross-Entropy loss with the distillation \textit{soft loss} from the original DeiT work~\cite{deit}, as we found it offers superior stability and performance in short-schedule training regimes.

\subsection{Downstream Tasks}
We assess the transferability of our compressed models on object detection (with instance segmentation) and semantic segmentation using the \texttt{mmdetection}~\cite{mmdetection} and \texttt{mmsegmentation}~\cite{mmseg2020} frameworks.

\paragraph{Datasets \& Models.}
We evaluate object detection on COCO 2017~\cite{COCO} using Mask-RCNN~\cite{maskrcnn} with a Swin-B~\cite{swin} backbone. Semantic segmentation is evaluated on ADE20k~\cite{ade20k} using UPerNet~\cite{xiao2018unified} with both Swin-B and DeiT-B~\cite{deit} backbones.

\paragraph{Baselines.}
Uncompressed baselines are established by training the full downstream models using standard schedules: the 1x schedule (12 epochs) for COCO and 160k iterations for ADE20k. Backbones are initialized with official ImageNet weights (ImageNet-22k $\to$ 1k fine-tuned for Swin-B, ImageNet-1k for DeiT-B). All training is conducted on $4 \times$ H100 GPUs.

\paragraph{Evaluation Protocols.}
Consistent with prior work~\cite{pela}, compression is applied \textit{only} to the backbone.
\begin{enumerate}
    \item \textbf{Zero-Shot (ZS):} Compression is applied directly to the converged downstream checkpoint. Rank configurations are derived via our search procedure, and the model is evaluated without further optimization. For sensitivity calibration, we employ 10 sample images with task-specific metrics: for object detection, we minimize feature map MSE at the FPN output (\texttt{neck.fpn\_convs.3.conv}). For segmentation, we retain the KL-Divergence metric used in classification.
    \item \textbf{Comparison with PELA:} We strictly follow PELA's~\cite{pela} multi-stage protocol: (1) compress the ImageNet-pretrained backbone, (2) retrain the backbone on ImageNet, (3) initialize the downstream model with this retrained backbone, and (4) train the full downstream model from scratch.
\end{enumerate}

\subsection{LLM Compression and Evaluation}\label{supp:subsec:LLM_eval_details}
For the LLM ablation, we compress \texttt{Qwen/Qwen3-1.7B} from huggingface. The model is compressed with uniform ratio for all layers to 0.7 remaining parameters. We calibrate approaches on WikiText2~\cite{merity2017pointer_wikitext2} train (256 sequences, of length
2048), and evaluate perplexity on the test set of WikiText2 with the same sequence length. We further report the average zero-shot accuracy across popular benchmark datasets (piqa~\cite{bisk2020piqa}, openbookqa~\cite{mihaylov2018openbookqa}, hellaswag~\cite{zellers2019hellaswag}, arc-challenge/-easy~\cite{clark2018arc}, winogrande~\cite{sakaguchi2021winogrande}) evaluated using the lm-eval library~\cite{eval-harness}.

%% file: figures/supplementary/extended_fim_figure.tex
\begin{figure*}[ht]
    \centering
    \subcaptionbox{Cross-token coupling diagnostic for Swin. \label{fig:motivation_swin_corr}}{%
        \parbox[b]{0.30\textwidth}{%
            \centering
        \begin{tikzpicture}
            \node[anchor=south west, inner sep=0] (containerimg) at (0,0) {%
                \includegraphics[width=\linewidth]{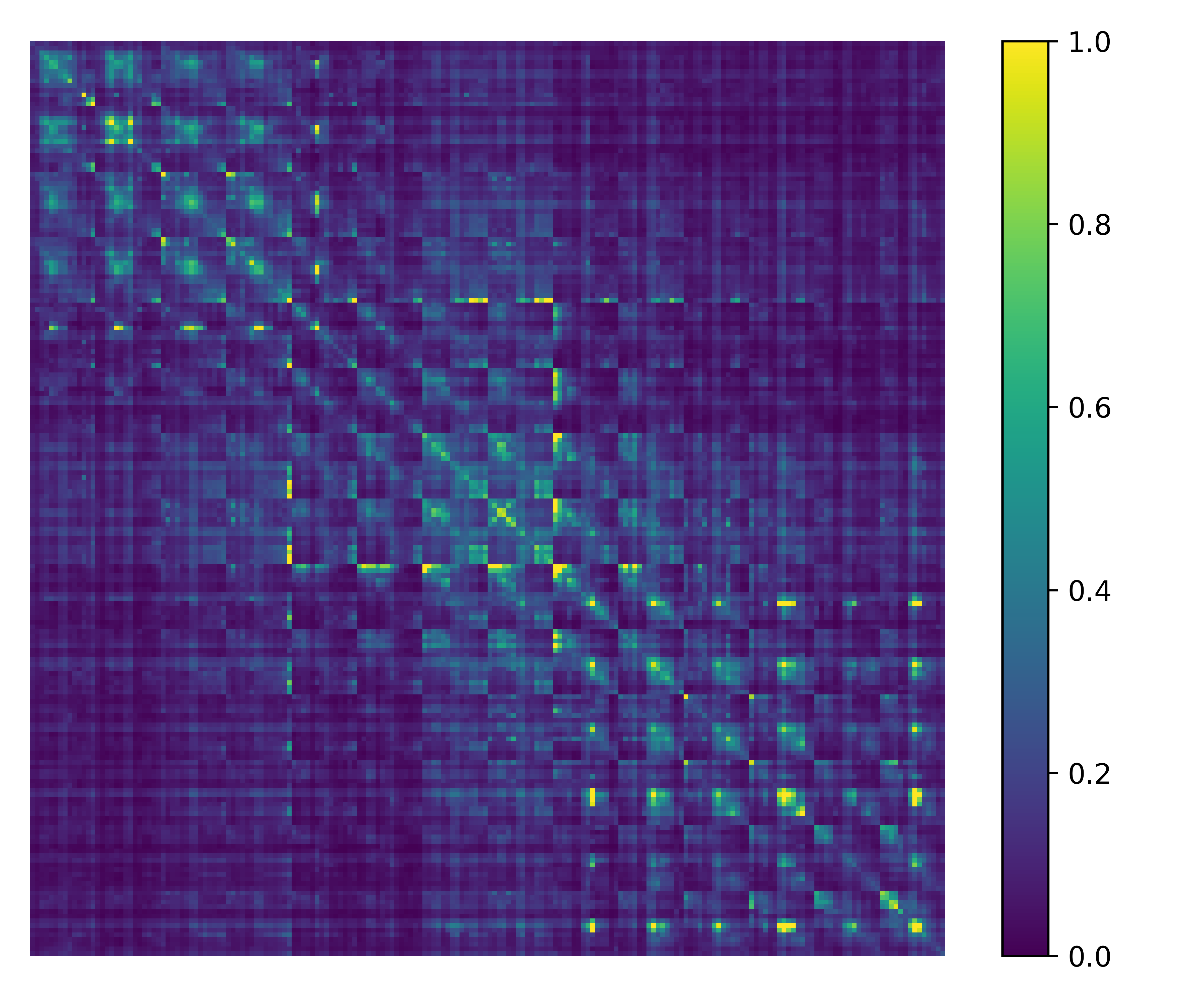}
            };
            \node[anchor=south west, text=white, font=\sffamily\tiny, inner sep=1pt] 
                  at ([xshift=3pt, yshift=5.0pt] containerimg.south west) 
                  {Swin-B};
        \end{tikzpicture}
        }%
    }
    \hspace{0.02\textwidth}
    \subcaptionbox{Cross-token coupling diagnostic for ConvNeXt. \label{fig:motivation_convnext_corr}}{%
        \parbox[b]{0.30\textwidth}{%
            \centering
        \begin{tikzpicture}
            \node[anchor=south west, inner sep=0] (containerimg) at (0,0) {%
                \includegraphics[width=\linewidth]{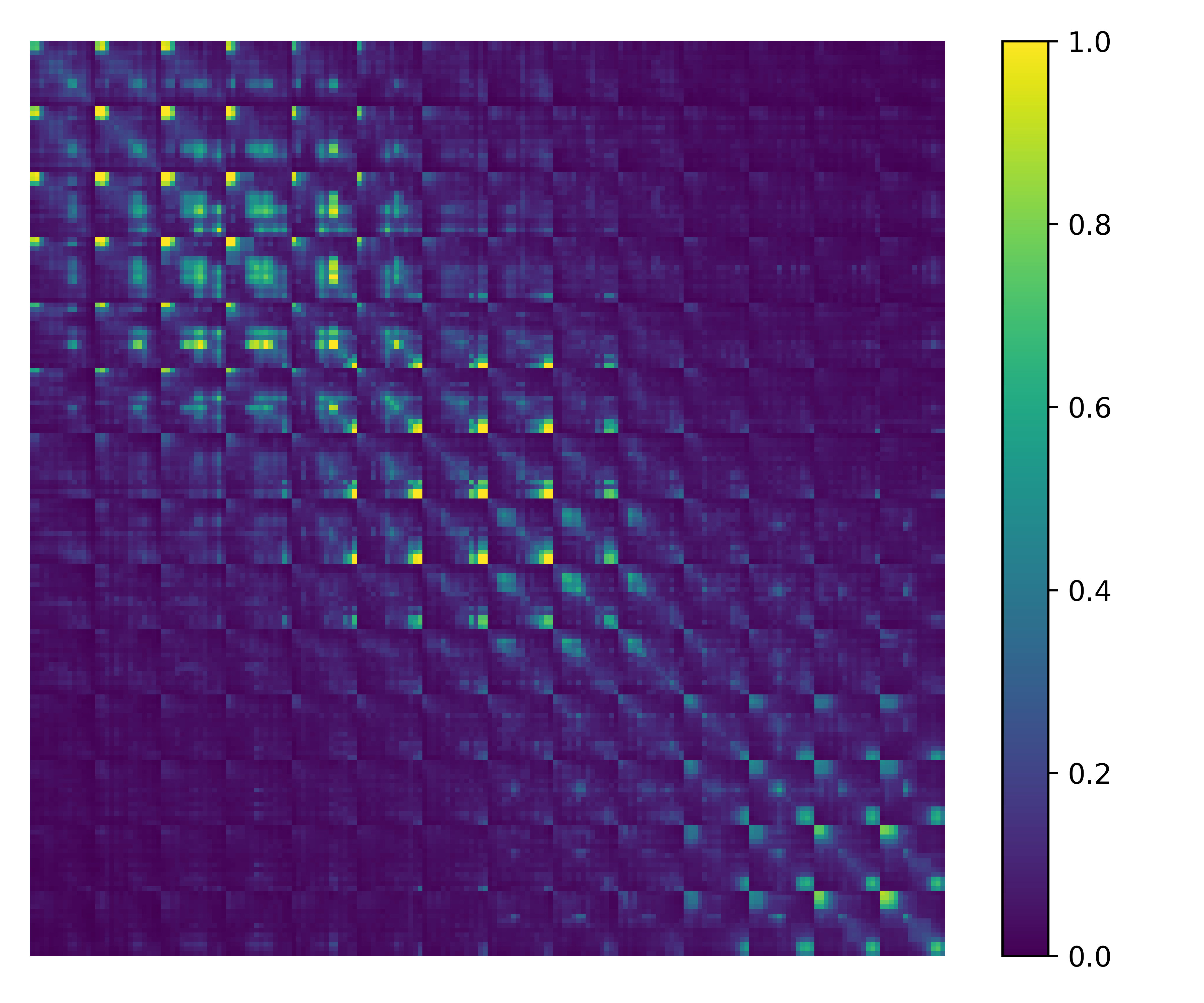}
            };
            \node[anchor=south west, text=white, font=\sffamily\tiny, inner sep=1pt] 
                  at ([xshift=3pt, yshift=5.0pt] containerimg.south west) 
                  {ConvNeXt};
        \end{tikzpicture}
        }%
    }
    \hspace{0.02\textwidth}
    \subcaptionbox{Cross-token coupling diagnostic for Mamba. \label{fig:motivation_mamba_corr}}{%
        \parbox[b]{0.30\textwidth}{%
            \centering
        \begin{tikzpicture}
            \node[anchor=south west, inner sep=0] (containerimg) at (0,0) {%
                \includegraphics[width=\linewidth]{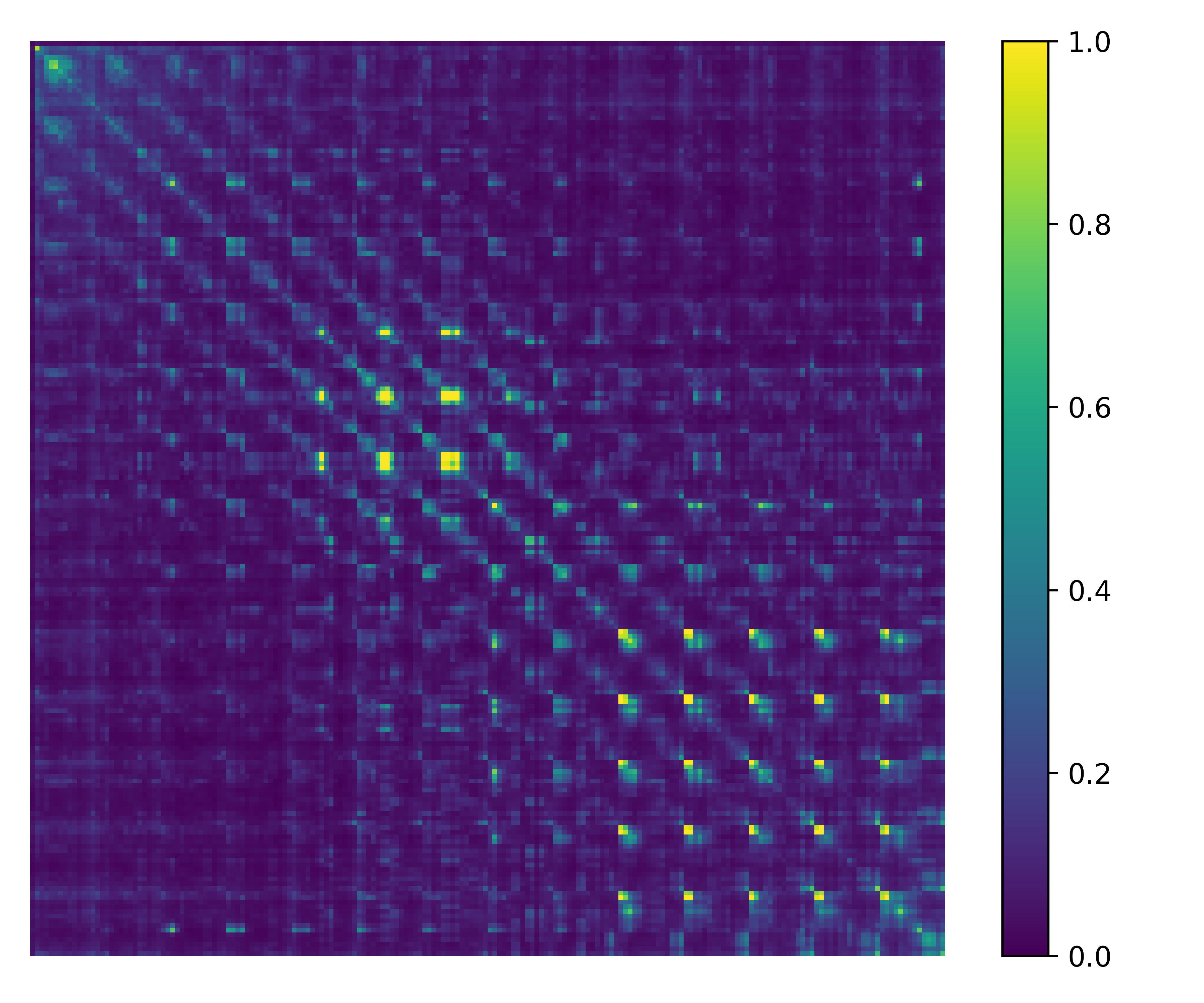}
            };
            \node[anchor=south west, text=white, font=\sffamily\tiny, inner sep=1pt] 
                  at ([xshift=3pt, yshift=5.0pt] containerimg.south west) 
                  {MambaVision};
        \end{tikzpicture}
        }%
    }

    \vspace{0.02\textwidth}

    \subcaptionbox{Activation–gradient distance correlation for Swin. \label{fig:motivation_swin_actgrad}}{%
        \parbox[b]{0.30\textwidth}{%
            \centering
            \input{figures/supplementary/FIM_subplots/swin_correlation}
        }%
    }
    \hspace{0.02\textwidth}
    \subcaptionbox{Activation–gradient distance correlation for ConvNeXt. \label{fig:motivation_convnext_actgrad}}{%
        \parbox[b]{0.30\textwidth}{%
            \centering
            \input{figures/supplementary/FIM_subplots/convnext_correlation}
        }%
    }
    \hspace{0.02\textwidth}
    \subcaptionbox{Activation–gradient distance correlation for Mamba. \label{fig:motivation_mamba_actgrad}}{%
        \parbox[b]{0.30\textwidth}{%
            \centering
            \input{figures/supplementary/FIM_subplots/mamba_correlation}
        }%
    }

    \caption{
        Comparison of token correlations (top) and activation–gradient correlations (bottom) across select linear layers of Swin, ConvNeXt, and Mamba models. The top displays cross-token coupling diagnostic, while the bottom shows input activation and output gradient distance correlations.
    }
    \label{supp:fig:extended_motivation_grid}
\end{figure*}

%% file: figures/supplementary/FIM_subplots/swin_correlation.tex
\pgfplotstableread[col sep=comma]{figures/additional_images/supplement/Swin/layers_2_blocks_16_mlp_fc2_act_grad_coup.csv}\dcorrdata

\begin{tikzpicture}
\begin{axis}[
    width=\linewidth,
    height=4.2cm,
    xlabel={Token index},
    ylabel={Distance correlation ($dCor$)},
    grid=both,
    grid style={gray!15},
    tick style={black},
    ymin=0.0,
    ymax=0.55,
    xmin=-5,
    xmax=205,
    line width=1pt,
    tick label style={font=\scriptsize},
    label style={font=\scriptsize},
    title style={font=\small},
    ylabel style={yshift=-0.5cm},
    xlabel style={yshift=+0.5cm},
    xticklabels={0, 0 , , , ,200},
]

\addplot[
    only marks,
    mark=x,
    mark size=.5pt,
    color=blue,
    fill=blue,
    forget plot,
] table [x=token, y=dcorr, filter discard warning=false, row predicate/.code={
    \pgfplotstablegetelem{#1}{significant}\of{\dcorrdata}
    \ifx\pgfplotsretval\empty
        \pgfplotstablerowfalse
    \else
        \edef\temp{\pgfplotsretval}
        \ifx\temp true\relax
            \pgfplotstablerowtrue
        \else
            \pgfplotstablerowfalse
        \fi
    \fi
}] {\dcorrdata};

\addplot[dashed, black, thick] coordinates {(-10,0.32) (220,0.32)};
\node[anchor=west, black] at (axis cs:80,0.4) {\scriptsize \textbf{Mean}};

\end{axis}
\end{tikzpicture}

%% file: figures/supplementary/FIM_subplots/convnext_correlation.tex
\pgfplotstableread[col sep=comma]{figures/additional_images/supplement/convnext/stages_2_blocks_22_mlp_fc1_act_grad_coup.csv}\dcorrdata

\begin{tikzpicture}
\begin{axis}[
    width=\linewidth,
    height=4.2cm,
    xlabel={Token index},
    ylabel={Distance correlation ($dCor$)},
    grid=both,
    grid style={gray!15},
    tick style={black},
    ymin=0.0,
    ymax=0.55,
    xmin=-5,
    xmax=205,
    line width=1pt,
    tick label style={font=\scriptsize},
    label style={font=\scriptsize},
    title style={font=\small},
    ylabel style={yshift=-0.5cm},
    xlabel style={yshift=+0.5cm},
    xticklabels={0, 0 , , , ,200},
]
\addplot[
    only marks,
    mark=x,
    mark size=.5pt,
    color=blue,
    fill=blue,
    forget plot,
] table [x=token, y=dcorr, filter discard warning=false, row predicate/.code={
    \pgfplotstablegetelem{#1}{significant}\of{\dcorrdata}
    \ifx\pgfplotsretval\empty
        \pgfplotstablerowfalse
    \else
        \edef\temp{\pgfplotsretval}
        \ifx\temp true\relax
            \pgfplotstablerowtrue
        \else
            \pgfplotstablerowfalse
        \fi
    \fi
}] {\dcorrdata};

\addplot[dashed, black, thick] coordinates {(-10,0.1) (220,0.1)};
\node[anchor=west, black] at (axis cs:100,0.17) {\scriptsize  \textbf{mean}};

\end{axis}
\end{tikzpicture}

%% file: figures/supplementary/FIM_subplots/mamba_correlation.tex
\pgfplotstableread[col sep=comma]{figures/additional_images/supplement/Mamba/levels_2_blocks_5_mixer_proj_drop_act_grad_coup.csv}\dcorrdata

\begin{tikzpicture}
\begin{axis}[
    width=\linewidth,
    height=4.2cm,
    xlabel={Token index},
    ylabel={Distance correlation ($dCor$)},
    grid=both,
    grid style={gray!15},
    tick style={black},
    ymin=0.0,
    ymax=0.55,
    xmin=-5,
    xmax=205,
    line width=1pt,
    tick label style={font=\scriptsize},
    label style={font=\scriptsize},
    title style={font=\small},
    ylabel style={yshift=-0.5cm},
    xlabel style={yshift=+0.5cm},
    xticklabels={0, 0 , , , ,200},
]
\addplot[
    only marks,
    mark=x,
    mark size=.5pt,
    color=blue,
    fill=blue,
    forget plot,
] table [x=token, y=dcorr, filter discard warning=false, row predicate/.code={
    \pgfplotstablegetelem{#1}{significant}\of{\dcorrdata}
    \ifx\pgfplotsretval\empty
        \pgfplotstablerowfalse
    \else
        \edef\temp{\pgfplotsretval}
        \ifx\temp true\relax
            \pgfplotstablerowtrue
        \else
            \pgfplotstablerowfalse
        \fi
    \fi
}] {\dcorrdata};

\addplot[dashed, black, thick] coordinates {(-10,0.275) (220,0.275)};
\node[anchor=west, black] at (axis cs:100,0.25) {\scriptsize \textbf{Mean}};

\end{axis}
\end{tikzpicture}

%% file: sec/supplementary/extended_emprical_insights.tex
\input{figures/supplementary/extended_fim_deit_layers}

This section extends the limited visualization of cross-token interactions (within the diagnostic metric explained in \cref{supp:subsec:fisher_comparison_motivation}) and  activation-gradient coupling shown in the main content (\cref{fig:motivation}).

\subsubsection{Generalization Across Architectures.}
We extend our analysis to three additional distinct architectures: Swin-B, ConvNeXt-B, and MambaVision-B. As illustrated in \cref{supp:fig:extended_motivation_grid}, all examined models exhibit cross-token pattern similar to DeiT. Furthermore, they all demonstrate couplings between input activations and output gradients. This finding empirically substantiates the limitation of standard approximations like KFAC~\cite{kfac-martens15}, which assume independence between activations and gradient terms.

\subsubsection{Generalization Across Layer Types.}
We further investigate whether this structural correlation is specific to MLP layers. \cref{supp:fig:extended_motivation_grid_deit_layertypes} presents the cross-token diagnostic for all the other layer types within the DeiT architecture, namely Query-Key-Value (QKV), Projection, and a FC1 layer. While the specific spectral signatures vary, the presence of structurally induced importance in the off-diagonals remains significant across all layer types.

\subsubsection{Summary.} From the empirical insights shown here it is clear, that both the phenomenon of cross-token interactions and couplings between input activations and output gradients are commonly found in vision models - not only \gls{vit}. While the strength of the phenomena varies between layers within models and also between architectures, both are prevalent in the examined architectures, highlighting the importance of investigating the impact of structural priors.

%% file: figures/supplementary/extended_fim_deit_layers.tex
\begin{figure*}[t]
    \centering
    \subcaptionbox{Token correlations in the FIM for QKV. \label{supp:fig:motivation_deit_qkv_corr}}{%
        \parbox[b]{0.30\textwidth}{%
            \centering
        \begin{tikzpicture}
            \node[anchor=south west, inner sep=0] (containerimg) at (0,0) {%
                \includegraphics[width=\linewidth]{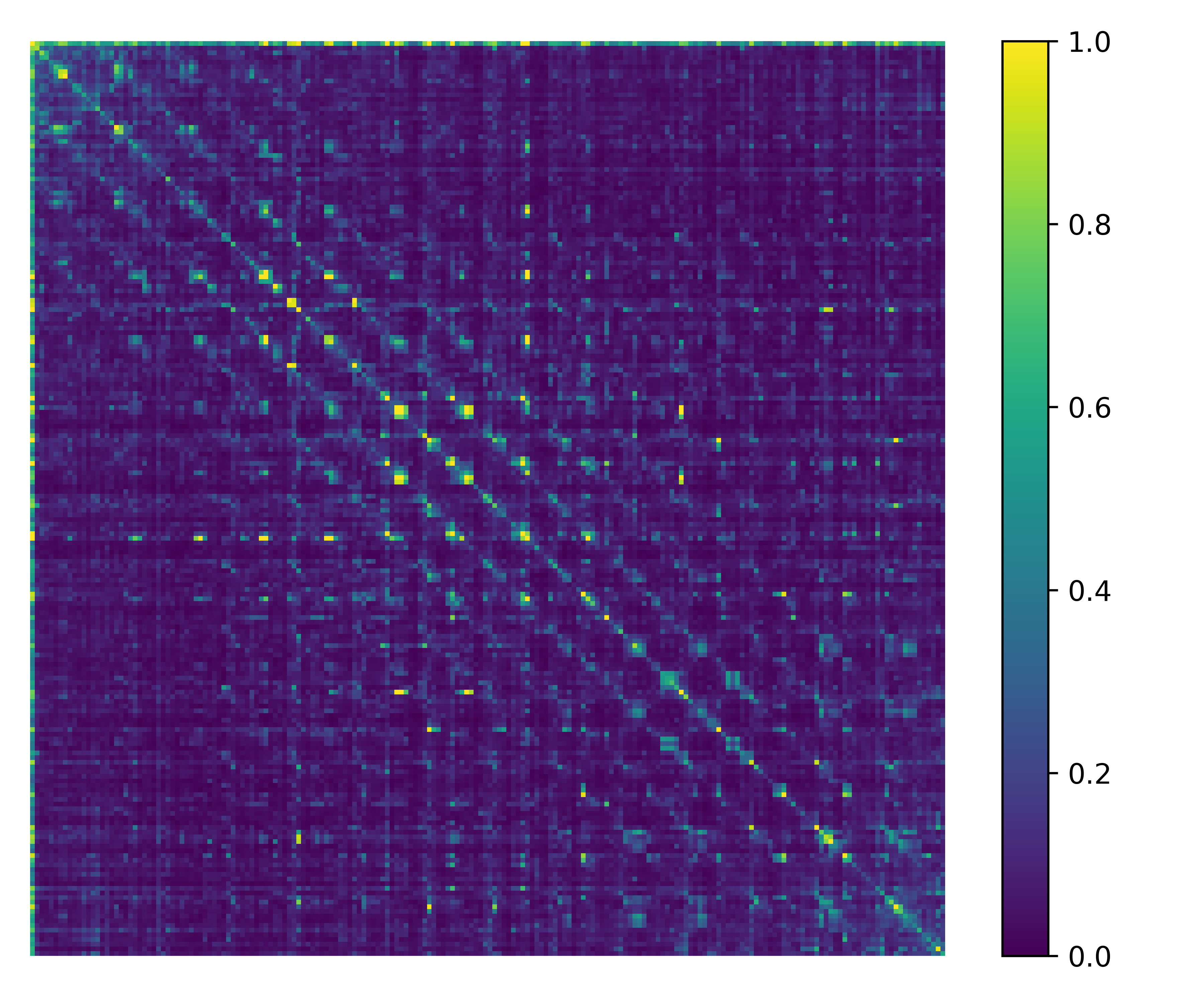}
            };
            \node[anchor=south west, text=white, font=\sffamily\tiny, inner sep=1pt] 
                  at ([xshift=3pt, yshift=5.0pt] containerimg.south west) 
                  {DeiT-B qkv};
        \end{tikzpicture}
        }%
    }
    \hspace{0.02\textwidth}
    \subcaptionbox{Token correlations in the FIM for Proj. \label{supp:fig:motivation_deit_proj_corr}}{%
        \parbox[b]{0.30\textwidth}{%
            \centering
        \begin{tikzpicture}
            \node[anchor=south west, inner sep=0] (containerimg) at (0,0) {%
                \includegraphics[width=\linewidth]{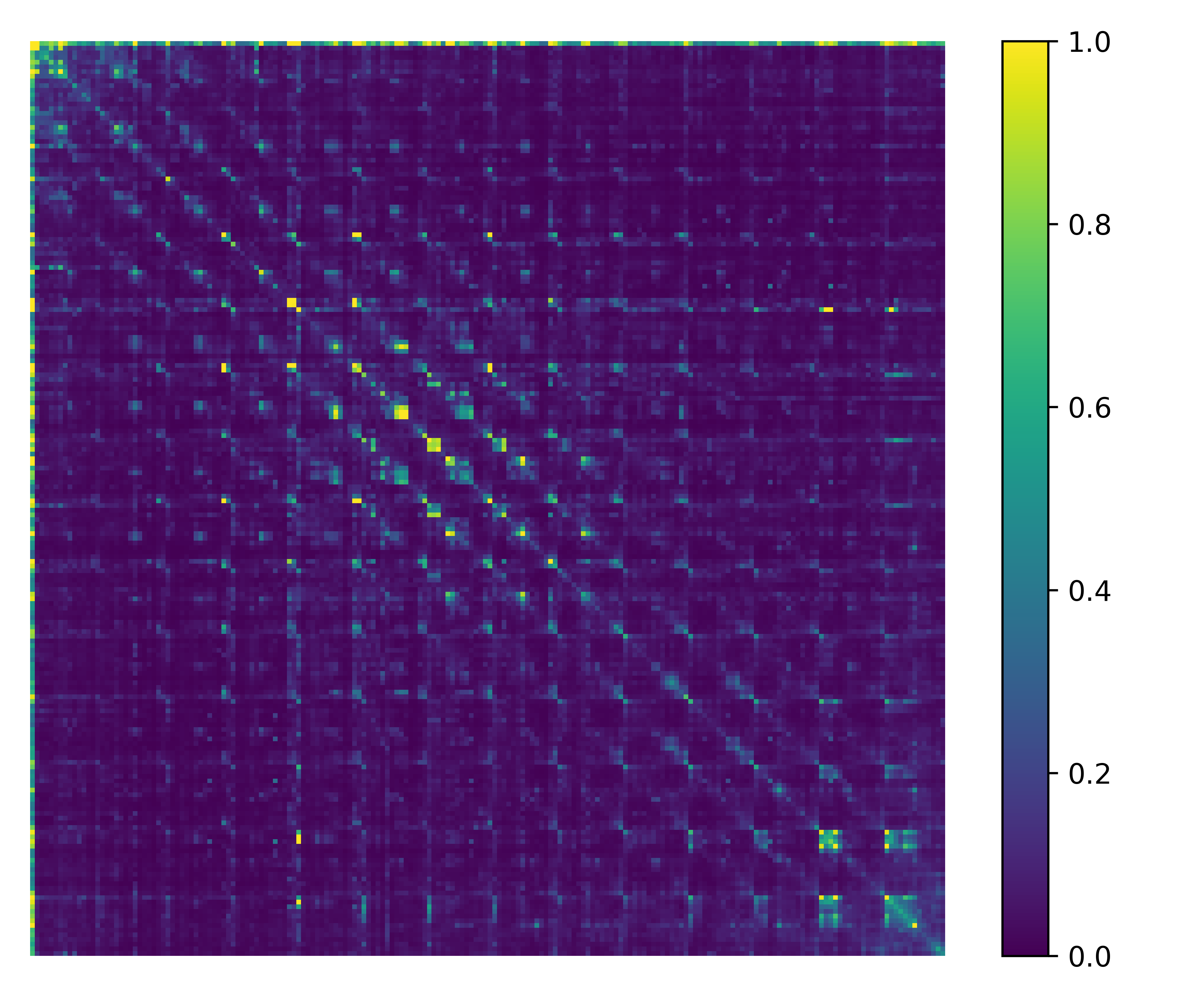}
            };
            \node[anchor=south west, text=white, font=\sffamily\tiny, inner sep=1pt] 
                  at ([xshift=3pt, yshift=5.0pt] containerimg.south west) 
                  {DeiT-B proj};
        \end{tikzpicture}
        }%
    }
    \hspace{0.02\textwidth}
    \subcaptionbox{Token correlations in the FIM for FC1. \label{supp
    :fig:motivation_deit_proj_corr}}{%
        \parbox[b]{0.30\textwidth}{%
            \centering
        \begin{tikzpicture}
            \node[anchor=south west, inner sep=0] (containerimg) at (0,0) {%
                \includegraphics[width=\linewidth]{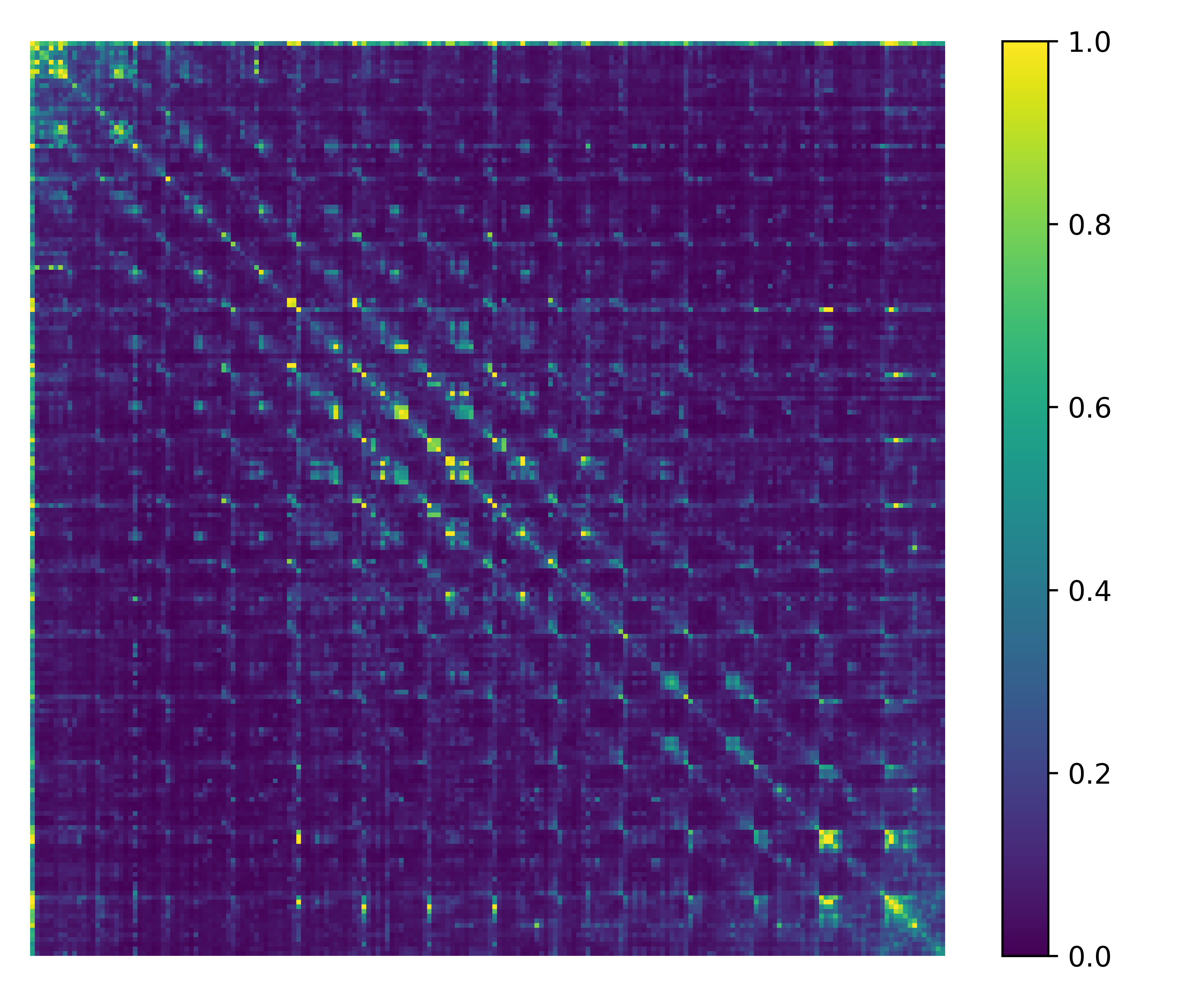}
            };
            \node[anchor=south west, text=white, font=\sffamily\tiny, inner sep=1pt] 
                  at ([xshift=3pt, yshift=5.0pt] containerimg.south west) 
                  {DeiT-B fc1};
        \end{tikzpicture}
        }%
    }

    \vspace{0.02\textwidth}
    \subcaptionbox{Activation–gradient correlations for QKV. \label{supp:fig:motivation_deit_proj_actgrad}}{%
        \parbox[b]{0.30\textwidth}{%
            \centering
            \input{figures/supplementary/FIM_subplots/deit_correlation_10qkv}
        }%
    }
    \hspace{0.02\textwidth}
    \subcaptionbox{Activation–gradient correlations for Proj. \label{fig:motivation_deit_qkv_actgrad}}{%
        \parbox[b]{0.30\textwidth}{%
            \centering
            \input{figures/supplementary/FIM_subplots/deit_correlation_10proj}
        }%
    }
    \hspace{0.02\textwidth}
    \subcaptionbox{Activation–gradient correlations for FC1. \label{fig:motivation_deit_fc1_actgrad}}{%
        \parbox[b]{0.30\textwidth}{%
            \centering
            \input{figures/supplementary/FIM_subplots/deit_correlation_10fc1}
        }%
    }

    \caption{
        Comparison of token correlations (top row) and activation–gradient correlations (bottom row)
        across QKV, Proj and FC1 layers. The top row shows cross-token Fisher information
        structures, while the bottom row the distance correlation between input activations
        and output gradients.
    }
    \label{supp:fig:extended_motivation_grid_deit_layertypes}
\end{figure*}

%% file: figures/supplementary/FIM_subplots/deit_correlation_10qkv.tex
\pgfplotstableread[col sep=comma]{figures/additional_images/supplement/deit/blocks_10_attn_qkv_act_grad_coup.csv}\dcorrdata

\begin{tikzpicture}
\begin{axis}[
    width=\linewidth,
    height=4.2cm,
    xlabel={Token index},
    ylabel={Distance correlation ($dCor$)},
    grid=both,
    grid style={gray!15},
    tick style={black},
    ymin=0.0,
    ymax=0.9,
    xmin=-5,
    xmax=205,
    line width=1pt,
    tick label style={font=\scriptsize},
    label style={font=\scriptsize},
    title style={font=\small},
    ylabel style={yshift=-0.5cm},
    xlabel style={yshift=+0.5cm},
    xticklabels={0, 0 , , , ,200},
]

\addplot[
    only marks,
    mark=x,
    mark size=.5pt,
    color=blue,
    fill=blue,
    forget plot,
] table [x=token, y=dcorr, filter discard warning=false, row predicate/.code={
    \pgfplotstablegetelem{#1}{significant}\of{\dcorrdata}
    \ifx\pgfplotsretval\empty
        \pgfplotstablerowfalse
    \else
        \edef\temp{\pgfplotsretval}
        \ifx\temp true\relax
            \pgfplotstablerowtrue
        \else
            \pgfplotstablerowfalse
        \fi
    \fi
}] {\dcorrdata};
\addplot[dashed, black, thick] coordinates {(-10,0.31) (220,0.31)};
\node[anchor=west, black] at (axis cs:100,0.4) {\scriptsize \textbf{Mean}};

\end{axis}
\end{tikzpicture}

%% file: figures/supplementary/FIM_subplots/deit_correlation_10proj.tex
\pgfplotstableread[col sep=comma]{figures/additional_images/supplement/deit/blocks_10_attn_proj_act_grad_coup.csv}\dcorrdata

\begin{tikzpicture}
\begin{axis}[
    width=\linewidth,
    height=4.2cm,
    xlabel={Token index},
    ylabel={Distance correlation ($dCor$)},
    grid=both,
    grid style={gray!15},
    tick style={black},
    ymin=0.0,
    ymax=0.9,
    xmin=-5,
    xmax=205,
    line width=1pt,
    tick label style={font=\scriptsize},
    label style={font=\scriptsize},
    title style={font=\small},
    ylabel style={yshift=-0.5cm},
    xlabel style={yshift=+0.5cm},
    xticklabels={0, 0 , , , ,200},
]
\addplot[
    only marks,
    mark=x,
    mark size=.5pt,
    color=blue,
    fill=blue,
    forget plot,
] table [x=token, y=dcorr, filter discard warning=false, row predicate/.code={
    \pgfplotstablegetelem{#1}{significant}\of{\dcorrdata}
    \ifx\pgfplotsretval\empty
        \pgfplotstablerowfalse
    \else
        \edef\temp{\pgfplotsretval}
        \ifx\temp true\relax
            \pgfplotstablerowtrue
        \else
            \pgfplotstablerowfalse
        \fi
    \fi
}] {\dcorrdata};
\addplot[dashed, black, thick] coordinates {(-10,0.28) (220,0.28)};
\node[anchor=west, black] at (axis cs:100,0.41) {\scriptsize \textbf{Mean}};

\end{axis}
\end{tikzpicture}

%% file: figures/supplementary/FIM_subplots/deit_correlation_10fc1.tex
\pgfplotstableread[col sep=comma]{figures/additional_images/supplement/deit/blocks_10_mlp_fc1_act_grad_coup.csv}\dcorrdata

\begin{tikzpicture}
\begin{axis}[
    width=\linewidth,
    height=4.2cm,
    xlabel={Token index},
    ylabel={Distance correlation ($dCor$)},
    grid=both,
    grid style={gray!15},
    tick style={black},
    ymin=0.0,
    ymax=0.9,
    xmin=-5,
    xmax=205,
    line width=1pt,
    tick label style={font=\scriptsize},
    label style={font=\scriptsize},
    title style={font=\small},
    ylabel style={yshift=-0.5cm},
    xlabel style={yshift=+0.5cm},
    xticklabels={0, 0 , , , ,200},
]
\addplot[
    only marks,
    mark=x,
    mark size=.5pt,
    color=blue,
    fill=blue,
    forget plot,
] table [x=token, y=dcorr, filter discard warning=false, row predicate/.code={
    \pgfplotstablegetelem{#1}{significant}\of{\dcorrdata}
    \ifx\pgfplotsretval\empty
        \pgfplotstablerowfalse
    \else
        \edef\temp{\pgfplotsretval}
        \ifx\temp true\relax
            \pgfplotstablerowtrue
        \else
            \pgfplotstablerowfalse
        \fi
    \fi
}] {\dcorrdata};
\addplot[dashed, black, thick] coordinates {(-10,0.41) (220,0.41)};
\node[anchor=west, black] at (axis cs:100,0.35) {\scriptsize \textbf{Mean}};

\end{axis}
\end{tikzpicture}

%% file: sec/supplementary/extended_throughput_eval.tex
While our main results demonstrate the efficacy of SVD for accelerating DeiT and Swin, we provide here a deeper analysis comparing our approach against alternative compression approaches.

\subsubsection{Benchmarking Semi-Structured Sparsity.}\label{supp:sec:extended_throughput_eval}
We start by benchmarking 2:4 semi-structured sparsity.
As shown in \cref{supp:fig:throughput_analysis_semi_and_memvit} (left), contrary to expectations, the throughput of the 2:4 sparse model is significantly lower than the dense baseline. On the A100, it achieves only 0.8$\times$ of the baseline throughput, and this degrades further to 0.5$\times$ on the H100.
While 2:4 sparsity is natively supported on Ampere (A100) and Hopper (H100) architectures~\cite{bai-2023}, its implementation appears to be highly optimized for specific operator types or workloads (e.g., large linear layers in NLP Transformers) and does not seem to be effective at layer sizes common in Vision Transformer architectures. The V100 GPU does not offer hardware support for 2:4 sparsity. CPU support is possible via third-party libraries such as DeepSparse~\cite{Neural_Magic_DeepSparse_2021}, but it is not a standard, natively accelerated feature, which is why we have not benchmarked it here.

\subsubsection{Comparing Standard SVD to MemViT.}
MemViT~\cite{azizi2024memoryViT} proposes enhancing decomposed layers with a low-rank residual module to reduce reconstruction error. We investigate the runtime implications of this addition in \cref{supp:fig:throughput_analysis_semi_and_memvit} (right).
We compare a DeiT-B model with standard SVD at 50\% compression to an equivalent model incorporating the MemViT low-rank extension (which adds 5\% of the original layer's parameters~\cite{azizi2024memoryViT}). The results show that the MemViT-enhanced model still improves throughput relative to the baseline. However, despite adding only a few parameters, the speedup is significantly lower. In particular, the MemViT variant achieves only $\approx$1.3$\times$ acceleration over the baseline, compared to $\approx$1.6$\times$ for standard SVD at a similar compression level. This demonstrates that standard SVD offers a superior acceleration-performance trade-off, motivating our choice of standard \gls{svd} over MemoryViT's \gls{svd} variant for this work.

\input{figures/throughput_analysis_supplementary}

%% file: figures/throughput_analysis_supplementary.tex
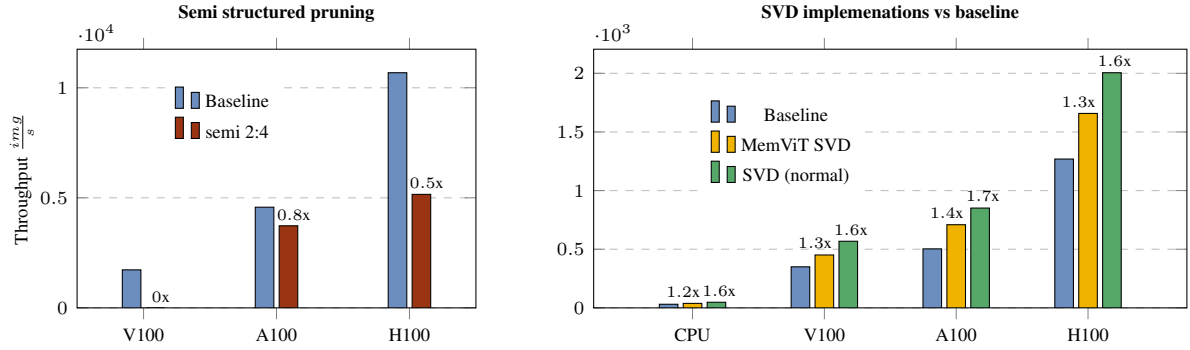
\begin{figure}[t]
    \centering
    \begin{subfigure}[t]{0.4\linewidth}
        \centering
        \begin{tikzpicture}
        \begin{axis}[
            ybar,
            bar width=7pt,
            width=\linewidth,
            height=5.0cm,
            ymin=0,
            ylabel={Throughput $\frac{img}{s}$},
            symbolic x coords={V100, A100, H100},
            xtick=data,
            tick label style={font=\scriptsize},
            legend style={
                draw=none,
                font=\scriptsize,
                at={(0.37,0.6)},
                anchor=south,
                legend columns=1},
            ymajorgrids=true,
            grid style=dashed,
            enlarge x limits=0.25,
            ylabel style={font=\scriptsize, yshift=-0.5cm},
            title={\scriptsize \textbf{Semi structured pruning}},
            title style={yshift=+0.2ex},
        ]

        \addplot+[ybar, fill=DarkBlue, draw=black, nodes near coords={}] %
            coordinates {(V100,1727) (A100,4574) (H100,10684) 
            }; 

        \addplot+[ybar, fill=DarkOrange, draw=black,
            point meta rel=per plot,
            visualization depends on={rawy/\thisrow{base} \as \ratio},
            nodes near coords={\text{\pgfmathprintnumber[fixed,precision=1]{\ratio}x}},
            every node near coord/.append style={font=\tiny, text=black, yshift=-1.2pt, xshift=2pt},
        ] table[row sep=\\, meta=base] {
            x      y     base  \\
            V100   0     1727   \\
            A100   3730  4574   \\
            H100   5157  10684  \\
        };

        \legend{Baseline, semi 2:4}
        \end{axis}
        \end{tikzpicture}
    \end{subfigure}
    \hspace{0.0\textwidth}
    \begin{subfigure}[t]{0.55\linewidth}
        \centering
        \begin{tikzpicture}
            \begin{axis}[
                ybar,
                bar width=7pt,
                width=\linewidth,
                height=5.0cm,
                ymin=0,
                symbolic x coords={CPU, V100, A100, H100},
                xtick=data,
                tick label style={font=\scriptsize},
                legend style={
                    draw=none,
                    font=\scriptsize,
                    at={(0.32,0.43)},
                    anchor=south,
                    legend columns=1},
                ymajorgrids=true,
                grid style=dashed,
                enlarge x limits=0.25,
                ylabel style={font=\small, yshift=-0.1cm},
                title={\scriptsize \textbf{SVD implemenations vs baseline}},
                title style={yshift=+0.2ex},
                scaled y ticks=base 10:-3,
                y tick label style={/pgf/number format/fixed},
            ]
            
            \addplot+[ybar, fill=DarkBlue, draw=black, nodes near coords={}]
                coordinates {(CPU,30.74) (V100,350) (A100,503) (H100,1269)};
            
            \addplot+[ybar, fill=DarkYellow, draw=black,
                point meta rel=per plot,
                visualization depends on={rawy/\thisrow{base} \as \ratio},
                nodes near coords={\text{\pgfmathprintnumber[fixed,precision=1]{\ratio}x}},
                every node near coord/.append style={
                    font=\tiny,
                    text=black,
                    yshift=-1pt,
                    anchor=south,
                    xshift=-3.5pt,
                },
            ] table[row sep=\\, meta=base] {
                x      y     base  \\
                CPU    38.23  30.74  \\
                V100   451    350   \\
                A100   709    503   \\
                H100   1658   1269  \\
            };
            
            \addplot+[ybar, fill=DarkGreen, draw=black,
                point meta rel=per plot,
                visualization depends on={rawy/\thisrow{base} \as \ratio},
                nodes near coords={\text{\pgfmathprintnumber[fixed,precision=1]{\ratio}x}},
                every node near coord/.append style={
                    font=\tiny,
                    text=black,
                    yshift=-1pt,
                    xshift=1pt,
                    anchor=south,
                },
            ] table[row sep=\\, meta=base] {
                x      y     base  \\
                CPU    47.72  30.74  \\
                V100   568    350   \\
                A100   851    503   \\
                H100   2005   1269  \\
            };
            
            \legend{Baseline, MemViT SVD, SVD (normal)}
            
            \end{axis}
        \end{tikzpicture}

    \end{subfigure}

    \caption{Extended throughput comparison on DeiT-B. The Semi structured pruning (left) cannot get any latency improvement over the baseline. The MemViT SVD approach is 20-30\% slower than the standard SVD approach as used by our work.}
    \label{supp:fig:throughput_analysis_semi_and_memvit}
\end{figure}

%% file: sec/supplementary/extended_downstream_results.tex
We provide expanded downstream evaluation results for semantic segmentation on the ADE20k dataset~\cite{ade20k}, utilizing UPerNet~\cite{xiao2018unified} models with both DeiT-B~\cite{deit} and Swin-B~\cite{swin} backbones. All experimental protocols regarding baselines and training schedules follow the setup detailed in \cref{supp:sec:extended_experiment_setup}.

The results are summarized in \cref{supp:tab:down_seg}.
For the \gls{deit}-B backbone, our method demonstrates exceptional zero-shot efficiency. Our model achieves 43.5 mIoU without any retraining, significantly outperforming the SVD-LLM baseline (38.1 mIoU). Notably, this zero-shot result also surpasses the fully retrained PELA method (42.4 mIoU), achieving superior accuracy while completely bypassing the computationally expensive multi-stage training pipeline required by the latter. With the addition of a short finetuning phase, our performance further improves to \textbf{44.0 mIoU}, recovering nearly the full performance of the uncompressed baseline (44.9 mIoU).

For the \textbf{\gls{swin}-B} backbone, we observe a similar dominance over the SVD-LLM baseline (46.2 vs. 37.7 mIoU). While the multi-stage PELA method initially leads with 48.2 mIoU, our approach effectively closes this gap with minimal compute. A brief finetuning phase (10k iterations, $<10\%$ of standard training) boosts our method (\textbf{\gls{ours}+ft}) to 48.1 mIoU, effectively matching PELA. Furthermore, when permitting a slightly more elaborate (yet still efficient) finetuning schedule, our method reaches 48.6 mIoU, approaching the uncompressed baseline of 49.4 mIoU.

In summary, our approach provides a high-performance, compute-efficient alternative to previous SVD-based methods. It delivers state-of-the-art compression trade-offs, outperforming expensive retraining pipelines in zero-shot settings or matching them with lower training cost.

\begin{table}[htbp]
    \centering
    \caption{Semantic segmentation mIoU on ADE20K dataset at 512x512 resolution.}
    \resizebox{0.65\linewidth}{!}{
    \begin{tabular}{c l c c c}
    \toprule
    \midrule
    Model & Method                             & Params(M) & GFLOPs & mIoU\\
    \midrule
    \multirow{5}{*}{DeiT-B}        
        & Base~\cite{deit}                     & 144.3 & 848.3 & 44.9 \\
        \cmidrule{2-5}
        & PELA~\cite{pela}                     & 101.8 & 761.1 & 42.4\\
        & SVD-LLM~\cite{wang2024svd-llm}       & 101.8 & 761.1 & 38.1 \\
        \rowcolor{m} &\textbf{\gls{ours}}      & 101.8 & 761.1 & \underline{43.5} \\
        \rowcolor{m} &\textbf{\gls{ours}+ft}   & 101.8 & 761.1 & \textbf{44.0} \\
    \midrule
    \multirow{5}{*}{Swin-B}
        & Base~\cite{swin}                     & 122.3 & 594.6 & 49.4\\
        \cmidrule{2-5}
        & PELA~\cite{pela}                     & 88.7  & 527.7 & \textbf{48.2}\\
        & SVD-LLM~\cite{wang2024svd-llm}       & 83.8  & 527.7 & 37.7\\
        \rowcolor{m} &\textbf{\gls{ours}}      & 86.8  & 527.7 & 46.2\\
        \rowcolor{m} &\textbf{\gls{ours}+ft}   & 86.8  & 527.7 & \underline{48.1}\\
    \bottomrule
    \bottomrule
    \end{tabular}}
    \label{supp:tab:down_seg}
\end{table}

%% file: sec/supplementary/extended_main_tables.tex
We extend the evaluation across models and \gls{svd} methods from the main content to additional compression rates. In particular, while the main content evaluated 50\% remaining linear layer FLOPs in \cref{tab:classification_50}, we extend it to 40\% and 60\% with results in \cref{tab:classification_40} and \cref{tab:classification_60}, respectively. Results resonate the same improvements seen in the main content, improving over prior works when using \gls{ours} with uniform compression, with larger gains when applying search on top of it.

\begin{table}[t]
    \centering
    \caption{Top-1 accuracy for various models on \gls{imagenet} at \textbf{uniform} compression (40\% linear layers FLOPs remaining) for related \gls{svd} methods. In addition to uniform compression, we report \gls{ours} with our \textbf{\gls{oursearch}} to isolate the effect of decomposition and search.}
    \resizebox{\textwidth}{!}{%
    \begin{tabular}{l  c c   c c   c c   c c }
    \toprule
    \toprule
    \multirow{2}{*}{Method} & \multicolumn{2}{c}{\textbf{DeiT-B}} & \multicolumn{2}{c}{\textbf{Swin-B}} & \multicolumn{2}{c}{\textbf{ConvNeXt-B}} & \multicolumn{2}{c}{\textbf{MambaVis.-B}} \\
    & GFLOPs & Top-1$\uparrow$ & GFLOPs & Top-1$\uparrow$ & GFLOPs & Top-1$\uparrow$ & GFLOPs & Top-1$\uparrow$ \\
    \midrule
    Baseline & 33.9 & 83.3 & 30.3 & 85.1 & 30.8 & 85.8 & 29.9 & 83.9 \\
    \cmidrule{2-9}
    PELA~\cite{pela} & 13.7 & 46.4 & 12.5 & 0.6 & 12.9 & 2.0 & 19.8 & 55.7 \\
    FW-SVD~\cite{hsu2022fwsvd} & 13.7 & 61.8 & 12.5 & 3.4 & 12.9 & 13.4 & 19.8 & 51.7 \\
    ASVD~\cite{yuan2024asvd} & 13.7 & 60.3 & 12.5 & 2.1 & 12.9 & 14.7 & 19.8 & 53.6 \\
    SVD-LLM~\cite{wang2024svd-llm} & 13.7 & 66.8 & 12.5 & 33.3 & 12.9 & 53.8 & 19.8 & 59.1 \\
    GFWSVD~\cite{chekalina2025generalizedfisherweightedsvdscalable} & 13.7 & 57.6 & 12.5 & 0.6 & 12.9 & 4.7 & 19.8 & 58.5 \\
    FLAR-SVD~\cite{2025_flar-svd} & 13.7 & 66.9 & 12.5 & 33.3 & 12.9 & 53.8 & 19.8 & 59.1 \\
    \rowcolor{m} \textbf{\gls{ours}} & 13.7 & \underline{71.6} & 12.5 & \underline{41.2} & 12.9 & \underline{62.1} & 19.8 & 60.7 \\
    \rowcolor{m} \textbf{\gls{ours}} $+$ \textbf{\gls{oursearch}} & 13.7 & \textbf{78.4} & 12.5 & \textbf{57.5} & 13.0 & \textbf{68.8} & 19.8 & \textbf{74.6} \\
    \bottomrule
    \bottomrule
    \end{tabular}
    }
    \label{tab:classification_40}
\end{table}

\begin{table}[t]
    \centering
    \caption{Top-1 accuracy for various models on \gls{imagenet} at \textbf{uniform} compression (60\% linear layers FLOPs remaining) for related \gls{svd} methods. In addition to uniform compression, we report \gls{ours} with our \textbf{\gls{oursearch}} to isolate the effect of decomposition and search.}
    \resizebox{\textwidth}{!}{%
    \begin{tabular}{l  c c   c c   c c   c c }
    \toprule
    \toprule
    \multirow{2}{*}{Method} & \multicolumn{2}{c}{\textbf{DeiT-B}} & \multicolumn{2}{c}{\textbf{Swin-B}} & \multicolumn{2}{c}{\textbf{ConvNeXt-B}} & \multicolumn{2}{c}{\textbf{MambaVis.-B}} \\
    & GFLOPs & Top-1$\uparrow$ & GFLOPs & Top-1$\uparrow$ & GFLOPs & Top-1$\uparrow$ & GFLOPs & Top-1$\uparrow$ \\
    \midrule
    Baseline & 33.9 & 83.3 & 30.3 & 85.1 & 30.8 & 85.8 & 29.9 & 83.9 \\
    \cmidrule{2-9}
    PELA~\cite{pela} & 20.4 & 75.3 & 18.4 & 46.0 & 18.9 & 56.0 & 23.2 & 76.3 \\
    FW-SVD~\cite{hsu2022fwsvd} & 20.4 & 78.0 & 18.4 & 61.0 & 18.9 & 67.7 & 23.2 & 75.9 \\
    ASVD~\cite{yuan2024asvd} & 20.4 & 77.6 & 18.4 & 59.6 & 18.9 & 72.2 & 23.2 & 76.7 \\
    SVD-LLM~\cite{wang2024svd-llm} & 20.4 & 78.8 & 18.4 & 73.9 & 18.9 & 79.7 & 23.2 & 77.4 \\
    GFWSVD~\cite{chekalina2025generalizedfisherweightedsvdscalable} & 20.4 & 77.3 & 18.4 & 52.2 & 18.9 & 49.0 & 23.2 & 77.9 \\
    FLAR-SVD~\cite{2025_flar-svd} & 20.4 & 78.8 & 18.4 & 73.9 & 18.9 & 79.7 & 23.2 & 77.4 \\
    \rowcolor{m} \textbf{\gls{ours}} & 20.4 & \underline{80.2} & 18.4 & \underline{76.6} & 18.9 & \underline{81.5} & 23.2 & \underline{78.1} \\
    \rowcolor{m} \textbf{\gls{ours}} $+$ \textbf{\gls{oursearch}} & 20.4 & \textbf{82.4} & 18.4 & \textbf{81.1} & 18.9 & \textbf{83.1} & 23.2 & \textbf{82.1} \\
    \bottomrule
    \bottomrule
    \end{tabular}
    }
    \label{tab:classification_60}
\end{table}

%% file: sec/supplementary/calibration_data_impact.tex
To evaluate the impact of the calibration set size, we test a range of differnt sets ranging from just 1024 images up to 8192 samples, which is half of what we have used for the evaluation in the main material. From the results the general accuracy increase observed with increasing calibration data samples shows, that all \gls{svd} methods benefit from additional data. MambaVision is the only exception with both GFWSVD and FLAR-SVD being almost immune to the loss of data. Moreover, the shrinkage used by FLAR-SVD gives it a clear edge over its non-shrinking SVD-LLM counterpart. However, overall, our approach remains competitive, even when fewer data is available. Moreover, adding search, recovers most of the performance across calibration set sizes, showing off the strong capabilities of the combination of \gls{ours} and \gls{oursearch}.
\begin{table}[]
    \centering
    \caption{Top-1 accuracy of different models for the best performing approaches under \textbf{uniform} compression with different calibration sample counts. In addition to uniform compression, \gls{ours} is reported in combination with \gls{oursearch}, denoted as \textbf{+s}}
    \label{supp:tab:calibration_impact}
    \begin{tabular}{c|c|c|c|c|c}
\toprule
\midrule
\multirow{2}{*}{Calib Cnt} & \multirow{2}{*}{Method} & \multicolumn{1}{c|}{DeiT-B} & \multicolumn{1}{c|}{Swin-B} & \multicolumn{1}{c|}{CoNxt-B} & \multicolumn{1}{c}{MamVis-B} \\
 &  & Top-1 $\uparrow$ & Top-1 $\uparrow$ & Top-1 $\uparrow$ & Top-1 $\uparrow$ \\
\midrule
\multirow{5}{*}{1024} & SVD-LLM~\cite{wang2024svd-llm} & 73.8 & 69.9 & 78.2 & 72.3 \\
& FLAR-SVD~\cite{2025_flar-svd} & \underline{74.4} & 71.0 & 78.7 & 74.6 \\
 & GFWSVD~\cite{chekalina2025generalizedfisherweightedsvdscalable} & 73.5 & 54.9 & 64.6 & \underline{77.8} \\
 \rowcolor{m} &\textbf{\ours} & 73.7 & \underline{71.6} & \underline{79.3} & 71.8 \\
 \rowcolor{m} &\textbf{\ours} \textbf{+s} & \textbf{81.0} & \textbf{80.6} & \textbf{82.8} & \textbf{81.7} \\
\midrule
\multirow{5}{*}{4096} & SVD-LLM~\cite{wang2024svd-llm} & 75.5 & 72.8 & 79.3 & 75.9 \\
 & FLAR-SVD~\cite{2025_flar-svd} & \underline{75.8} & 73.4 & 79.4 & 77.0 \\
 & GFWSVD~\cite{chekalina2025generalizedfisherweightedsvdscalable} & 73.9 & 52.7 & 60.2 & \underline{78.1} \\
 \rowcolor{m} &\textbf{\ours} & 75.7 & \underline{75.1} & \underline{80.7} & 76.0 \\
 \rowcolor{m} &\textbf{\ours} \textbf{+s} & \textbf{81.1} & \textbf{80.9} & \textbf{83.0} & \textbf{81.9} \\
\midrule
\multirow{5}{*}{8192} & SVD-LLM~\cite{wang2024svd-llm} & 76.0 & 73.5 & 79.7 & 76.8 \\
& FLAR-SVD~\cite{2025_flar-svd} & 76.3 & 73.9 & 79.8 & 77.5 \\
 & GFWSVD~\cite{chekalina2025generalizedfisherweightedsvdscalable} & 73.1 & 50.5 & 65.4 & \underline{78.1} \\
 \rowcolor{m} &\textbf{\ours} & \underline{76.4} & \underline{76.1} & \underline{81.0} & 77.2 \\
 \rowcolor{m} &\textbf{\ours} \textbf{+s} & \textbf{81.0} & \textbf{80.9} & \textbf{83.1} & \textbf{82.1} \\
\midrule
\bottomrule
\end{tabular}
\end{table}

%% file: tables/individual_contributions.tex
\begin{figure*}[t]
    \centering
    \begin{minipage}[t]{0.23\linewidth}
        \vspace{0pt}
        \centering
        \begin{tikzpicture}
            \node[anchor=south west, inner sep=0] (containerimg) at (0,0) {%
                \includegraphics[width=\linewidth]{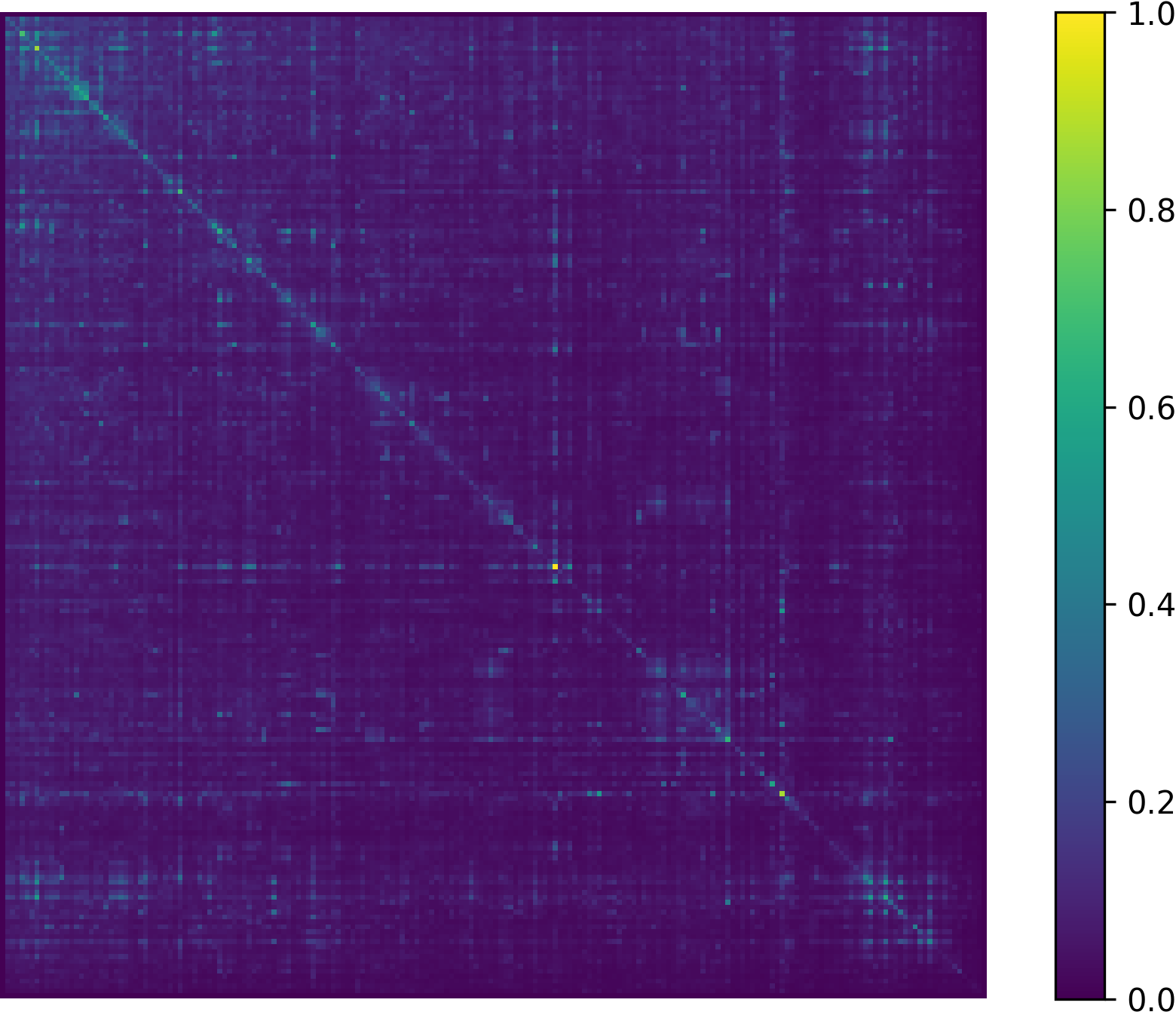}%
            };
            \node[anchor=south west, text=white, font=\sffamily\tiny, inner sep=1pt] 
                  at ([xshift=2pt, yshift=2pt] containerimg.south west) 
                  {Qwen3-1.7B};
        \end{tikzpicture}

        \vspace{0.2cm}
        \caption{Qualitative cross-token moments (Qwen3-1.7B).}
        \label{fig:qwen_moments}
    \end{minipage}%
    \hfill
    \begin{minipage}[t]{0.31\linewidth}
        \vspace{0pt}
        \centering
        \makeatletter\def\@captype{table}\makeatother 
        \caption{Performance comparison of \gls{ours} against other approaches on Qwen3-1.7B benchmark datasets.}
        \label{tab:qwen_performance}
        
        \vspace{0.15cm}
        \resizebox{0.78\linewidth}{!}{%
        \begin{tabular}{lcc}
            \toprule
            \toprule
            \textbf{Method} & ppl.$\downarrow$~ & Acc.$\uparrow$ \\ \midrule
            SVD-LLM & 57.7 & 35.4\\
            GFWSVD~~ & $>$1k & 31.3\\
            Shampoo$^2$ & 44.6 & 34.4 \\
            \rowcolor{m} \textbf{\gls{ours}}  & 40.8 & 36.5 \\ 
            \bottomrule
            \bottomrule
        \end{tabular}%
        }
    \end{minipage}%
    \hfill
    \begin{minipage}[t]{0.42\linewidth}
        \vspace{0pt}
        \centering
        \makeatletter\def\@captype{table}\makeatother 
        \caption{Individual contributions of our decomposition and search on Top-1 accuracy and total time (tt) in minutes for DeiT-B and Swin-B compressed to 50\%.}
        \label{tab:individual_contribution}
        
        \vspace{0.15cm}
        \resizebox{\linewidth}{!}{%
        \begin{tabular}{l c c c c}
        \toprule
        \midrule
                                       & \multicolumn{2}{c}{DeiT-B}  & \multicolumn{2}{c}{Swin-B}  \\\cmidrule{2-5}
                                       & Top-1$\uparrow$ & tt[m]$\downarrow$   & Top-1$\uparrow$ & tt[m]$\downarrow$  \\
        \midrule
        Full model                 & 81.8 & -    & 85.1      & - \\
        \cmidrule{1-5}
        \textbf{our} \ours~(uniform) & 77.5  & 8.3  & 76.6   & 10.9 \\
        + our ILP search           & 81.3 & 21.0 & 81.0 & 39.2 \\
        + fewer sensitivity tests  & 81.2 & 14.8 & 80.6   & 27.3 \\
        + sensitivity interpolation& 81.3 & 14.8 & 81.1 & 27.3 \\
        \bottomrule
        \bottomrule
        \end{tabular}%
        }
    \end{minipage}
\end{figure*}

%% file: sec/supplementary/llm_stuff.tex
Having validated \gls{ours} for vision despite reduced \gls{fim} modeling accuracy, we extend our analysis to Large Language Models (LLMs).\
We first examine the qualitative cross-token moments of Qwen3-1.7B in \cref{fig:qwen_moments}. Compared to the structured off-diagonal patterns observed in vision models (see \cref{fig:motivation:a} and Supplement~\cref{supp:fig:extended_motivation_grid}), the Qwen3 patterns appear less spatially regular. This is consistent with the modality difference: language tokens do not have fixed geometric positions across samples, so cross-token dependencies are more sequence- and context-dependent than the spatial correlations induced by image patches.

Despite these differences, \gls{ours} still outperforms \acrshort{svd-llm}, \acrshort{gfwsvd}, and Shampoo$^2$ at an identical compression ratio on Wikitext perplexity and zero-shot language benchmarks. As shown in \cref{tab:qwen_performance}, \gls{ours} achieves the best performance across both metrics, with a perplexity score of 40.8 and a zero-shot accuracy of 36.5\%. Whereas similar to the observations in \cref{sec:fisher_aprrox_main_content}, \acrshort{gfwsvd} is numerically unstable, leading to degraded performance (perplexity $>1$k), Shampoo$^2$ avoids this collapse (44.6 ppl, 34.4\% Acc). This suggests that the proposed combination of token-local aggregation and within-token activation-gradient coupling may remain beneficial beyond vision models, even when cross-token structure is less spatially regular. Supplement Section~\ref{supp:subsec:LLM_eval_details} provides additional evaluation details.

%% file: sec/supplementary/preliminaries.tex
\paragraph{KFAC-expand (token-local, decoupled).}
KFAC~\cite{kfac-martens15} in the weight-sharing setting (often termed \emph{KFAC-expand}~\cite{runa-kfacreduce})
combines two approximations:
(i) it discards cross-token mixed moments (drop $t\neq s$ in \cref{eq:ws_fisher_sum}), and
(ii) it assumes independence between activations and gradients at the same token.
This yields
\begin{equation}
\mathbf{F}
\;\approx\;
\mathbb{E}\!\left[\sum_{t=1}^{T}(\mathbf{x}_t\mathbf{x}_t^\top)\otimes(\mathbf{g}_t\mathbf{g}_t^\top)\right]
\;\approx\;
\mathbb{E}[\mathbf{x}\mathbf{x}^\top]\otimes \mathbb{E}[\mathbf{g}\mathbf{g}^\top],
\label{eq:kfac_expand}
\end{equation}
up to an immaterial scalar factor in $T$.
KFAC-expand is efficient and enforces a strong token-local prior, but it removes within-token activation-gradient coupling.

\paragraph{KFAC-reduce (cross-token in factors, still decoupled).}
Eschenhagen et al.~\cite{runa-kfacreduce} derive a second variant, \emph{KFAC-reduce}, motivated by approximating
a \emph{sum of Kronecker products} by a \emph{Kronecker product of sums} over the shared dimension.
In our notation, this corresponds to the replacement
\begin{equation}
\sum_{t=1}^{T} \mathbf{x}_t\otimes \mathbf{g}_t
\;\approx\;
\Big(\sum_{t=1}^{T}\mathbf{x}_t\Big)\otimes \Big(\sum_{t=1}^{T}\mathbf{g}_t\Big),
\label{eq:kfac_reduce_vec_approx}
\end{equation}
which induces the Kronecker factors
\begin{equation}
\mathbf{F}\;\approx\;
\mathbb{E}\!\Big[\Big(\sum\nolimits_t \mathbf{x}_t\Big)\Big(\sum\nolimits_s \mathbf{x}_s\Big)^\top\Big]
\;\otimes\;
\mathbb{E}\!\Big[\Big(\sum\nolimits_t \mathbf{g}_t\Big)\Big(\sum\nolimits_s \mathbf{g}_s\Big)^\top\Big],
\label{eq:kfac_reduce}
\end{equation}
again up to scaling by $T$.
Unlike \cref{eq:kfac_expand}, the activation and gradient factors in \cref{eq:kfac_reduce} include \emph{cross-token mixed moments}
(e.g., $\mathbb{E}[\mathbf{x}_t\mathbf{x}_s^\top]$ for $t\neq s$), but the approximation remains \emph{decoupled}:
it still does not model activation-gradient dependence.

\paragraph{Optimal rank-1 Kronecker approximation (rearrangement).}
A complementary approach chooses $(\mathbf{A},\mathbf{B})$ as the best rank-1 Kronecker fit in Frobenius norm,
\begin{equation}
\min_{\mathbf{A},\mathbf{B}}\;\big\|\mathbf{F}-\mathbf{A}\otimes\mathbf{B}\big\|_F^2,
\label{eq:oka}
\end{equation}
whose solution is given by the leading singular vectors of a rearranged matrix $\mathcal{R}(\mathbf{F})$~\cite{VanLoan1993}.
Since $\mathbf{F}$ cannot be formed explicitly, iterative methods compute this component using only matrix-vector products.
Lanczos-based estimators (e.g., used by GFWSVD~\cite{chekalina2025generalizedfisherweightedsvdscalable}) approximate the leading component more accurately,
while Shampoo$^2$~\cite{shampoo-squared} corresponds to a single power-iteration step from identity initialization.
By construction, these methods target the globally dominant energy of $\mathbf{F}$ in \cref{eq:ws_fisher_sum},
which can implicitly reflect cross-token mixed moments and activation-gradient couplings.

%% file: sec/supplementary/kroneker_fwsvd_closed_form_solution.tex
\label{supp:sec:fw-svd-objective}

Here we provide the full theorem and proof justifying that the optimal, loss-aware SVD compression problem reduces to finding the Kronecker factors $\mathbf{A}$ and $\mathbf{B}$ of the expected Fisher Information Matrix. This is a modified version of the theorem presented in related work~\cite{chekalina2025generalizedfisherweightedsvdscalable}.

\begin{theorem}
\label{theorem:kfacSVDmain_supp}
Let $\mathbf{W}$ be the weight matrix of a linear layer. Assume:
\begin{enumerate}
    \item The Hessian $\mathbf{H}$ is block-diagonal with respect to other layers.
    \item The expected Hessian equals the Fisher, $\mathbb{E}[\mathbf{H}_\mathbf{W}] =\mathbf{F}_\mathbf{W}$. This holds for MLE losses under standard regularity conditions.
    \item The expected Fisher factorizes as $\mathbb{E}[\mathbf{F}_\mathbf{W}] = \mathbf{A} \otimes \mathbf{B}$, with $\mathbf{A}, \mathbf{B}$ positive definite and $\mathbf{A} = \mathbf{L}_\mathbf{A}^\top \mathbf{L}_\mathbf{A}$, $\mathbf{B} = \mathbf{L}_\mathbf{B}^\top \mathbf{L}_\mathbf{B}$.
\end{enumerate}
The optimal rank-$k$ approximation $\widetilde{\mathbf{W}}_k$ minimizing the second-order loss $\mathbb{E}\!\left[\Delta \mathcal{L}\right]$ is
\[
\boxed{\widetilde{\mathbf{W}}_k = \mathbf{L}_\mathbf{B}^{-1} \, \mathbf{U}_k \, \boldsymbol{\Sigma}_k \, \mathbf{V}_k^\top \, \mathbf{L}_\mathbf{A}^{-\top}}
\]
where $\mathbf{U}_k \boldsymbol{\Sigma}_k \mathbf{V}_k^\top$ is the rank-$k$ truncated SVD of the "whitened" matrix $\mathbf{W}' = \mathbf{L}_\mathbf{B} \, \mathbf{W} \, \mathbf{L}_\mathbf{A}^\top$.
\end{theorem}

\begin{proof}
Let $\Delta \mathbf{W} = \mathbf{W} - \widetilde{\mathbf{W}}_k$. A second-order Taylor expansion of the loss is
\begin{multline*}
\mathbb{E}[\Delta \mathcal{L}] \approx
\langle \mathbb{E}[\nabla \mathcal{L}], \operatorname{vec}(\Delta \mathbf{W}) \rangle
+ \tfrac{1}{2}\,\operatorname{vec}(\Delta \mathbf{W})^\top \mathbb{E}[\mathbf{H}_\mathbf{W}]\,\operatorname{vec}(\Delta \mathbf{W}).
\end{multline*}
At a stationary point, $\mathbb{E}[\nabla \mathcal{L}] = 0$. Using Assumptions (1--3), the objective simplifies to minimizing
\[
\mathbb{E}[\Delta \mathcal{L}]
\approx \tfrac{1}{2}\,\operatorname{vec}(\Delta \mathbf{W})^\top (\mathbf{A} \otimes \mathbf{B})\,\operatorname{vec}(\Delta \mathbf{W}).
\]
Using standard matrix identities (e.g., $\operatorname{vec}(\mathbf{X}\mathbf{Y}\mathbf{Z}) = (\mathbf{Z}^\top \otimes \mathbf{X})\operatorname{vec}(\mathbf{Y})$ and \\ $\operatorname{tr}(\mathbf{X}^\top \mathbf{Y}) = \operatorname{vec}(\mathbf{X})^\top \operatorname{vec}(\mathbf{Y})$), this is equivalent to minimizing the Frobenius norm of the whitened error:
\begin{align*}
& \operatorname{vec}(\Delta \mathbf{W})^\top (\mathbf{A} \otimes \mathbf{B}) \operatorname{vec}(\Delta \mathbf{W}) \\
&\quad = \operatorname{tr}(\Delta \mathbf{W}^\top \mathbf{B}\,\Delta \mathbf{W}\,\mathbf{A}) \\
&\quad = \operatorname{tr}(\Delta \mathbf{W}^\top \mathbf{L}_\mathbf{B}^\top \mathbf{L}_\mathbf{B}\,\Delta \mathbf{W}\mathbf{L}_\mathbf{A}^\top \mathbf{L}_\mathbf{A}) \\
&\quad = \operatorname{tr}\big( (\mathbf{L}_\mathbf{B}\,\Delta \mathbf{W}\,\mathbf{L}_\mathbf{A}^\top)^\top(\mathbf{L}_\mathbf{B}\,\Delta \mathbf{W}\,\mathbf{L}_\mathbf{A}^\top) \big) \\
&\quad = \left\lVert\, \mathbf{L}_\mathbf{B}\,(\mathbf{W} - \widetilde{\mathbf{W}}_k)\, \mathbf{L}_\mathbf{A}^\top \right\rVert_F^{2}.
\end{align*}
By the Eckart–Young–Mirsky theorem, this objective is minimized by taking the rank-$k$ SVD of the whitened matrix $\mathbf{W}'$:
\[
    \mathbf{W}' = \mathbf{L}_\mathbf{B}\,\mathbf{W}\,\mathbf{L}_\mathbf{A}^\top.
\]
Let its truncated SVD be $\mathbf{W}'_k = \mathbf{U}_k\,\boldsymbol{\Sigma}_k\,\mathbf{V}_k^\top$. The optimal $\widetilde{\mathbf{W}}_k$ is found by un-whitening:
\[
\begin{aligned}
\mathbf{L}_\mathbf{B}\,\widetilde{\mathbf{W}}_k\,\mathbf{L}_\mathbf{A}^\top &= \mathbf{U}_k\,\boldsymbol{\Sigma}_k\,\mathbf{V}_k^\top \\
\implies \quad \widetilde{\mathbf{W}}_k &= \mathbf{L}_\mathbf{B}^{-1}\,\mathbf{U}_k\,\boldsymbol{\Sigma}_k\,\mathbf{V}_k^\top\,\mathbf{L}_\mathbf{A}^{-\top}.
\end{aligned}
\]
(Note: $\mathbf{L}^{-\top} = (\mathbf{L}^\top)^{-1}$).
\end{proof}

%% file: sec/supplementary/method_algorithms.tex
\subsubsection{Algorithmic Description of \ours{}}\label{supp:sec:method_pseudo_code}
We provide the detailed pseudocode for \ours{} in two parts. \cref{supp:algo:our_framework} outlines the main compression procedure, and \cref{supp:algo:our_whitening} details the computation of the FIM Kronecker factors $\mathbf{A}$ and $\mathbf{B}$ used for whitening the weight matrix before decomposition. For clarity, the factors $\mathbf{A}$ and $\mathbf{B}$ from \cref{sec:methodology} correspond to the \textit{column scaling} and \textit{row scaling} terms, respectively, in the pseudocode.

\begin{algorithm}[ht]
\caption{Pseudocode of \ours{}}
\begin{algorithmic}[1] 
\State \textbf{Input:} $M$: Original model, $C$: Calibration Data, $k$: Target rank
\State \textbf{Input:} $\alpha_R=0.7, \alpha_C=0.1$: Row and column regularization strengths
\State \textbf{Output:} $M'$: Compressed model
\Procedure{\ours{}}{$M, C, k$}
    \State $\text{Set}_R, \text{Set}_C \gets \textproc{ComputeFIMKronFacts}(M, C)$ \Comment{See \cref{supp:algo:our_whitening}}
    \For{$W$ \textbf{in} $M$ \text{to compress}}
        \State $R \gets \text{Set}_R(W)$, $C \gets \text{Set}_C(W)$ \# Get layer-specific FIM factors

        \State $R_{\text{reg}} \gets (1 - \alpha_R)R + \alpha_R \cdot \text{mean}(\text{diag}(R)) \cdot I$ \# Regularize w/ shrinkage
        \State $C_{\text{reg}} \gets (1 - \alpha_C)C + \alpha_C \cdot \text{mean}(\text{diag}(C)) \cdot I$

        \State \# Compute whitening matrices
        \State $S_R \gets \operatorname{Cholesky}(R_{\text{reg}})$
        \State $S_C \gets \operatorname{Cholesky}(C_{\text{reg}})$

        \State \# Apply two-sided whitening
        \State $W_S \gets S_R W S_C$

        \State $U, \Sigma, V \gets \operatorname{SVD}(W_S)$
        \State $\Sigma_k \gets \operatorname{Trunc.}(\Sigma, k)$

        \State \# Compute two matrices \& unwhiten
        \State $W'_u \gets S_R^{-1} U (\Sigma_k)^{1/2}$ 
        \State $W'_v \gets (\Sigma_k)^{1/2} V^T S_C^{-1}$ 

        \State \# Replace $W$ with $W'_u$ and $W'_v$
        \State $M'(W) \gets W'_u, W'_v$ 
    \EndFor
    \State \Return{$M'$}
\EndProcedure
\end{algorithmic}
\label{supp:algo:our_framework}
\end{algorithm}
\begin{algorithm}[ht]
\caption{Pseudocode of \gls{ours}'s Kronecker Factor Computation}
\begin{algorithmic}[1] 
\State \textbf{Input:} $M$: Original model, $C$: Calibration Data
\State \textbf{Output:} $\text{Set}_R, \text{Set}_C$: Sets of row and col. FIM factors
\Procedure{ComputeFIMKronFacts}{$M,C$}
    \State Initialize $R_W \gets 0$, $C_W \gets 0$ for all $W$
    
    \For{$(data, target)$ \textbf{in} $C$}
        \State \# Generate $X_W, G_W$ for each layer
        \State $out \gets M(data)$
        \State $loss \gets \text{CrossEntropyLoss}(out, target)$
        \State $loss.backward()$
        
        \For{$W$ \textbf{in} $M$ \text{to compress}}
            \State $X \gets \text{InputActivation}(W)$
            \State $G \gets \text{OutputGradient}(W)$
            \State $G' \gets \text{grad\_clip}(G)$
            \State $\Delta R_{\text{batch}} \gets G'^T \text{diag}(XX^T) G'$ \# Row factor computation
            \State $\Delta C_{\text{batch}} \gets X^T \text{diag}(G'G'^T) X$ \# Column factor computation
            \State $\Delta C_{\text{batch}} \gets \text{Normalize}(\Delta C_{\text{batch}})$ \# Normalize

            \State \# Accumulate
            \State $R_W \gets R_W + \text{mean}_{\text{batch}}(\Delta R_{\text{batch}})$ 
            \State $C_W \gets C_W + \text{mean}_{\text{batch}}(\Delta C_{\text{batch}})$
        \EndFor
    \EndFor
    
    \State $\text{Set}_R \gets \{R_W\}$, $\text{Set}_C \gets \{C_W\}$
    \State \Return $\text{Set}_R, \text{Set}_C$
\EndProcedure
\end{algorithmic}
\label{supp:algo:our_whitening}
\end{algorithm}

%% file: sec/supplementary/search_extensions.tex
\label{supp:sec:interpolation_details}
Our search algorithm relies on discrete cost-error profiles $\mathcal{D}_i = \{(c_j, e_j)\}$ generated by sampling sensitivity at fixed intervals (e.g., $0.1, 0.2 \dots 0.9$). To enable the Integer Linear Programming (ILP) solver to select fine-grained ranks, we interpolate these profiles into a denser set $\widetilde{\mathcal{D}}_i$.

\input{figures/supp_sensitivity_interpolation}

\paragraph{Interpolation Challenges.} As shown in \cref{supp:fig:sensitivity_analysis:qkv_sensitivity} (red curve), cubic splines can create artificial local minima that mislead the optimizer. Differently, a piecewise-linear interpolation of $\mathcal{D}_i$, preserves monotonicity and avoids such artifacts. However, since the cost-error curve is approximately convex, linear interpolation ignores the curvature and may therefore wrongly approximate between points.

Combining both perspectives, we introduce a locally consistent interpolation (Algorithm~\ref{alg:interpolate}) to capture local curvature without compromising stability. Concretely, for sliding windows of three points, we fit local cubic splines and retain a spline only if it is monotonous. 
In this way, the resulting profile is smooth and well reflects the non-linear error frontier (Blue Curve, \cref{supp:fig:sensitivity_analysis:qkv_sensitivity}). Nevertheless, we observe only marginal performance gains over the simple linear interpolation, making the latter a simple, practical option for future use.

\paragraph{Interpolation Algorithm} To achieve locally consistent and artifact-free interpolation, we adopt the sliding-window strategy detailed in Algorithm~\ref{alg:interpolate}. Specifically, for a profile $\mathcal{D}_i$, we perform:

\begin{enumerate}
    \item \textbf{Window Generation:} We generate overlapping windows of three adjacent measurement points across the profile.
    \item \textbf{Local Fitting:} For each triplet, we fit a local cubic spline $f_{\text{cubic}}$ and define a corresponding reference piecewise-linear function $f_{\text{linear}}$.
    \item \textbf{Deviation Check:} We evaluate both functions at fine-grained query points $c_q \in [c_0, c_2]$ to compute the absolute deviation $|f_{\text{cubic}}(c_q) - f_{\text{linear}}(c_q)|$.
    \item \textbf{Selection:} For any query point $c_q$ covered by multiple windows, we select the cubic prediction that yields the smallest deviation from the linear baseline.
\end{enumerate}

\begin{algorithm}[t]
\caption{Locally Consistent Interpolation}
\label{alg:interpolate}
\begin{algorithmic}[1]
\Require Discrete profile $\mathcal{D}_i = \{(c_{ij}, e_{ij})\}_{j=1}^{n_i}$ (sorted), query resolution $\delta$
\Ensure Interpolated profile measured profile $\widetilde{\mathcal{D}}_i$

\State Initialize $E_{\text{best}} \gets \emptyset$
\State Initialize $M_{\text{error}} \gets \emptyset$ 

\State Let $n' = |\mathcal{D}'_i|$

\For{$j = 1$ to $n' - 2$}
    \State \# Define window
    \State $(c_0,e_0),(c_1,e_1),(c_2,e_2) \gets \mathcal{D}'_i[j{:}j{+}2]$
    \State Fit cubic spline $f_{\text{cubic}}$ to the window
    \State Define piecewise-linear $f_{\text{linear}}$ over $[c_0, c_2]$
    
    \State $C_{\text{query}} \gets \{c_q \mid c_q \in [c_0, c_2] \text{ with step } \delta\}$
    
    \State \# Calculate the total L1 error for this window's fit
    \State $d_{\text{wid}} \gets \sum_{c_q \in C_{\text{query}}} |f_{\text{cubic}}(c_q) - f_{\text{linear}}(c_q)|$
    
    \State \# Update all points in window if its fit is best
    \For{each $c_q \in C_{\text{query}}$}
        \If{$c_q \notin \text{keys}(M_{\text{error}})$ \textbf{or} $d_{\text{wid}} < M_{\text{error}}[c_q]$}
            \State \# Store/update prediction and best error
            \State $M_{\text{error}}[c_q] \gets d_{\text{wid}}$
            \State $E_{\text{best}}[c_q] \gets f_{\text{cubic}}(c_q)$
        \EndIf
    \EndFor
\EndFor

\State \# Convert to original format, readd no compression
\State $\widetilde{\mathcal{D}}_i \gets \{(c, e) \mid c \in \text{keys}(E_{\text{best}}), e = E_{\text{best}}[c]\}$
\State $\widetilde{\mathcal{D}}_i \gets \widetilde{\mathcal{D}}_i \cup \{(1.0, 0)\}$
\State \Return $\widetilde{\mathcal{D}}_i$
\end{algorithmic}
\end{algorithm}

%% file: figures/supp_sensitivity_interpolation.tex
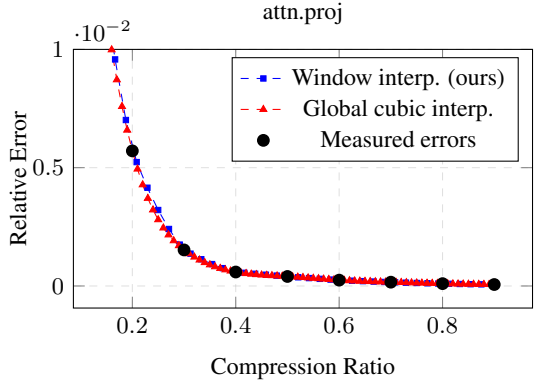
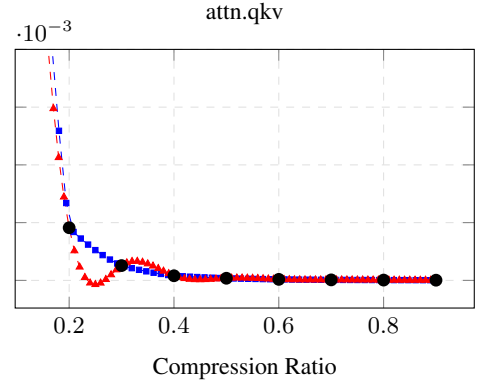
\begin{figure*}[htb!]
\centering

\begin{subfigure}[b]{0.47\textwidth}
    \centering
    \begin{tikzpicture}[trim axis left, trim axis right]
    \begin{axis}[
        width=0.95\textwidth,
        height=5cm,
        title={attn.proj},
        title style={font=\small},
        label style={font=\small},
        legend style={font=\footnotesize}, 
        xlabel={Compression Ratio},
        ylabel={Relative Error},
        grid=both,
        grid style={dashed, line width=.2pt, opacity=0.5},
        legend pos=north east,
        ymax=1e-2,
        tick label style={font=\footnotesize},
        ylabel style={font=\footnotesize, yshift=-0.5cm},
    ]
    \addplot [
        dashed, mark=square*, mark size=1.0pt, color=blue, mark options={solid}
    ] coordinates {
        (0.10416666666666667, 0.02015491129754486) (0.125, 0.016145987159688957) (0.14583333333333334, 0.012619665669060648) (0.16666666666666666, 0.009575946825659938) (0.1875, 0.007014830629486824) (0.20833333333333334, 0.005236813357745025) (0.22916666666666666, 0.004153112019442537) (0.25, 0.0032105402860906906) (0.2708333333333333, 0.00240909815768949) (0.2916666666666667, 0.001748785634238932) (0.3125, 0.0013666199765793863) (0.3333333333333333, 0.0011299336781828771) (0.3541666666666667, 0.0009256331269271969) (0.375, 0.000753718322812347) (0.3958333333333333, 0.0006141892658383266) (0.4166666666666667, 0.0005566175645476) (0.4375, 0.0005159404827281833) (0.4583333333333333, 0.0004766668927105558) (0.4791666666666667, 0.0004387967944947175) (0.5, 0.00040233018808066845) (0.5208333333333334, 0.0003672670734684086) (0.5416666666666666, 0.00033360745065793814) (0.5625, 0.00030135131964925677) (0.5833333333333334, 0.00027049868044236455) (0.6041666666666666, 0.00024273915698813653) (0.625, 0.0002229441747658711) (0.6458333333333334, 0.00020422360705651946) (0.6666666666666666, 0.00018657745386008173) (0.6875, 0.0001700057151765577) (0.7083333333333334, 0.00015454461319879754) (0.7291666666666666, 0.00014018344591527746) (0.75, 0.0001268555315618869) (0.7708333333333334, 0.00011456087013862592) (0.7916666666666666, 0.00010329946164549462) (0.8125, 9.307130608249284e-05) (0.8333333333333334, 8.387640344962064e-05) (0.8541666666666666, 7.571475374687809e-05) (0.875, 6.858635697426509e-05) (0.8958333333333334, 6.249121313178168e-05)
    };
    \addplot [
        dashed, mark=triangle*, mark size=1.5pt, color=red, mark options={solid}
    ] coordinates {
        (0.1, 0.021014608442783356) (0.11, 0.018721834586361362) (0.12000000000000001, 0.016622967623779548) (0.13, 0.014708816972380249) (0.14, 0.012970192049505804) (0.15000000000000002, 0.011397902272498558) (0.16, 0.009982757058700855) (0.17, 0.008715565825455024) (0.18, 0.007587137990103413) (0.19, 0.006588282969988359) (0.2, 0.005709810182452202) (0.21000000000000002, 0.004942529044837283) (0.22, 0.004277248974485945) (0.23, 0.0037047793887405246) (0.24000000000000002, 0.0032159297049433625) (0.25, 0.002801509340436801) (0.26, 0.0024523277125631775) (0.27, 0.002159194238664835) (0.28, 0.0019129183360841124) (0.29000000000000004, 0.0017043094221633505) (0.30000000000000004, 0.0015241769142448893) (0.31, 0.0013646532734346687) (0.32, 0.0012231631358930212) (0.33, 0.0010984541815438806) (0.33999999999999997, 0.0009892740903111794) (0.35, 0.0008943705421188495) (0.36, 0.0008124912168908241) (0.37, 0.0007423837945510359) (0.38, 0.0006827959550234179) (0.39, 0.0006324753782319021) (0.4, 0.0005901697441004217) (0.41000000000000003, 0.0005546966379086371) (0.42000000000000004, 0.0005251532663591207) (0.43000000000000005, 0.0005007067415101727) (0.44000000000000006, 0.0004805241754200932) (0.45000000000000007, 0.0004637726801471824) (0.45999999999999996, 0.0004496193677497409) (0.47, 0.0004372313502860684) (0.48, 0.0004257757398144654) (0.49, 0.00041441964839323197) (0.5, 0.00040233018808066845) (0.51, 0.0003888634330694912) (0.52, 0.00037413130609008165) (0.53, 0.00035843469200723713) (0.54, 0.0003420744756857553) (0.55, 0.0003253515419904335) (0.56, 0.00030856677578606936) (0.5700000000000001, 0.00029202106193746037) (0.58, 0.00027601528530940415) (0.59, 0.0002608503307666979) (0.6, 0.00024682708317413926) (0.61, 0.0002341714226276963) (0.62, 0.00022280921014801908) (0.63, 0.00021259130198692823) (0.64, 0.00020336855439624435) (0.65, 0.00019499182362778809) (0.66, 0.00018731196593338) (0.67, 0.0001801798375648407) (0.6799999999999999, 0.00017344629477399092) (0.69, 0.0001669621938126511) (0.7, 0.00016057839093264192) (0.71, 0.0001541753814628464) (0.72, 0.00014775221704039723) (0.73, 0.0001413375883794894) (0.74, 0.00013496018619431802) (0.75, 0.00012864870119907814) (0.76, 0.0001224318241079648) (0.77, 0.00011633824563517311) (0.78, 0.00011039665649489811) (0.79, 0.0001046357474013348) (0.8, 9.908420906867832e-05) (0.8099999999999999, 9.377073221112374e-05) (0.82, 8.872400754286603e-05) (0.83, 8.39727257781003e-05) (0.84, 7.95455776310216e-05) (0.85, 7.5471253815825e-05) (0.86, 7.177844504670557e-05) (0.87, 6.849584203785833e-05) (0.88, 6.565213550347838e-05) (0.89, 6.327601615776077e-05) (0.9, 6.139617471490055e-05)
    };
    \addplot [
        only marks, mark=*, mark size=2pt, color=black, line width=0.8pt, mark options={solid}
    ] coordinates {
        (0.1, 0.021014608442783356) (0.2, 0.005709810182452202) (0.3, 0.0015241769142448902) (0.4, 0.0005901697441004217) (0.5, 0.00040233018808066845) (0.6, 0.00024682708317413926) (0.7, 0.00016057839093264192) (0.8, 9.908420906867832e-05) (0.9, 6.139617471490055e-05)
    };
    
    \legend{
        {Window interp. (ours)},
        {Global cubic interp.},
        {Measured errors}
    }
    
    \end{axis}
    \end{tikzpicture}
    \caption{Sensitivity for \texttt{blocks.1.attn.proj}}
    \label{supp:fig:sensitivity_analysis:proj_sensitivity}
\end{subfigure}
\begin{subfigure}[b]{0.47\textwidth}
    \centering
    \begin{tikzpicture}[trim axis left, trim axis right]
    \begin{axis}[
        width=0.95\textwidth,
        height=5cm,
        title={attn.qkv},
        title style={font=\small},
        label style={font=\small},
        draw=none, 
        xlabel={Compression Ratio},
        yticklabels={,,},
        grid=both,
        grid style={dashed, line width=.2pt, opacity=0.5},
        tick label style={font=\footnotesize},
        ymax=2e-3,
    ]

    \addplot [
        dashed, mark=square*, mark size=1.0pt, color=blue, mark options={solid}
    ] coordinates {
        (0.1111111111111111, 0.0063541161484246) (0.125, 0.005084632106445497) (0.1388888888888889, 0.003943987809269642) (0.1527777777777778, 0.0029321832568970367) (0.16666666666666666, 0.002049218449327681) (0.18055555555555555, 0.0012950933865615716) (0.19444444444444445, 0.0006698080685987118) (0.20833333333333334, 0.0004193195149784717) (0.2222222222222222, 0.0003622654789697989) (0.2361111111111111, 0.0003098304708252985) (0.25, 0.0002620144905449706) (0.2638888888888889, 0.00021881753812881522) (0.2777777777777778, 0.00018023961357683236) (0.2916666666666667, 0.00014628071688902198) (0.3055555555555556, 0.00012147455330098936) (0.3194444444444444, 0.00010575445067691667) (0.3333333333333333, 9.13197691261303e-05) (0.3472222222222222, 7.817050864863027e-05) (0.3611111111111111, 6.630666924441664e-05) (0.375, 5.572825091348933e-05) (0.3888888888888889, 4.643525365584841e-05) (0.4027777777777778, 3.916681406452218e-05) (0.4166666666666667, 3.5506796267428494e-05) (0.4305555555555556, 3.2076290125009286e-05) (0.4444444444444444, 2.8875295637264558e-05) (0.4583333333333333, 2.5903812804194283e-05) (0.4722222222222222, 2.3161841625798487e-05) (0.4861111111111111, 2.064938210207716e-05) (0.5, 1.8366434233030304e-05) (0.5138888888888888, 1.6714428960259616e-05) (0.5277777777777778, 1.5162441490033806e-05) (0.5416666666666666, 1.3710471822352907e-05) (0.5555555555555556, 1.2358519957216888e-05) (0.5694444444444444, 1.1106585894625779e-05) (0.5833333333333334, 9.954669634579554e-06) (0.5972222222222222, 8.90277117707823e-06) (0.6111111111111112, 8.067914958375668e-06) (0.625, 7.3211912479109725e-06) (0.6388888888888888, 6.628772669579503e-06) (0.6527777777777778, 5.9906592233812555e-06) (0.6666666666666666, 5.406850909316239e-06) (0.6805555555555556, 4.877347727384447e-06) (0.6944444444444444, 4.402149677585884e-06) (0.7083333333333334, 4.063786595528047e-06) (0.7222222222222222, 3.8014031478841955e-06) (0.7361111111111112, 3.551643097218267e-06) (0.75, 3.314506443530263e-06) (0.7638888888888888, 3.089993186820182e-06) (0.7777777777777778, 2.8781033270880236e-06) (0.7916666666666666, 2.6788368643337903e-06) (0.8055555555555556, 2.4921937985574783e-06) (0.8194444444444444, 2.318174129759091e-06) (0.8333333333333334, 2.1567778579386267e-06) (0.8472222222222222, 2.008004983096086e-06) (0.8611111111111112, 1.8718555052314686e-06) (0.875, 1.748329424344775e-06) (0.8888888888888888, 1.637426740436005e-06) 
    };

    \addplot [
        dashed, mark=triangle*, mark size=1.5pt, color=red, mark options={solid}
    ] coordinates {
        (0.1, 0.00746246799826622) (0.11, 0.006230071433590482) (0.12000000000000001, 0.005137466228989572) (0.13, 0.004176541180422512) (0.14, 0.003339185083848319) (0.15000000000000002, 0.0026172867352260132) (0.16, 0.0020027349305146154) (0.17, 0.0014874184656731419) (0.18, 0.001063226136660617) (0.19, 0.0007220467394360558) (0.2, 0.0004557690699584782) (0.21000000000000002, 0.00025628192418690525) (0.22, 0.00011547409808035615) (0.23, 2.5234387597849564e-05) (0.24000000000000002, -2.2548411301594618e-05) (0.25, -3.598550265895701e-05) (0.26, -2.3188090515217983e-05) (0.27, 7.73262108864144e-06) (0.28, 4.866542811164144e-05) (0.29000000000000004, 9.149912651280076e-05) (0.30000000000000004, 0.00012812251225113885) (0.31, 0.00015209598336053445) (0.32, 0.00016366634617430133) (0.33, 0.00016475200910061184) (0.33999999999999997, 0.00015727138054763855) (0.35, 0.00014314286892355406) (0.36, 0.00012428488263653074) (0.37, 0.00010261583009474138) (0.38, 8.005411970635838e-05) (0.39, 5.851815987955442e-05) (0.4, 3.9926359022501856e-05) (0.41000000000000003, 2.5777505034020233e-05) (0.42000000000000004, 1.58919037755158e-05) (0.43000000000000005, 9.670240599041647e-06) (0.44000000000000006, 6.513200856650876e-06) (0.45000000000000007, 5.821469900396567e-06) (0.45999999999999996, 6.9957330823317906e-06) (0.47, 9.436675754509698e-06) (0.48, 1.2544983268983352e-05) (0.49, 1.572134097780588e-05) (0.5, 1.8366434233030304e-05) (0.51, 2.0005904181042605e-05) (0.52, 2.0665215145559967e-05) (0.53, 2.0494787244632406e-05) (0.54, 1.964504059630994e-05) (0.55, 1.826639531864259e-05) (0.56, 1.6509271529680387e-05) (0.5700000000000001, 1.4524089347473334e-05) (0.58, 1.2461268890071484e-05) (0.59, 1.047123027552481e-05) (0.6, 8.704393621883355e-06) (0.61, 7.279001762185544e-06) (0.62, 6.184588389423413e-06) (0.63, 5.378509911577403e-06) (0.64, 4.8181227366279555e-06) (0.65, 4.4607832725555105e-06) (0.66, 4.26384792734051e-06) (0.67, 4.184673108963395e-06) (0.6799999999999999, 4.180615225404603e-06) (0.69, 4.209030684644579e-06) (0.7, 4.227275894663762e-06) (0.71, 4.200803825616416e-06) (0.72, 4.127453696352089e-06) (0.73, 4.013161287894152e-06) (0.74, 3.863862381265975e-06) (0.75, 3.685492757490931e-06) (0.76, 3.4839881975923876e-06) (0.77, 3.2652844825937177e-06) (0.78, 3.0353173935182917e-06) (0.79, 2.8000227113894787e-06) (0.8, 2.5653362172306515e-06) (0.8099999999999999, 2.337193692065182e-06) (0.82, 2.121530916916436e-06) (0.83, 1.924283672807787e-06) (0.84, 1.7513877407626054e-06) (0.85, 1.608778901804262e-06) (0.86, 1.5023929369561278e-06) (0.87, 1.4381656272415727e-06) (0.88, 1.4220327536839683e-06) (0.89, 1.4599300973066848e-06) (0.9, 1.557793439133093e-06)
    };

    \addplot [
        only marks, mark=*, mark size=2.0pt, color=black, line width=0.8pt, mark options={solid}
    ] coordinates {
        (0.1, 0.00746246799826622) (0.2, 0.0004557690699584782) (0.3, 0.0001281225122511387) (0.4, 3.9926359022501856e-05) (0.5, 1.8366434233030304e-05) (0.6, 8.704393621883355e-06) (0.7, 4.227275894663762e-06) (0.8, 2.5653362172306515e-06) (0.9, 1.5577934391330928e-06)
    };
    
    \end{axis}
    \end{tikzpicture}
    \caption{Sensitivity for \texttt{blocks.1.attn.qkv}}
    \label{supp:fig:sensitivity_analysis:qkv_sensitivity}
\end{subfigure}

\caption{Comparison of interpolation methods for two layer profiles from the DeiT model. Black points show the discrete, measured (cost, error) data points. Red curves illustrate the artifacts (e.g., artificial minima in the \texttt{qkv} plot) created by a naive, global cubic spline interpolation. Blue curves show the results of our robust, sliding-window method (Algorithm~\ref{alg:interpolate}), which provides a locally consistent fit and avoids such artifacts.}
\label{supp:fig:sensitivity_analysis}

\end{figure*}

%% file: egbib.bib
@String(CVPR= {IEEE Conf. Comput. Vis. Pattern Recog.})

@String(ICCV= {Int. Conf. Comput. Vis.})

@String(ECCV= {Eur. Conf. Comput. Vis.})

@String(NIPS= {Adv. Neural Inform. Process. Syst.})

@String(ICLR = {Int. Conf. Learn. Represent.})

@String(AAAI = {AAAI})

@String(CVPR  = {CVPR})

@String(ICCV  = {ICCV})

@String(ECCV  = {ECCV})

@String(NIPS  = {NeurIPS})

@String(ICLR  = {ICLR})

@inproceedings{
hsu2022fwsvd,
title={Language model compression with weighted low-rank factorization},
author={Yen-Chang Hsu and Ting Hua and Sungen Chang and Qian Lou and Yilin Shen and Hongxia Jin},
booktitle=ICLR,
year={2022},
url={https://openreview.net/forum?id=uPv9Y3gmAI5}
}

@misc{yuan2024asvd,
      title={ASVD: Activation-aware Singular Value Decomposition for Compressing Large Language Models}, 
      author={Zhihang Yuan and Yuzhang Shang and Yue Song and Qiang Wu and Yan Yan and Guangyu Sun},
      year={2024},
      eprint={2312.05821},
      archivePrefix={arXiv},
      primaryClass={cs.CL},
      url={https://arxiv.org/abs/2312.05821}, 
}

@inproceedings{wang2024svd-llm,
  title={{SVD}-{LLM}: Truncation-aware Singular Value Decomposition for Large Language Model Compression},
  author={Xin Wang and Yu Zheng and Zhongwei Wan and Mi Zhang},
  booktitle=ICLR,
  year={2025},
  url={https://openreview.net/forum?id=LNYIUouhdt}
}

@inproceedings{2025-svd-llmv2,
    title = "{SVD}-{LLM} V2: Optimizing Singular Value Truncation for Large Language Model Compression",
    author = "Wang, Xin  and
      Alam, Samiul  and
      Wan, Zhongwei  and
      Shen, Hui  and
      Zhang, Mi",
    editor = "Chiruzzo, Luis  and
      Ritter, Alan  and
      Wang, Lu",
    booktitle = "Proceedings of the 2025 Conference of the Nations of the Americas Chapter of the Association for Computational Linguistics: Human Language Technologies (Volume 1: Long Papers)",
    month = apr,
    year = "2025",
    address = "Albuquerque, New Mexico",
    publisher = "Association for Computational Linguistics",
    url = "https://aclanthology.org/2025.naacl-long.217/",
    doi = "10.18653/v1/2025.naacl-long.217",
    pages = "4287--4296",
    ISBN = "979-8-89176-189-6",
}

@INPROCEEDINGS{DeepCompress,
  author={Ahmed, Sabbir and Arafat, Abdullah Al and Najafi, Deniz and Mahmood, Akhlak and Rizve, Mamshad Nayeem and Al Nahian, Mohaiminul and Zhou, Ranyang and Angizi, Shaahin and Rakin, Adnan Siraj},
  booktitle={2025 IEEE/CVF Conference on Computer Vision and Pattern Recognition (CVPR)}, 
  title={DeepCompress-ViT: Rethinking Model Compression to Enhance Efficiency of Vision Transformers at the Edge}, 
  year={2025},
  volume={},
  number={},
  pages={30147-30156},
  doi={10.1109/CVPR52734.2025.02806}
  }

@InProceedings{2025_flar-svd,
    author    = {Thoma, Moritz and Villasante, Jorge and Aghajanzadeh, Emad and Sampath, Shambhavi Balamuthu and Mori, Pierpaolo and Groetzinger, Maximilian and Dylkin, Daniil and Vemparala, Manoj-Rohit and Fasfous, Nael and Frickenstein, Alexander and Mueller-Gritschneder, Daniel and Schlichtmann, Ulf},
    title     = {FLAR-SVD: Fast and Latency-Aware Singular Value Decomposition for Model Compression},
    booktitle = {Proceedings of the Computer Vision and Pattern Recognition Conference (CVPR) Workshops},
    month     = {June},
    year      = {2025},
    pages     = {1898-1907}
}

@misc{chekalina2025generalizedfisherweightedsvdscalable,
      title={Generalized Fisher-Weighted SVD: Scalable Kronecker-Factored Fisher Approximation for Compressing Large Language Models}, 
      author={Viktoriia Chekalina and Daniil Moskovskiy and Daria Cherniuk and Maxim Kurkin and Andrey Kuznetsov and Evgeny Frolov},
      year={2025},
      eprint={2505.17974},
      archivePrefix={arXiv},
      primaryClass={cs.LG},
      url={https://arxiv.org/abs/2505.17974}, 
}

@inproceedings{
    qinsi2025dobisvd,
    title={Dobi-{SVD}: Differentiable {SVD} for {LLM} Compression and Some New Perspectives},
    author={Wang Qinsi and Jinghan Ke and Masayoshi Tomizuka and Kurt Keutzer and Chenfeng Xu},
    booktitle={The Thirteenth International Conference on Learning Representations},
    year={2025},
    url={https://openreview.net/forum?id=kws76i5XB8}
}

@inproceedings{pela,
  author       = {Yangyang Guo and Guangzhi Wang and Mohan Kankanhalli},
  title        = {PELA: Learning Parameter-Efficient Models with Low-Rank Approximation},
  booktitle    = {CVPR},
  year         = {2024}
}

@InProceedings{azizi2024memoryViT,
author="Azizi, Seyedarmin
and Nazemi, Mahdi
and Pedram, Massoud",
editor="Del Bue, Alessio
and Canton, Cristian
and Pont-Tuset, Jordi
and Tommasi, Tatiana",
title="Memory-Efficient Vision Transformers: An Activation-Aware Mixed-Rank Compression Strategy",
booktitle="Computer Vision -- ECCV 2024 Workshops",
year="2025",
publisher="Springer Nature Switzerland",
address="Cham",
pages="55--66",
isbn="978-3-031-91979-4"
}

@INPROCEEDINGS{Luo2024FastLRD,
  author={Luo, Yuan-June and Tai, Yu-Shan and Lin, Ming-Guang and Wu, An-Yeu Andy},
  booktitle={2024 IEEE International Symposium on Circuits and Systems (ISCAS)}, 
  title={Similarity-Aware Fast Low-Rank Decomposition Framework for Vision Transformers}, 
  year={2024},
  volume={},
  number={},
  pages={1-5},
  doi={10.1109/ISCAS58744.2024.10557934}}

@InProceedings{xiao2023comcat,
  title = 	 {{COMCAT}: Towards Efficient Compression and Customization of Attention-Based Vision Models},
  author =       {Xiao, Jinqi and Yin, Miao and Gong, Yu and Zang, Xiao and Ren, Jian and Yuan, Bo},
  booktitle = {Proceedings of the 40th International Conference on Machine Learning},
  pages = 	 {38125--38136},
  year = 	 {2023},
  editor = 	 {Krause, Andreas and Brunskill, Emma and Cho, Kyunghyun and Engelhardt, Barbara and Sabato, Sivan and Scarlett, Jonathan},
  volume = 	 {202},
  series = 	 {Proceedings of Machine Learning Research},
  month = 	 {23--29 Jul},
  publisher =    {PMLR},
  url = 	 {https://proceedings.mlr.press/v202/xiao23e.html}
}

@INPROCEEDINGS{Chang2024FLoRA,
  author={Chang, Chi-Chih and Sung, Yuan-Yao and Yu, Shixing and Huang, Ning-Chi and Marculescu, Diana and Wu, Kai-Chiang},
  booktitle={2024 IEEE/CVF Winter Conference on Applications of Computer Vision}, 
  title={FLORA: Fine-grained Low-Rank Architecture Search for Vision Transformer}, 
  year={2024},
  volume={},
  number={},
  pages={2470-2479},
  doi={10.1109/WACV57701.2024.00247}
}

@inproceedings{
    thoma2026advancing,
    title={Advancing {SVD}-based {LLM} Compression via Layer-Wise Error Model Search},
    author={Moritz Thoma and Maximilian Groezinger and Maximilian Forstenh{\"a}usler and Emad Aghajanzadeh and Manoj Rohit Vemparala and Christos Anagnostopoulos and Pierpaolo Mori and Nael Fasfous and Alexander Frickenstein and Daniel Mueller-Gritschneder and Ulf Schlichtmann},
    booktitle={Forty-third International Conference on Machine Learning},
    year={2026},
    url={https://openreview.net/forum?id=IjIgNPFuCt}
}

@inproceedings{gao2024adaptive_ARS,
  title={Adaptive rank selections for low-rank approximation of language models},
  author={Gao, Shangqian and Hua, Ting and Hsu, Yen-Chang and Shen, Yilin and Jin, Hongxia},
  booktitle={Proceedings of the 2024 Conference of the North American Chapter of the Association for Computational Linguistics: Human Language Technologies (Volume 1: Long Papers)},
  pages={227--241},
  year={2024}
}

@Inbook{VanLoan1993,
author="Van Loan, C. F.
and Pitsianis, N.",
editor="Moonen, Marc S.
and Golub, Gene H.
and De Moor, Bart L. R.",
title="Approximation with Kronecker Products",
bookTitle="Linear Algebra for Large Scale and Real-Time Applications",
year="1993",
publisher="Springer Netherlands",
address="Dordrecht",
pages="293--314",
isbn="978-94-015-8196-7",
doi="10.1007/978-94-015-8196-7_17",
url="https://doi.org/10.1007/978-94-015-8196-7_17"
}

@InProceedings{kfac-martens15,
  title = 	 {Optimizing Neural Networks with Kronecker-factored Approximate Curvature},
  author = 	 {Martens, James and Grosse, Roger},
  booktitle = 	 {Proceedings of the 32nd International Conference on Machine Learning},
  pages = 	 {2408--2417},
  year = 	 {2015},
  editor = 	 {Bach, Francis and Blei, David},
  volume = 	 {37},
  series = 	 {Proceedings of Machine Learning Research},
  address = 	 {Lille, France},
  month = 	 {07--09 Jul},
  publisher =    {PMLR},
  url = 	 {https://proceedings.mlr.press/v37/martens15.html}
}

@inproceedings{runa-kfacreduce,
 author = {Eschenhagen, Runa and Immer, Alexander and Turner, Richard and Schneider, Frank and Hennig, Philipp},
 booktitle = {Advances in Neural Information Processing Systems},
 editor = {A. Oh and T. Naumann and A. Globerson and K. Saenko and M. Hardt and S. Levine},
 pages = {33624--33655},
 publisher = {Curran Associates, Inc.},
 title = {Kronecker-Factored Approximate Curvature for Modern Neural Network Architectures},
 url = {https://proceedings.neurips.cc/paper_files/paper/2023/file/6a6679e3d5b9f7d5f09cdb79a5fc3fd8-Paper-Conference.pdf},
 volume = {36},
 year = {2023}
}

@InProceedings{shampoo-gupta,
  title = 	 {Shampoo: Preconditioned Stochastic Tensor Optimization},
  author =       {Gupta, Vineet and Koren, Tomer and Singer, Yoram},
  booktitle = 	 {Proceedings of the 35th International Conference on Machine Learning},
  pages = 	 {1842--1850},
  year = 	 {2018},
  editor = 	 {Dy, Jennifer and Krause, Andreas},
  volume = 	 {80},
  series = 	 {Proceedings of Machine Learning Research},
  month = 	 {10--15 Jul},
  publisher =    {PMLR},
  url = 	 {https://proceedings.mlr.press/v80/gupta18a.html}
}

@inproceedings{
shampoo-squared,
title={A New Perspective on Shampoo's Preconditioner},
author={Depen Morwani and Itai Shapira and Nikhil Vyas and eran malach and Sham M. Kakade and Lucas Janson},
booktitle={The Thirteenth International Conference on Learning Representations},
year={2025},
url={https://openreview.net/forum?id=c6zI3Cp8c6}
}

@article{fisherhessianmartens,
  author  = {James Martens},
  title   = {New Insights and Perspectives on the Natural Gradient Method},
  journal = {Journal of Machine Learning Research},
  year    = {2020},
  volume  = {21},
  number  = {146},
  pages   = {1--76},
  url     = {http://jmlr.org/papers/v21/17-678.html}
}

@INPROCEEDINGS{imagenet,
  author={Deng, Jia and Dong, Wei and Socher, Richard and Li, Li-Jia and Kai Li and Li Fei-Fei},
  booktitle=CVPR, 
  title={ImageNet: A large-scale hierarchical image database}, 
  year={2009},
  volume={},
  number={},
  pages={248-255},
  doi={10.1109/CVPR.2009.5206848}}

@INPROCEEDINGS{ade20k,
  author={Zhou, Bolei and Zhao, Hang and Puig, Xavier and Fidler, Sanja and Barriuso, Adela and Torralba, Antonio},
  booktitle={2017 IEEE Conference on Computer Vision and Pattern Recognition (CVPR)}, 
  title={Scene Parsing through ADE20K Dataset}, 
  year={2017},
  volume={},
  number={},
  pages={5122-5130},
  doi={10.1109/CVPR.2017.544}}

@inproceedings{COCO,
  title={Microsoft coco: Common objects in context},
  author={Lin, Tsung-Yi and Maire, Michael and Belongie, Serge and Hays, James and Perona, Pietro and Ramanan, Deva and Doll{\'a}r, Piotr and Zitnick, C Lawrence},
  booktitle={Computer vision--ECCV 2014: 13th European conference, zurich, Switzerland, September 6-12, 2014, proceedings, part v 13},
  pages={740--755},
  year={2014},
  organization={Springer}
}

@article{zellers2019hellaswag,
  title={Hellaswag: Can a machine really finish your sentence?},
  author={Zellers, Rowan and Holtzman, Ari and Bisk, Yonatan and Farhadi, Ali and Choi, Yejin},
  journal={arXiv preprint arXiv:1905.07830},
  year={2019}
}

@article{mihaylov2018openbookqa,
  title={Can a suit of armor conduct electricity? a new dataset for open book question answering},
  author={Mihaylov, Todor and Clark, Peter and Khot, Tushar and Sabharwal, Ashish},
  journal={arXiv preprint arXiv:1809.02789},
  year={2018}
}

@article{sakaguchi2021winogrande,
  title={Winogrande: An adversarial winograd schema challenge at scale},
  author={Sakaguchi, Keisuke and Bras, Ronan Le and Bhagavatula, Chandra and Choi, Yejin},
  journal={Communications of the ACM},
  volume={64},
  number={9},
  pages={99--106},
  year={2021},
  publisher={ACM New York, NY, USA}
}

@article{clark2018arc,
  title={Think you have solved question answering? try arc, the ai2 reasoning challenge},
  author={Clark, Peter and Cowhey, Isaac and Etzioni, Oren and Khot, Tushar and Sabharwal, Ashish and Schoenick, Carissa and Tafjord, Oyvind},
  journal={arXiv preprint arXiv:1803.05457},
  year={2018}
}

@inproceedings{bisk2020piqa,
  title={Piqa: Reasoning about physical commonsense in natural language},
  author={Bisk, Yonatan and Zellers, Rowan and Gao, Jianfeng and Choi, Yejin and others},
  booktitle={Proceedings of the AAAI conference on artificial intelligence},
  volume={34},
  number={05},
  pages={7432--7439},
  year={2020}
}

@inproceedings{
merity2017pointer_wikitext2,
title={Pointer Sentinel Mixture Models},
author={Stephen Merity and Caiming Xiong and James Bradbury and Richard Socher},
booktitle={International Conference on Learning Representations},
year={2017},
url={https://openreview.net/forum?id=Byj72udxe}
}

@InProceedings{deit,
  title =     {Training data-efficient image transformers \& distillation through attention},
  author =    {Touvron, Hugo and Cord, Matthieu and Douze, Matthijs and Massa, Francisco and Sablayrolles, Alexandre and Jegou, Herve},
  booktitle = {Proceedings of the 38th International Conference on Machine Learning},
  pages =     {10347--10357},
  year =      {2021},
  volume =    {139},
  month =     {July}
}

@article{convnext,
  author  = {Zhuang Liu and Hanzi Mao and Chao-Yuan Wu and Christoph Feichtenhofer and Trevor Darrell and Saining Xie},
  title   = {A ConvNet for the 2020s},
  journal = CVPR,
  year    = {2022},
}

@inproceedings{swin,
  title={Swin Transformer: Hierarchical Vision Transformer using Shifted Windows},
  author={Liu, Ze and Lin, Yutong and Cao, Yue and Hu, Han and Wei, Yixuan and Zhang, Zheng and Lin, Stephen and Guo, Baining},
  booktitle={Proceedings of the IEEE/CVF International Conference on Computer Vision},
  year={2021}
}

@inproceedings{mambavision,
  title={Mambavision: A hybrid mamba-transformer vision backbone},
  author={Hatamizadeh, Ali and Kautz, Jan},
  booktitle={Proceedings of the Computer Vision and Pattern Recognition Conference},
  pages={25261--25270},
  year={2025}
}

@inproceedings{
    vit,
    title={An Image is Worth 16x16 Words: Transformers for Image Recognition at Scale},
    author={Alexey Dosovitskiy and Lucas Beyer and Alexander Kolesnikov and Dirk Weissenborn and Xiaohua Zhai and Thomas Unterthiner and Mostafa Dehghani and Matthias Minderer and Georg Heigold and Sylvain Gelly and Jakob Uszkoreit and Neil Houlsby},
    booktitle={International Conference on Learning Representations},
    year={2021},
    url={https://openreview.net/forum?id=YicbFdNTTy}
}

@INPROCEEDINGS{maskrcnn,
  author={He, Kaiming and Gkioxari, Georgia and Dollár, Piotr and Girshick, Ross},
  booktitle={2017 IEEE International Conference on Computer Vision (ICCV)}, 
  title={Mask R-CNN}, 
  year={2017},
  volume={},
  number={},
  pages={2980-2988},
  doi={10.1109/ICCV.2017.322}}

@inproceedings{xiao2018unified,
  title={Unified perceptual parsing for scene understanding},
  author={Xiao, Tete and Liu, Yingcheng and Zhou, Bolei and Jiang, Yuning and Sun, Jian},
  booktitle={Proceedings of the European conference on computer vision (ECCV)},
  pages={418--434},
  year={2018}
}

@software{john_forrest_2024_13347261_cbc,
  author       = {John Forrest and
                  Ted Ralphs and
                  Stefan Vigerske and
                  Haroldo Gambini Santos and
                  John Forrest and
                  Lou Hafer and
                  Bjarni Kristjansson and
                  jpfasano and
                  EdwinStraver and
                  Jan-Willem and
                  Miles Lubin and
                  rlougee and
                  a-andre and
                  jpgoncal1 and
                  Samuel Brito and
                  h-i-gassmann and
                  Cristina and
                  Matthew Saltzman and
                  tosttost and
                  Bruno Pitrus and
                  Fumiaki MATSUSHIMA and
                  Patrick Vossler and
                  Ron @ SWGY and
                  to-st},
  title        = {coin-or/Cbc: Release releases/2.10.12},
  month        = aug,
  year         = 2024,
  publisher    = {Zenodo},
  version      = {releases/2.10.12},
  doi          = {10.5281/zenodo.13347261},
  url          = {https://doi.org/10.5281/zenodo.13347261},
}

@Manual{pulp,
  title        = {PuLP: A Linear Programming Toolkit for Python},
  author       = {Stuart Mitchell and Michael O'Sullivan and Iain Dunning},
  year         = {2011},
  note         = {Version 3.3.0, \url{https://coin-or.github.io/pulp/}},
  organization = {COIN-OR},
  url          = {https://coin-or.github.io/pulp/}
}

@misc{eval-harness,
  author       = {Gao, Leo and Tow, Jonathan and Abbasi, Baber and Biderman, Stella and Black, Sid and DiPofi, Anthony and Foster, Charles and Golding, Laurence and Hsu, Jeffrey and Le Noac'h, Alain and Li, Haonan and McDonell, Kyle and Muennighoff, Niklas and Ociepa, Chris and Phang, Jason and Reynolds, Laria and Schoelkopf, Hailey and Skowron, Aviya and Sutawika, Lintang and Tang, Eric and Thite, Anish and Wang, Ben and Wang, Kevin and Zou, Andy},
  title        = {The Language Model Evaluation Harness},
  month        = 07,
  year         = 2024,
  publisher    = {Zenodo},
  version      = {v0.4.3},
  doi          = {10.5281/zenodo.12608602},
  url          = {https://zenodo.org/records/12608602}
}

@InProceedings{Yang_2023_Nvit,
    author    = {Yang, Huanrui and Yin, Hongxu and Shen, Maying and Molchanov, Pavlo and Li, Hai and Kautz, Jan},
    title     = {Global Vision Transformer Pruning With Hessian-Aware Saliency},
    booktitle = CVPR,
    month     = {June},
    year      = {2023},
    pages     = {18547-18557}
}

@inproceedings{fang2024isomorphicpruning,
  title={Isomorphic pruning for vision models},
  author={Fang, Gongfan and Ma, Xinyin and Mi, Michael Bi and Wang, Xinchao},
  booktitle=ECCV,
  pages={232--250},
  year={2024},
  organization={Springer}
}

@article{gao2024disp,
  title={Disp-llm: Dimension-independent structural pruning for large language models},
  author={Gao, Shangqian and Lin, Chi-Heng and Hua, Ting and Tang, Zheng and Shen, Yilin and Jin, Hongxia and Hsu, Yen-Chang},
  journal=NIPS,
  volume={37},
  pages={72219--72244},
  year={2024}
}

@inproceedings{
    sparsegpt,
    author = {Frantar, Elias and Alistarh, Dan},
    title = {SparseGPT: massive language models can be accurately pruned in one-shot},
    year = {2023},
    publisher = {JMLR.org},
    booktitle = {Proceedings of the 40th International Conference on Machine Learning},
    articleno = {414},
    numpages = {15},
    location = {Honolulu, Hawaii, USA},
    series = {ICML'23}
}

@article{fang2024maskllm,
  title={Maskllm: Learnable semi-structured sparsity for large language models},
  author={Fang, Gongfan and Yin, Hongxu and Muralidharan, Saurav and Heinrich, Greg and Pool, Jeff and Kautz, Jan and Molchanov, Pavlo and Wang, Xinchao},
  journal={Advances in Neural Information Processing Systems},
  volume={37},
  pages={7736--7758},
  year={2024}
}

@InProceedings{Agarwal_2024_CVPR,
    author    = {Agarwal, Parakh and Mathew, Manu and Patel, Kunal Ranjan and Tripathi, Varun and Swami, Pramod},
    title     = {Prune Efficiently by Soft Pruning},
    booktitle = {Proceedings of the IEEE/CVF Conference on Computer Vision and Pattern Recognition (CVPR) Workshops},
    month     = {June},
    year      = {2024},
    pages     = {2210-2217}
}

@inproceedings{luo2025icp,
  title={ICP: Immediate Compensation Pruning for Mid-to-high Sparsity},
  author={Luo, Xin and Fu, Xueming and Jiang, Zihang and Zhou, S Kevin},
  booktitle={Proceedings of the Computer Vision and Pattern Recognition Conference},
  pages={9487--9496},
  year={2025}
}

@software{Wightman_PyTorch_Image_Models,
author = {Wightman, Ross},
doi = {10.5281/zenodo.4414861},
license = {Apache 2.0},
title = {{PyTorch Image Models}},
url = {https://github.com/huggingface/pytorch-image-models},
version = {1.0.11}
}

@article{mmdetection,
  title   = {{MMDetection}: Open MMLab Detection Toolbox and Benchmark},
  author  = {Chen, Kai and Wang, Jiaqi and Pang, Jiangmiao and Cao, Yuhang and
             Xiong, Yu and Li, Xiaoxiao and Sun, Shuyang and Feng, Wansen and
             Liu, Ziwei and Xu, Jiarui and Zhang, Zheng and Cheng, Dazhi and
             Zhu, Chenchen and Cheng, Tianheng and Zhao, Qijie and Li, Buyu and
             Lu, Xin and Zhu, Rui and Wu, Yue and Dai, Jifeng and Wang, Jingdong
             and Shi, Jianping and Ouyang, Wanli and Loy, Chen Change and Lin, Dahua},
  journal= {arXiv preprint arXiv:1906.07155},
  year={2019}
}

@misc{mmseg2020,
    title={{MMSegmentation}: OpenMMLab Semantic Segmentation Toolbox and Benchmark},
    author={MMSegmentation Contributors},
    howpublished = {\url{https://github.com/open-mmlab/mmsegmentation}},
    year={2020}
}

@software{Neural_Magic_DeepSparse_2021,
author = {Neural Magic},
month = feb,
title = {{DeepSparse}},
url = {https://github.com/neuralmagic/deepsparse},
year = {2021}
}

@article{bai-2023,
	author = {Bai, Hongxiao and Bai, Hongxiao},
	month = {7},
	title = {{Structured sparsity in the NVIDIA Ampere architecture and applications in search engines}},
	year = {2023},
	url = {https://developer.nvidia.com/blog/structured-sparsity-in-the-nvidia-ampere-architecture-and-applications-in-search-engines/},
}

@INPROCEEDINGS {runddontwalk,
author = { Chen, Jierun and Kao, Shiu-hong and He, Hao and Zhuo, Weipeng and Wen, Song and Lee, Chul-Ho and Chan, S.-H. Gary },
booktitle = { 2023 IEEE/CVF Conference on Computer Vision and Pattern Recognition (CVPR) },
title = {{ Run, Don't Walk: Chasing Higher FLOPS for Faster Neural Networks }},
year = {2023},
volume = {},
ISSN = {},
pages = {12021-12031},
doi = {10.1109/CVPR52729.2023.01157},
url = {https://doi.ieeecomputersociety.org/10.1109/CVPR52729.2023.01157},
publisher = {IEEE Computer Society},
address = {Los Alamitos, CA, USA},
month =Jun}

@InProceedings{Fu_2025_CVPR,
    author    = {Fu, Minghao and Yu, Hao and Shao, Jie and Zhou, Junjie and Zhu, Ke and Wu, Jianxin},
    title     = {Quantization without Tears},
    booktitle = {Proceedings of the IEEE/CVF Conference on Computer Vision and Pattern Recognition (CVPR)},
    month     = {June},
    year      = {2025},
    pages     = {4462-4472}
}

@inproceedings{
    frantar2023optq,
    title={{OPTQ}: Accurate Quantization for Generative Pre-trained Transformers},
    author={Elias Frantar and Saleh Ashkboos and Torsten Hoefler and Dan Alistarh},
    booktitle={The Eleventh International Conference on Learning Representations },
    year={2023},
    url={https://openreview.net/forum?id=tcbBPnfwxS}
}

@inproceedings{
liu2018darts,
title={{DARTS}: Differentiable Architecture Search},
author={Hanxiao Liu and Karen Simonyan and Yiming Yang},
booktitle={International Conference on Learning Representations},
year={2019},
url={https://openreview.net/forum?id=S1eYHoC5FX},
}

@article{szekely2007dcorr,
  title = {Measuring and testing dependence by correlation of distances},
  volume = {35},
  ISSN = {0090-5364},
  url = {http://dx.doi.org/10.1214/009053607000000505},
  DOI = {10.1214/009053607000000505},
  number = {6},
  journal = {The Annals of Statistics},
  publisher = {Institute of Mathematical Statistics},
  author = {Székely,  Gábor J. and Rizzo,  Maria L. and Bakirov,  Nail K.},
  year = {2007},
  month = dec 
}

@article{panda2019hyppo,
  title={hyppo: A multivariate hypothesis testing Python package},
  author={Panda, Sambit and Palaniappan, Satish and Xiong, Junhao and Bridgeford, Eric W and Mehta, Ronak and Shen, Cencheng and Vogelstein, Joshua T},
  journal={arXiv preprint arXiv:1907.02088},
  year={2019}
}

@article{avron2011trace, author = {Avron, Haim and Toledo, Sivan}, title = {Randomized algorithms for estimating the trace of an implicit symmetric positive semi-definite matrix}, year = {2011}, issue_date = {April 2011}, publisher = {Association for Computing Machinery}, address = {New York, NY, USA}, volume = {58}, number = {2}, issn = {0004-5411}, url = {https://doi.org/10.1145/1944345.1944349}, doi = {10.1145/1944345.1944349}, journal = {J. ACM}, month = apr, articleno = {8}, numpages = {34} }

@article{hutchinson1989stochastic,
  title={A stochastic estimator of the trace of the influence matrix for Laplacian smoothing splines},
  author={Hutchinson, Michael F},
  journal={Communications in Statistics-Simulation and Computation},
  volume={18},
  number={3},
  pages={1059--1076},
  year={1989},
  publisher={Taylor \& Francis}
}

@book{deboor1978practical,
  author    = {de Boor, Carl},
  title     = {A Practical Guide to Splines},
  series    = {Applied Mathematical Sciences},
  volume    = {27},
  publisher = {Springer-Verlag},
  address   = {New York},
  year      = {1978}
}

@article{fritsch1980monotone,
  author  = {Fritsch, F. N. and Carlson, R. E.},
  title   = {Monotone Piecewise Cubic Interpolation},
  journal = {SIAM Journal on Numerical Analysis},
  volume  = {17}, number = {2}, pages = {238--246}, year = {1980},
  doi     = {10.1137/0717021}
}

@book{nemhauser1988integer,
  author    = {Nemhauser, George L. and Wolsey, Laurence A.},
  title     = {Integer and Combinatorial Optimization},
  series    = {Wiley-Interscience Series in Discrete Mathematics and Optimization},
  publisher = {John Wiley \& Sons},
  address   = {New York},
  year      = {1988},
  isbn      = {0-471-82819-X}
}
